\pdfoutput=1
\documentclass[acmsmall,screen]{acmart}

\newcommand{\tablesize}{\footnotesize}

\usepackage{amsmath}
\usepackage{amsthm}
\usepackage{thm-restate}

\declaretheorem[name=Definition]{definition}

\usepackage{hyperref}

\usepackage{algorithm}
\usepackage[noend]{algpseudocode} 
\usepackage{mathtools} 
\usepackage{xcolor}
\usepackage{booktabs}
\usepackage{graphicx}
\usepackage{caption, subcaption}
\usepackage{multirow}
 \usepackage{setspace}
 \usepackage[normalem]{ulem}
 \usepackage{enumitem}
 \usepackage{wrapfig}
 \usepackage{bm}
 \usepackage{xspace}
\usepackage{multirow}

\DeclareMathOperator*{\argmin}{arg\,min}
\DeclareMathOperator*{\argmax}{arg\,max}

\AtBeginDocument{%
  \providecommand\BibTeX{{%
    \normalfont B\kern-0.5em{\scshape i\kern-0.25em b}\kern-0.8em\TeX}}}

\begin{document}

\setcopyright{acmlicensed}
\acmPrice{}
\acmDOI{10.1145/3591299}
\acmYear{2023}
\copyrightyear{2023}
\acmSubmissionID{pldi23main-p706-p}
\acmJournal{PACMPL}
\acmVolume{7}
\acmNumber{PLDI}
\acmArticle{185}
\acmMonth{6}


\title{Incremental Verification of Neural Networks}

\author{Shubham Ugare}
\affiliation{%
  \institution{University of Illinois Urbana-Champaign}
  \country{USA}
}
\author{Debangshu Banerjee}
\affiliation{%
  \institution{University of Illinois Urbana-Champaign}
  \country{USA}
}

\author{Sasa Misailovic}
\affiliation{%
  \institution{University of Illinois Urbana-Champaign}
  \country{USA}
}

\author{Gagandeep Singh}
\affiliation{%
  \institution{University of Illinois Urbana-Champaign and VMware Research}
  \country{USA}
}


    \definecolor{WowColor}{rgb}{.75,0,.75}
\definecolor{SubtleColor}{rgb}{0,0,.50}



\newcommand{\Tool}{IVAN\xspace}

\newcommand{\vbound}{A}
\newcommand{\hbranch}{H}
\newcommand{\lb}{\textit{LB}}

\newcommand*{\node}{n}

\newcommand{\Norig}{N\xspace} 
\newcommand{\tbr}{t_\textit{H}}
\newcommand{\tbo}{t_\textit{A}}
\newcommand{\Timeb}{\textit{Time}_\Delta}
\newcommand{\Time}{\textit{Time}}

\newcommand{\inpreg}{\phi_t\xspace} 

\newcommand{\nodes}[1]{\textit{Nodes}(#1)}
\newcommand{\leaves}[1]{\textit{Leaves}(#1)}
\newcommand{\children}[1]{\textit{Children}(#1)}
\newcommand{\spec}[1]{\varphi_{#1}}
\newcommand{\nroot}{n_\textit{root}}
\newcommand{\T}[2]{T^{#1}_{#2}}
\newcommand{\Tinit}{T^{\perturbedNetwork}_0}
\newcommand{\add}{\textit{Split}}
\newcommand{\delete}{\textit{Delete}}
\newcommand{\treeset}{\mathcal{T}_\mathcal{N}}
\newcommand{\treesetforN}{\mathbf{T}}

\newcommand{\Tprune}{T_{\textit{P}}}
\newcommand{\queue}{Q}
\newcommand{\nqueue}{Q_\textit{new}}
\newcommand{\nodeMap}{M}

\newcommand{\hratio}{\alpha}
\newcommand{\threshold}{\theta}

\newcommand{\timebase}{\tau_\textit{B}}
\newcommand{\timetool}{\tau_\textit{\Tool{}}}
\newcommand{\speedup}{\mathit{Sp}}
\newcommand{\solved}{\textit{+Solved}}



\newcommand{\upto}{43x}
\newcommand{\geomean}{2.4x}
\newcommand{\geomeanhard}{3.1x}
\newcommand{\geomeanglobal}{3.8x}

\newcommand{\NA}[1]{\textcolor{SubtleColor}{ {\tiny \bf ($\star$)} #1}}
\newcommand{\NO}[1]{\textcolor{SubtleColor}{ {\tiny \bf ($\star$)} \sout{#1}}}
\newcommand{\LATER}[1]{\textcolor{SubtleColor}{ {\tiny \bf ($\dagger$)} #1}}
\newcommand{\TBD}[1]{\textcolor{SubtleColor}{ {\tiny \bf (!)} #1}}
\newcommand{\PROBLEM}[1]{\textcolor{WowColor}{ {\bf (!!)} {\bf #1}}}

\newcounter{margincounter}
\newcommand{\new}[1]{{#1}} 

\newcommand{\displaycounter}{{\arabic{margincounter}}}
\newcommand{\incdisplaycounter}{{\stepcounter{margincounter}\arabic{margincounter}}}

\newcommand{\fTBD}[1]{\textcolor{SubtleColor}{$\,^{(\incdisplaycounter)}$}\marginpar{\tiny\textcolor{SubtleColor}{ {\tiny $(\displaycounter)$} #1}}}

\newcommand{\fPROBLEM}[1]{\textcolor{WowColor}{$\,^{((\incdisplaycounter))}$}\marginpar{\tiny\textcolor{WowColor}{ {\bf $\mathbf{((\displaycounter))}$} {\bf #1}}}}

\newcommand{\fLATER}[1]{\textcolor{SubtleColor}{$\,^{(\incdisplaycounter\dagger)}$}\marginpar{\tiny\textcolor{SubtleColor}{ {\tiny $(\displaycounter\dagger)$} #1}}}

\newcommand*{\maxPaper}{\mathit{max}}
\newcommand*{\OrgNetwork}{N}
\newcommand*{\perturbedNetwork}{N^a}
\newcommand*{\Relu}{\mathit{ReLU}}
\newcommand*{\Reluabstract}[1]{\mathit{A_{ReLU}^{\#}}(#1)}
\newcommand*{\InputDimension}{m}
\newcommand*{\OutputDimension}{n}
\newcommand*{\Dimension}{n}
\newcommand*{\Layers}{l}
\newcommand*{\Reluset}{\mathcal{R}}
\newcommand*{\Statesp}{\mathcal{S}}
\newcommand*{\SplitRelu}{\mathcal{P}}
\newcommand*{\Solv}{\mathcal{F}}
\newcommand*{\GenericNet}{N}
\newcommand*{\PreCond}{\varphi}
\newcommand*{\PostCond}{\psi}
\newcommand*{\LinearProb}{\mathcal{L}}
\newcommand*{\SolvFunc}{V_{\mathcal{T}}}
\newcommand*{\ProblemMin}{LB}
\newcommand*{\ProblemMax}{\mathcal{\beta}}
\newcommand*{\FeasibleReg}{\mathcal{F}}
\newcommand*{\Eps}{\mathcal{E}}
\newcommand*{\Lpc}{\mathcal{C}}
\newcommand*{\EpsNorm}{\delta}
\newcommand*{\MaxNorm}{\eta}
\newcommand*{\Difflayer}{i_{0}}
\newcommand*{\Lpconst}{\zeta}
\newcommand*{\EpsNormBound}{\Delta_{max}}
\newcommand{\distri}[1]{\mathbf{C}_{#1}(\mathcal{N})}

\newcommand{\neta}{FCN-MNIST }
\newcommand{\netb}{CONV-MNIST } 
\newcommand{\netc}{CONV-CIFAR }
\newcommand{\netd}{CONV-CIFAR-WIDE }
\newcommand{\nete}{CONV-CIFAR-DEEP }

\newcommand{\ver}{\emph{Verified}}
\newcommand{\counterex}{\emph{Counterexample}}
\newcommand{\unknown}{\emph{Unknown}}

\newcommand{\activeList}{Active}
\newcommand{\unsolvedList}{Unsolved}

\begin{abstract}
Complete verification of deep neural networks (DNNs) can exactly determine whether the DNN satisfies a desired trustworthy property (e.g., robustness, fairness) on an infinite set of inputs or not. Despite the tremendous progress to improve the scalability of complete verifiers over the years on individual DNNs, they are inherently inefficient when a deployed DNN is updated to improve its inference speed or accuracy. The inefficiency is because the expensive verifier needs to be run from scratch on the updated DNN. To improve efficiency, we propose a new, general framework for incremental and complete DNN verification based on the design of novel theory, data structure, and algorithms. Our contributions implemented in a tool named \Tool{} yield an overall geometric mean speedup of \geomean{} for verifying challenging MNIST and CIFAR10 classifiers and a geometric mean speedup of \geomeanglobal{} for the ACAS-XU classifiers  over the state-of-the-art baselines. 
\end{abstract}

\maketitle

\vspace{-.1in}
\section{Introduction}
\label{sec:intro}

\par Deep neural networks (DNNs) are being increasingly deployed for safety-critical applications in many domains including autonomous driving \cite{bojarski2016end}, healthcare \cite{AMATO201347}, and aviation \cite{acasxu:18}. 
However, the black-box construction, vulnerability against adversarial changes to in-distribution
inputs~\cite{DBLP:journals/corr/SzegedyZSBEGF13,madry2017towards}, and fragility against out-of-distribution data~\cite{GokhaleAKTBY:21,chen2022robust} is the main hindrance to the trustworthy deployment of deep neural networks in real-world applications. 
Recent years have witnessed increasing work on developing verifiers for formally checking whether the behavior of DNNs (see \cite{https://doi.org/10.48550/arxiv.2104.02466, albarghouthi-book} for a survey) on an infinite set of inputs is trustworthy or not. For example, existing verifiers can formally prove~\cite{wang2018neurify, gehr2018ai2, bunel2020branch, bunel2020efficient, DBLP:conf/cav/BakTHJ20, ehlers2017formal} that the infinite number of images obtained after varying the intensity of pixels in an original image by a small amount will be classified correctly. Verification yields better insights into the trustworthiness of DNNs than standard test-set accuracy measurements, which only check DNN performance on a finite number of inputs. The insights can be used for selecting the most trustworthy DNN for deployment among a set of DNNs trained for the same task. 
%
Existing verifiers can be broadly classified as either complete or incomplete. Incomplete methods are more scalable but may fail to prove or disprove a trustworthiness property~\cite{gehr2018ai2,singh2018fast,singh2019abstract,singh2019beyond, zhang2018crown,Lirpa:20,DBLP:conf/nips/SalmanY0HZ19}. \new{A complete verifier always verifies the property if the property holds or otherwise returns a counterexample}. Complete verification methods are more desirable as they are guaranteed to provide an exact answer for the verification task~\cite{wang2018neurify, gehr2018ai2, bunel2020branch, bunel2020efficient, DBLP:conf/cav/BakTHJ20, ehlers2017formal, ferrari2022complete, fromherz2021fast, wang2021beta, depalma2021scaling, anderson2020strong, zhang2022general}.


\noindent \textbf{\bf Limitation of Existing Works:}
\new{
The deployed DNNs are modified for reasons such as approximation \cite{DBLP:journals/corr/abs-2103-13630, DBLP:conf/mlsys/BlalockOFG20}, fine-tuning \cite{tajbakhsh2016convolutional}, model repair \cite{SotoudehT19}, or transfer learning \cite{weiss2016survey}. 
Various approximations such as quantization, and pruning slightly perturb the DNN weights, and the updated DNN is used for the same task \cite{tf_quantization, DBLP:journals/corr/abs-2103-13630, 9586276}. Similarly, fine-tuning can also be performed to repair the network on buggy inputs while maintaining the accuracy on the original training inputs \cite{fu2022sound}.}
Each time a new DNN is created, expensive complete verification needs to be performed to check whether it is trustworthy. A fundamental limitation of all existing approaches for complete verification of DNNs is that the verifier needs to be run from scratch end-to-end every time the network is even slightly modified. As a result, developers still rely on test set accuracy as the main metric for measuring the quality of a trained network. This limitation of existing verifiers restricts their applicability as a tool for evaluating the trustworthiness of DNNs.

\noindent \textbf{This Work: Incremental and Complete Verification of DNNs:} In this work, we address the fundamental limitation of existing complete verifiers by presenting  \Tool{}, the first general technique for incremental and complete verification of DNNs. 
An original network and its updated network have similar behaviors on most of the inputs, therefore the proofs of property on these networks are also related. 
\Tool{} accelerates the complete verification of a trustworthy property on the updated network by leveraging the proof of the same property on the original network. \Tool{} can be built on top of any \new{Branch and Bound} (BaB) based method. \new{The BaB verifier recursively partitions the verification problem to gain precision. It is currently the dominant technology for constructing complete verifiers \cite{wang2018neurify, bunel2020branch, bunel2020efficient, DBLP:conf/cav/BakTHJ20, ehlers2017formal, ferrari2022complete, fromherz2021fast, wang2021beta, depalma2021scaling, Pailoor19, zhang2022general}.} 

\noindent \textbf{\bf Challenges:} 
The main challenge in building an incremental verifier on top of a non-incremental one is to determine which information to pass on and how to effectively reuse this information. Formal methods research has developed numerous techniques for incremental verification of programs, that reuse the proof from previous revisions for verifying the new revision of the program \cite{10.1145/2465449.2465456, Lachnech-etal:TACAS01, DBLP:conf/lics/OHearn18, DBLP:conf/pldi/0002CS21}.  However, often the program commits are local changes that affect only a small part of the big program. In contrast, most DNN updates result in weight perturbation across one or many layers of the network. This poses a different and more difficult challenge than incremental program verification. Additionally, DNN complete verifiers employ distinct heuristics for branching. A key challenge is to develop a generic method that incrementally verifies a network perturbed across multiple layers and is applicable to multiple complete verification methods, yet can provide significant performance benefits.

\noindent \textbf{Our Solution:} \Tool{} computes a specification tree -- a novel tree data structure representing the trace of BaB -- from the execution of the complete verifier on the original network. We design new algorithms to refine the specification tree to create a more compact tree. At a high level, the refinement involves reordering the branching decisions such that the decisions that worked well in the original verification are prioritized. Besides, it removes the branching decisions that worked poorly in the original verification by pruning nodes and edges in the specification tree. \Tool{} also improves the branching strategy in BaB for the updated network based on the observed effectiveness of branching choices when verifying the original DNN. The compact specification tree and the improved branching strategy guide the BaB execution on the updated network to faster verification, compared to non-incremental verification that starts from scratch. 
\Tool{} yields up to \upto{} speedup over the baseline based on state-of-the-art non-incremental verification techniques \cite{ijcai2021p351, bunel2020branch, singh2018fast}. It achieves a geometric mean speedup of \geomean{} across challenging fully-connected and convolutional networks over the baseline. \Tool{} is generic and can work with various common BaB branching strategies in the literature (input splitting, ReLU splitting). 

\noindent \textbf{\bf Main Contributions:}
The main contributions of this paper are:
\vspace{-0.07in}
\begin{itemize}
     \item We present a novel, general framework for incremental and complete DNN verification by designing new algorithms and data structure that allows us to succinctly encode influential branching strategies to perform efficient incremental verification of the updated network. 
    \item We identify a class of network modifications that can be efficiently verified by our framework by providing theoretical bounds on the amount of modifications. 
    \item  We implement our approach into a tool named \Tool{} and show its effectiveness over multiple state-of-the-art complete verification techniques, using distinct branching strategies (ReLU splitting and input splitting), in incrementally verifying both local and global properties of fully-connected and convolutional networks \new{with ReLU activations} trained on the popular ACAS-XU, MNIST, and CIFAR10 datasets. Our results show that for MNIST and CIFAR10 classifiers, using the ReLU splitting technique \cite{ijcai2021p351} \Tool{} yields a geometric mean speedup of \geomean{} over the state-of-the-art baseline \cite{bunel2020branch, ehlers2017formal}. For ACAS-XU, using the input splitting technique \Tool{} achieves a geometric mean speedup of \geomeanglobal{} over RefineZono~\cite{singh2019boosting}.
\end{itemize}
\vspace{-0.07in}

\noindent \Tool{} implementation is open-source, publicly available at {\color{blue}\url{https://github.com/uiuc-focal-lab/IVAN}}. An extended version of this paper containing all the proofs and additional experiments is available at {\color{blue}\url{https://arxiv.org/abs/2304.01874}}.

\section{Overview}
\begin{figure}[!t]
\centering
\includegraphics[width=13cm]{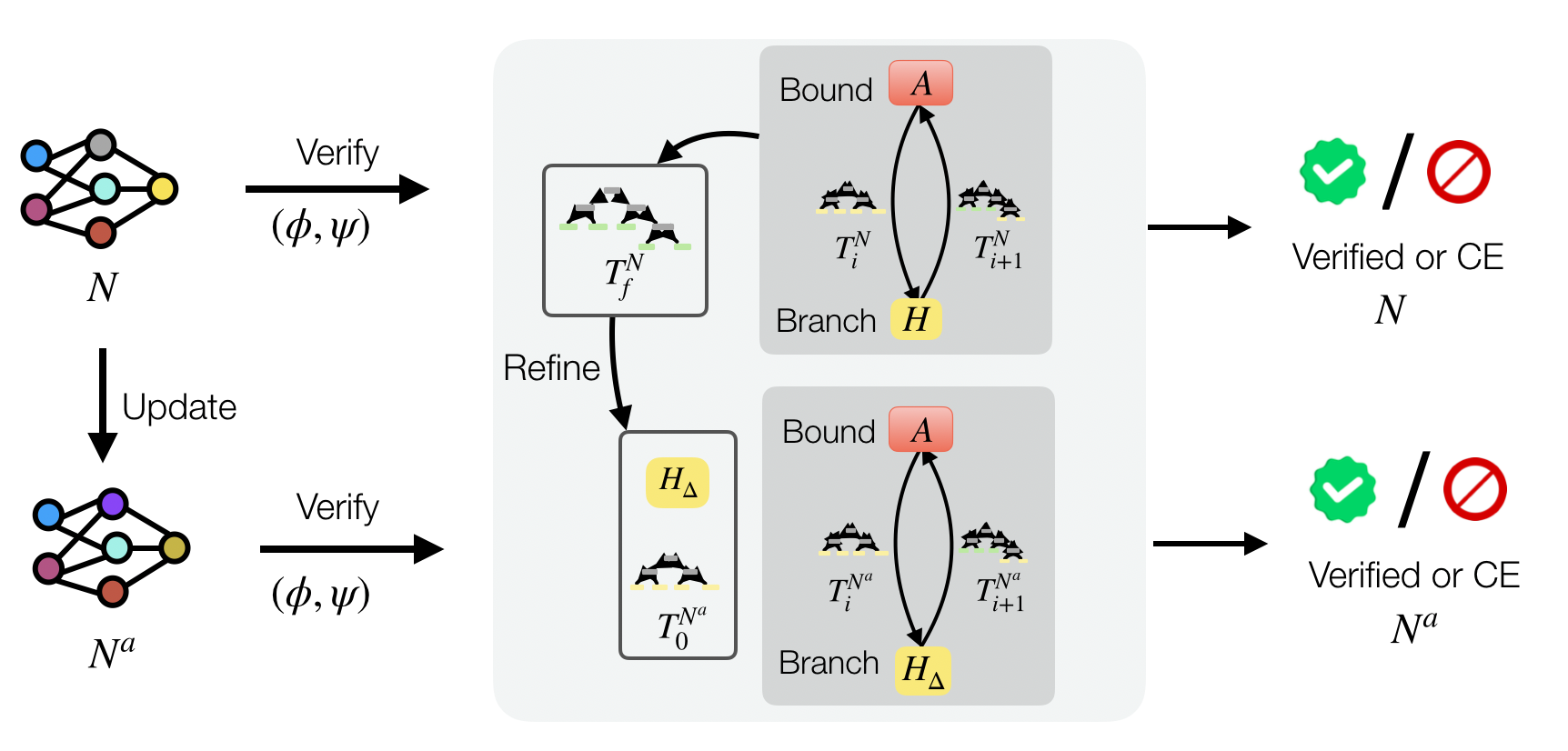}
\vspace{-.1in}
\caption{Workflow of \Tool{} from left to right. $\Tool{}$ takes the original network $N$, input specification $\phi$ and output specification $\psi$. It is built on top of a BaB-based complete verifier that utilizes an analyzer $\vbound$ for the bounding, and heuristic $\hbranch$ for branching. \Tool{} refines a specification tree $T^N_f$, result of verifying $N$, to create a compact tree $\Tinit$ and updated branching heuristic $\hbranch_\Delta$. \Tool{} performs faster verification of $\perturbedNetwork$ exploiting both $\Tinit$ and $\hbranch_\Delta$.} 
\label{fig:ivan}
\vspace{-.3in}
\end{figure}

Figure~\ref{fig:ivan} illustrates the high-level idea behind the workings of \Tool{}. It takes as input the original neural network $N$, the updated network $N^a$, a local or global input region $\phi$, and the output property $\psi$. The goal of \Tool{} is to check whether for all inputs in $\phi$, the outputs of networks $N$ and $N^a$ satisfy $\psi$. $N$ and $N^a$ have similar behaviors on the inputs in $\phi$, therefore the proofs of the property on these networks are also related. 
\Tool{} accelerates the complete verification of the property $(\phi, \psi)$ on $N^a$ by leveraging the proof of the same property on $N$.


{\noindent \bf Neural Network Verifier:} 
Popular verification properties considered in the literature have $\psi:=C^TY \geq 0$, where $C$ is a column vector and $Y = N(X),$ for $X \in \phi$. Most state-of-the-art complete verifiers use BaB to solve this problem. These techniques use an analyzer that computes the linear approximation of the network output $Y$ through a convex approximation of the problem domain. This linear approximation of $Y$ is used to perform the bounding step to show for the lower bound $\lb(C^TY)$ that $\lb(C^TY) \geq 0$. If the bounding step cannot prove the property, the verification problem is partitioned into subproblems using a branching heuristic $\hbranch$. The partitioning splits the problem space allowing a more precise convex approximation of the split subproblems. This leads to gains in the precision of $\lb$ computation. Various choices for the analyzer and the branching strategies exist which represent different trade-offs between precision and speed.

\Tool{} leverages a specification tree representation and novel algorithms to store and transfer the proof of the property from $N$ to $N^a$ for accelerating the verification on $N^a$. We show the workings of $\Tool{}$ through the following illustrative example.

\begin{figure}[t]
\centering
\includegraphics[width=13cm]{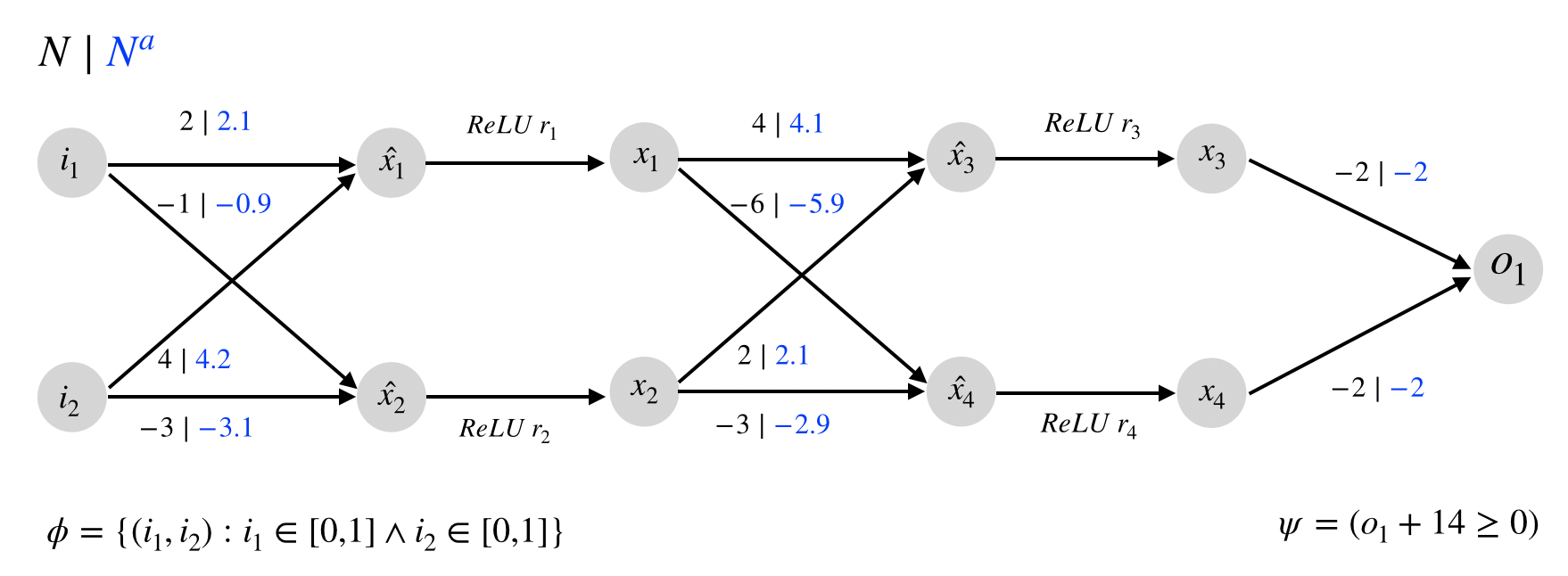}
\caption{Example original network $N$ and its perturbation $N^{a}$ (blue weights). Each layer consists of a linear function followed by the ReLU activation function. $\phi$ is the input specification and $\psi$ is the output specification.}
\label{fig:nn}
\vspace{-0.2in}
\end{figure}

\vspace{-0.07in}
\subsection{Illustrative Example}
We consider the two networks $N$ and $\perturbedNetwork$ with the same architecture as shown in Figure~\ref{fig:nn}. Most practical network updates result in network weight perturbations e.g., quantization, model repair, and fine-tuning. Network $\perturbedNetwork$ is obtained by updating (perturbing the weights) of network $N$. These networks apply ReLU activation at the end of each affine layer except for the final layer. The weights for the affine layers are shown on the edges. We consider the verification property $(\phi, \psi)$ such that $\phi = \{(i_1, i_2): i_1 \in [0, 1] \land i_2 \in [0, 1]\}$ and $\psi = (o_1 + 14 \geq 0)$. Let $\Reluset = \{r_1, r_2, r_3, r_4\}$ denote the set of ReLUs in the considered architecture. $\Reluset$ is a function of the architecture of the DNNs and is common for both $N$ and $\perturbedNetwork$. 

{\noindent \bf Branch and Bound:} 
We consider a complete verifier that uses a sound analyzer $\vbound$ based on the exact encoding of the affine layers and the common triangle linear relaxation \cite{bunel2020branch, bunel2020efficient, ehlers2017formal} for over-approximating the non-linear ReLU function. 
%
If due to over-approximation of the ReLU function, the analyzer cannot prove or disprove the property, the verifier partitions the problem by splitting the problem domain. The analyzer is more precise if it separately analyzes the split subproblems and merges the results. There are two main strategies for branching considered in the literature, input splitting \cite{anderson2020strong, wang2018neurify}, and ReLU splitting \cite{bunel2020branch, bunel2020efficient, ehlers2017formal, ferrari2022complete, depalma2021scaling}. We show \Tool{}'s effectiveness on both branching strategies in our evaluation (Section~\ref{sec:local}, Section~\ref{sec:acas}). However, for this discussion, we focus on ReLU splitting which is scalable \mbox{for the verification of high-dimensional inputs.}

{\noindent \bf ReLU splitting:}
 An unsolved problem is partitioned into two cases, where the cases assume the input $\hat{x}_i$ to ReLU unit $r_i$ satisfies the predicates $\hat{x}_i \geq 0$ and $\hat{x}_i < 0$ respectively. Splitting a ReLU $r_i$ eliminates the analyzer imprecision in the approximation of $r_i$. When we split all the ReLUs in $\Reluset$, the analyzer is exact. Nevertheless, splitting all $\Reluset$ is expensive as it requires $2^{|\Reluset|}$ analyzer bounding calls. The state-of-the-art techniques use the heuristic function $\hbranch$ to find the best ReLU to split at each step, leading to considerably scalable complete verification.
 
The branching function $\hbranch$ scores the ReLUs $\Reluset$ for branching at each unsolved problem to partition the problem. If $\Reluset' \subseteq \Reluset$ denotes the subset of ReLUs that are not split in the current subproblem, then the verifier computes $r = \argmax_{\Reluset'} \hbranch$ to choose the $r$ for the current split. $\hbranch$ is a function of the exact subproblem that it branches and hence depends on $\phi$, $\psi$, the network, and the branching assumptions made for the subproblem. However, for the purpose of this running example, we consider a simple constant branching heuristic $\hbranch$ that ranks $\hbranch(r_1) > \hbranch(r_3) > \hbranch(r_4) > \hbranch(r_2)$ independent of the subproblem and the network. This assumption is only for the illustration of our idea, we show in the evaluation (Section~\ref{sec:eval}) that \Tool{} can work with state-of-the-art branching heuristics \cite{ijcai2021p351, bunel2020branch}. 

\subsection{\Tool{} Algorithm}

{\noindent\bf Specification Tree:}
\Tool{} uses a rooted binary tree data structure to store the trace of splitting decisions during BaB execution. 
A specification split is a finer specification parameterized by the subset of ReLUs in $\Reluset$. 
The root node is associated with the specification $(\phi, \psi)$. 
All other nodes represent the specification splits obtained by splitting the problem domain recursively. Each internal node in the tree has two children, the result of the branching of the associated specification. 

The split decision can be represented as a predicate. For a ReLU $r_i$ with input $\hat{x}_i$, let $r_i^+ := (\hat{x}_i \geq 0)$ and $r_i^- := (\hat{x}_i < 0)$ denote the split decisions. A split of ReLU $r_i$ at node $\node$ creates two children nodes $n_l$ and $n_r$, each encoding the new specification splits.
Each edge in the specification tree represents the split decision made at the branching step. An edge connects an internal node with its child node, and we label it with the additional predicate that is assumed by the child subproblem. 
A split of ReLU $r$ at node $n$ adds nodes $n_l$ and $n_r$ that are connected with edges labeled with predicates $r_i^+$ and $r_i^-$ respectively. 
If $\spec{n} = (\phi', \psi)$ is the specification split at $\node$, then $\spec{\node_l} = (\phi' \land r^+, \psi)$ and $\spec{\node_r} = (\varphi' \land r^-, \psi)$.
The names of the nodes have no relation to the networks or the property, they are used for referencing a particular specification. However, the edges of the tree are tied to the network architecture through the labels. Although the specification tree is created as a trace of verification of a particular network $N$, it is only a function of the ReLU units in the architecture of $N$. This allows us to use the branching decisions in the specification tree for guiding the verification of any updated network $\perturbedNetwork$ that has the same architecture as $N$. 
We use $\lb_{N}(\node)$ to denote the lower bound $\lb(C^TY)$ obtained by the analyzer $\vbound$ on for the subproblem encoded by $\node$, on the network $N$. 

\begin{figure}[tbp]
\centering\vspace{-.05in}
\includegraphics[width=14cm]{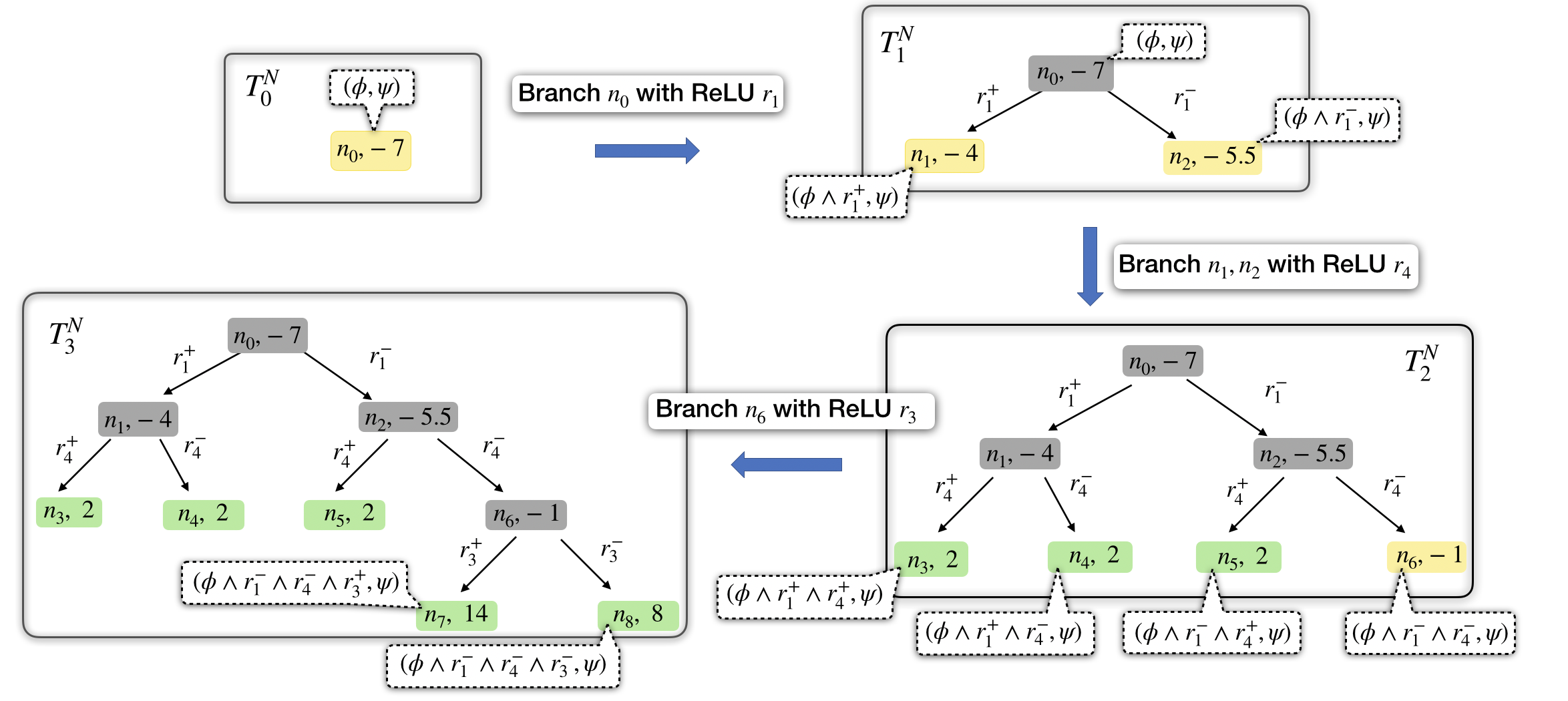}
\caption{Steps in Branch and Bound algorithm for complete verification of $N$. The nodes are labeled with a name and the $\lb_{N}(\node)$. The nodes in the specification tree are annotated with their specifications. The edges are labeled with the branching predicates. Each step in BaB partitions unsolved specifications in $T^{N}_i$ into specification splits in $T^{N}_{i+1}$. The proof is complete when all specification splits corresponding to the leaf nodes are solved.}
\label{fig:bab}
\vspace{-.15in}
\end{figure}

Figure~\ref{fig:bab} demonstrates the steps of BaB execution on $N$. Each node represents the specification refined by BaB. We use function $\lb_{N}(\node)$ to denote the $\lb(C^TY) = \lb(o_1 + 14)$ value obtained by the analyzer $\vbound$ at node $\node$. The specification is verified for the subproblem of $\node$ if the $\lb_{N}(\node) \geq 0$. 
If $\lb_{N}(\node) < 0$, the analyzer returns a counterexample~(CE). The CE is a point in the convex approximation of the problem domain and it may be possible that it is spurious, and does not belong to the concrete problem domain. If the CE is not spurious, the specification is disproved and the proof halts. But, if the CE is spurious then the problem is \mbox{unsolved, and it is further partitioned.}

In the first step, for the specification $(\phi, \psi)$ encoded by the root node $\node_0$, the analyzer computes $\lb_{N}(\node_0) = -7$, which is insufficient to prove the specification. Further, the CE provided by the analyzer is spurious, and thus the analyzer cannot solve the problem.  
The root node $n_0$ specification $(\phi, \psi)$ is split by ReLU split of $r_1$ chosen by the heuristic function $\hbranch$. Accordingly, in the specification tree, the node $n_0$ is split into two nodes $n_1$ and $n_2$, with the specification splits $(\phi \land r_1^+, \psi)$ and $(\phi \land r_1^-, \psi)$ respectively. This procedure of recursively splitting the problem and correspondingly updating the specification tree continues until either all the specifications of the leaf nodes are verified, or a CE is found. 
In the final specification tree ($T^{N}_3$ in this case), the leaf nodes are associated with the specifications that the analyzer could solve, and the internal nodes represent the specifications that the analyzer could not solve for network $N$. For BaB starting from scratch, each node in the specification tree maps to a specification that invoked an analyzer call in BaB execution. Figure~\ref{fig:bab} presents that the verifier successfully proves the property with a specification tree containing 9 nodes. Thus, the verification invokes the analyzer 9 times and performs 4 nodes \mbox{branchings for computing $\lb$.}

Figure~\ref{fig:pt1} presents the specification tree for $\perturbedNetwork$ at end of the verifying the property $(\phi, \psi)$. Although the $LB(C^TY)$ computed by the analyzer for each node specifications is different for $\perturbedNetwork$ compared to $N$, the final specification tree is identical for both networks. Our techniques in \Tool{} are motivated by our observation that the final specification tree for network $N$ and its updated version $N^a$ have structural similarities. Moreover, we find that for a DNN update that perturbs the network weight within a fixed bound, these trees are identical. We claim that there are two reasons for this: (i) the specifications that are solved by the analyzer for $N$ are solved by the analyzer for $\perturbedNetwork$ (specifications of the leaf nodes of the specification tree) and (ii) the specifications that are unsolved by the analyzer for $N$ are unsolved for $\perturbedNetwork$ (specifications of the internal nodes of the specification tree). In Section \ref{sec:perturb}, we provide theoretical bounds on the network perturbations such that these claims hold true (Theorem ~\ref{thm:perturb1}). Nevertheless, for networks obtained by perturbation beyond the theoretical bounds, the specification trees are still similar if not identical. In our evaluation, we observe this similarity for large networks with practical updates e.g., quantization (Section~\ref{sec:eval}). 


\begin{figure}[t]
\centering

\begin{subfigure}[b]{0.45\textwidth}
 \centering
 \includegraphics[width=\textwidth]{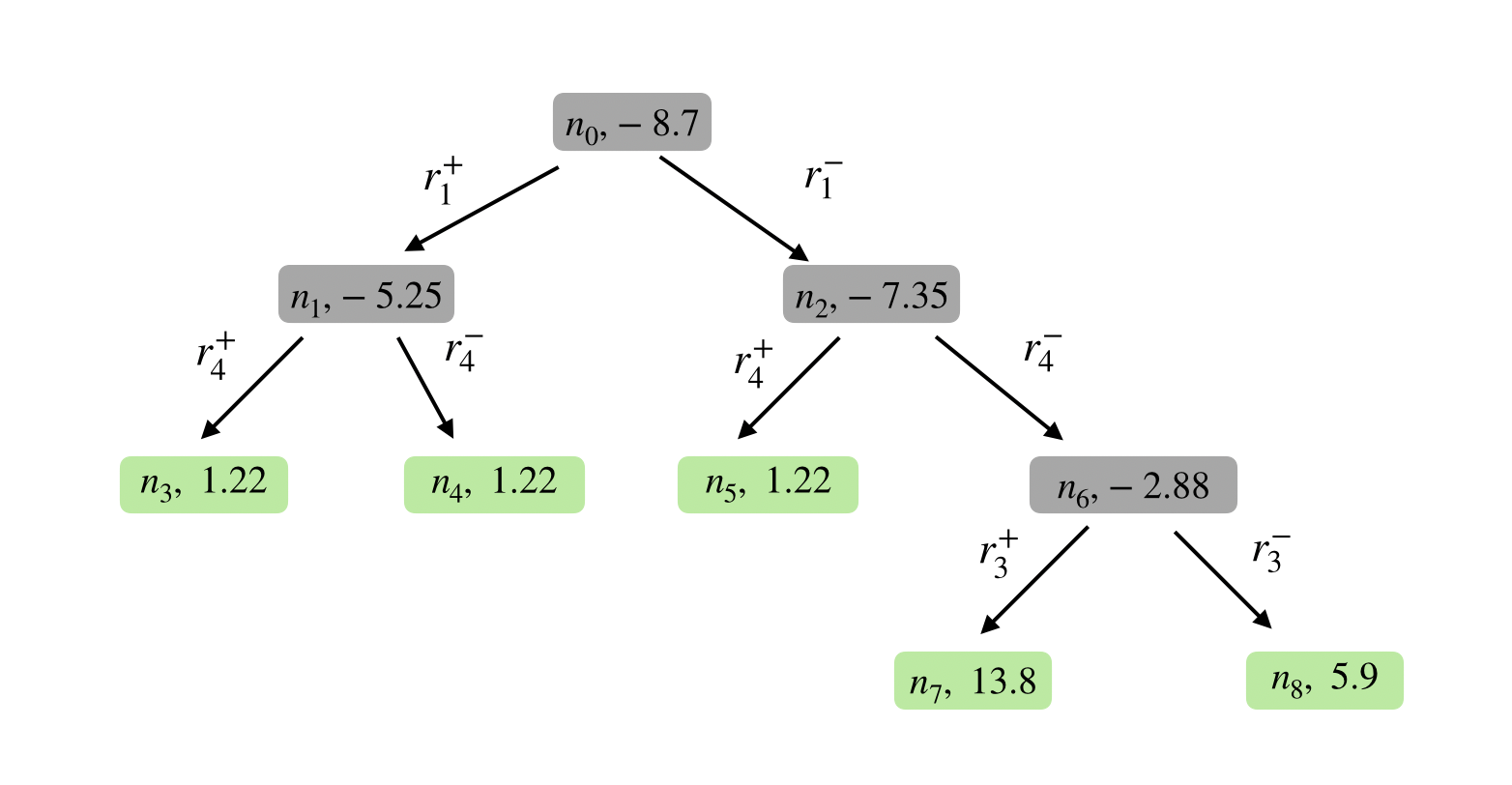}
 \caption{BaB specification tree for $\perturbedNetwork$. It requires 
 9 node boundings and 4 node branchings. }
 \label{fig:pt1}
\end{subfigure}
\hspace{10mm}
\begin{subfigure}[b]{0.45\textwidth}
 \centering
 \includegraphics[width=\textwidth]{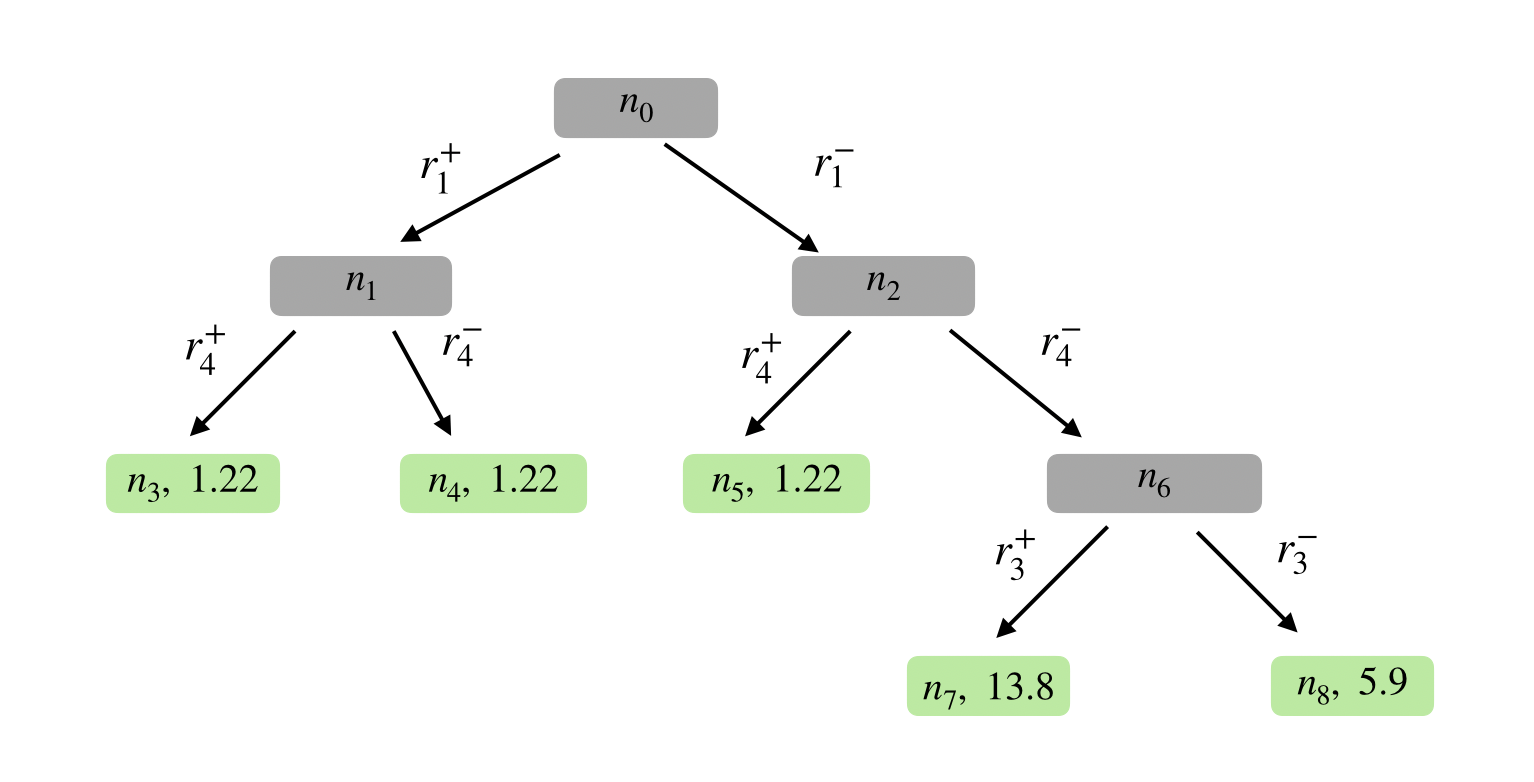}
 \caption{BaB specification tree for $\perturbedNetwork$ with reuse. It requires 
 5 node boundings and 0 node branchings.}
 \label{fig:pt2}
\end{subfigure}
\hspace{10mm}
\begin{subfigure}[b]{0.45\textwidth}
 \centering
 \includegraphics[width=\textwidth]{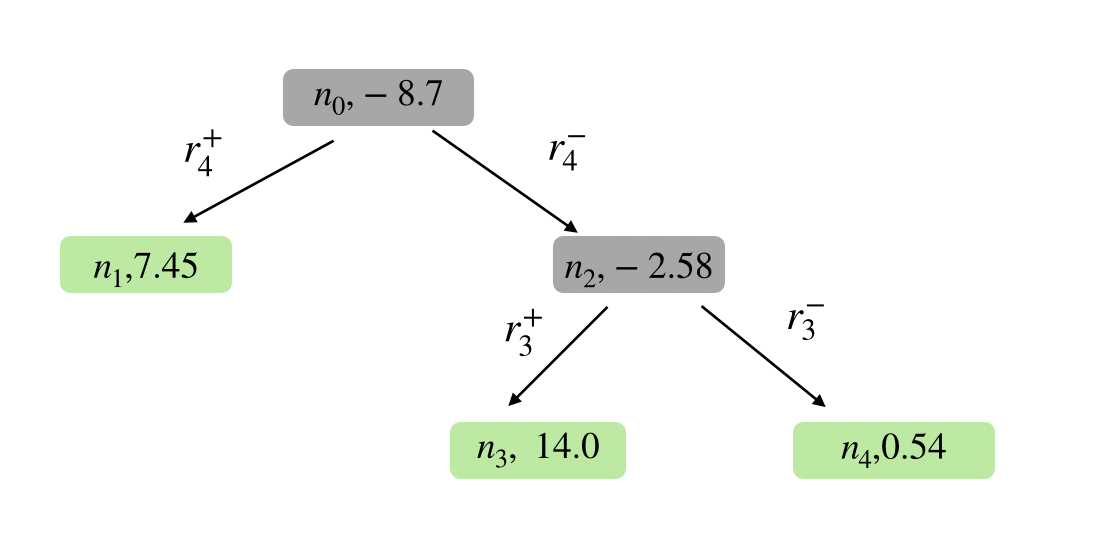}
 \caption{BaB specification tree for $\perturbedNetwork$ with reorder. It requires 
 5 node boundings and 2 node branchings.}
 \label{fig:pt3}
\end{subfigure}
\hspace{10mm}
\begin{subfigure}[b]{0.45\textwidth}
 \centering
 \includegraphics[width=\textwidth]{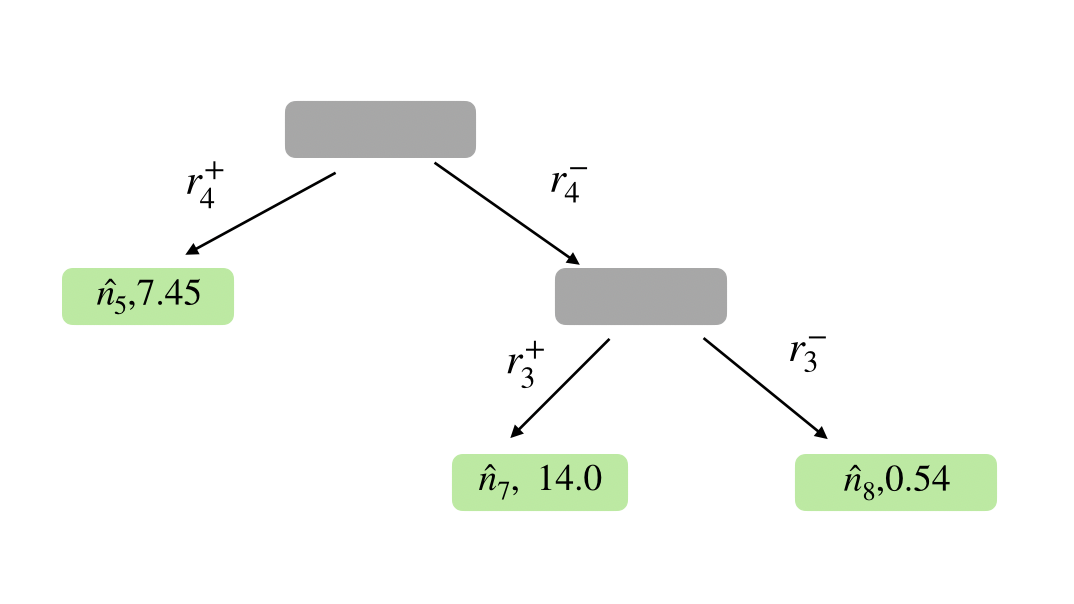}
 \caption{BaB specification tree for $\perturbedNetwork$ with \Tool{}. It requires 
 3 node boundings and 0 node branchings.}
 \label{fig:pt4}
\end{subfigure}
\hspace{10mm}

\vspace{-.2in}
\caption{BaB specification tree for various techniques proposed for incremental verification.}
\label{fig:scatter2}
\vspace{-.1in}
\end{figure} 

{\noindent \bf Reuse:} 
We first introduce our concept of specification tree reuse which uses $T^N_f$, the final tree after verifying $N$, as the starting tree $\Tinit$ for the verification of $\perturbedNetwork$. In contrast, the standard BaB verification starts with a single node tree that represents the unpartitioned initial specification $(\phi, \psi)$. In the reuse technique, \Tool{} starts BaB verification of $\perturbedNetwork$ from the leaves of $\Tinit= T^N_f$.
For our running example, analyzer $\vbound$ successfully verifies $\perturbedNetwork$ specifications for all the leaf nodes of the specification tree $\Tinit$ (Figure~\ref{fig:pt2}). We show that for any specification tree (created on the same network architecture), verifying the subproblem property on all the leaves of the specification tree is equivalent to verifying the main property $(\phi, \psi)$ (Lemma~\ref{lemma:invariance}). 
Verifying the property on $\perturbedNetwork$ from scratch requires 9 analyzer calls and 4 node branchings. However, with the reuse technique, we could prove the property with 5 analyzer calls corresponding to the leaves of $\Tinit$ and without any node branching. 
Theorem~\ref{thm:perturb1} guarantees that the specification of the leaf nodes should be verified on $\perturbedNetwork$ by the analyzer if the network perturbations are lower than a fixed bound. Although for larger perturbations, we may have to split leaves of $\Tinit$ further for complete verification, we empirically observe that the reuse technique is still effective to gain speedup \mbox{on most practical network perturbations.}

{\noindent \bf Reorder:}  
A split is more effective if it leads to fewer further subproblems that the verifier has to solve to prove the property. Finding the optimal split is expensive. Hence, the heuristic $\hbranch$ is used to estimate the effectiveness of a split, and to choose the split with the highest estimated effectiveness. Often the estimates are imprecise and lead to ineffective splits. We use $\lb_{N}(n)$ to give an approximation to quantifying the effectiveness of a split. We discuss this exact formulation of the observed effectiveness scores $\hbranch_{obs}$ in Section~\ref{sec:ivan}.
Our second concept in \Tool{} is based on our insight that if a particular branching decision is effective for verifying $N$ then it should be effective for verifying $\perturbedNetwork$. Likewise, if a particular branching decision is ineffective in the verification of $N$, it should be ineffective in verifying $\perturbedNetwork$. Based on this insight, we use the observed effectiveness score of splits in verifying $N$ to modify the original branching heuristic $\hbranch$ to an improved heuristic $\hbranch_\Delta$.  
$\hbranch_\Delta$ takes the weighted sum of original branching heuristic $\hbranch$ and observed effectiveness scores on $N$ denoted by $\hbranch_{obs}$. We formulate the effectiveness of a split and $\hbranch_\Delta$ in Section~\ref{sec:ivan}. For simplicity, in the running example, we rerank the ReLUs based on the observed effectiveness of the splits as $\hbranch_\Delta(r_4) > \hbranch_\Delta(r_3) > \hbranch_\Delta(r_2) > \hbranch_\Delta(r_1)$. Figure~\ref{fig:pt3} presents the specification tree for verifying $\perturbedNetwork$ with the updated branching heuristic $\hbranch_\Delta$ that requires 5 analyzer calls and 2 node branchings. Reorder technique starts from scratch with a different branching order $\hbranch_\Delta$ and it is incomparable in theory to the reuse technique. In Section~\ref{sec:ablation}, we observe \mbox{that reorder works better in most experiments.}

{\noindent \bf Bringing All Together:} Our main algorithm combines our novel concepts of specification tree reuse and reorder yielding larger speedups than possible with only reuse or reorder. Specification tree reuse and reorder are not completely orthogonal and thus combining them is not straightforward. 
Since in reuse we start verifying $\perturbedNetwork$ with the final specification tree $T^N_f$, the splits are already performed with the original order ($r_1, r_4, r_3, r_2$ in our example). Our augmented heuristic function $\hbranch_\Delta$ will have a limited effect if we reuse $\Tinit = T^N_f$, since the existing tree branches may already be sufficient to prove the property.

{\noindent \bf Constructing a Pruned Specification Tree:} It is difficult to predict the structure of the tree with augmented order. For instance, in our example, $N$ is verified with $r_1, r_4, r_3, r_2$ order and we have $T^N_f$ branched in that order. However, we cannot predict the final structure of the specification tree if branched with our augmented order $r_4, r_3, r_2, r_1$ without actually performing those splits from scratch (as it was done in Figure~\ref{fig:pt3}). 

\begin{figure}[t]
\vspace{-0.15in}
\centering
\includegraphics[width=12cm]{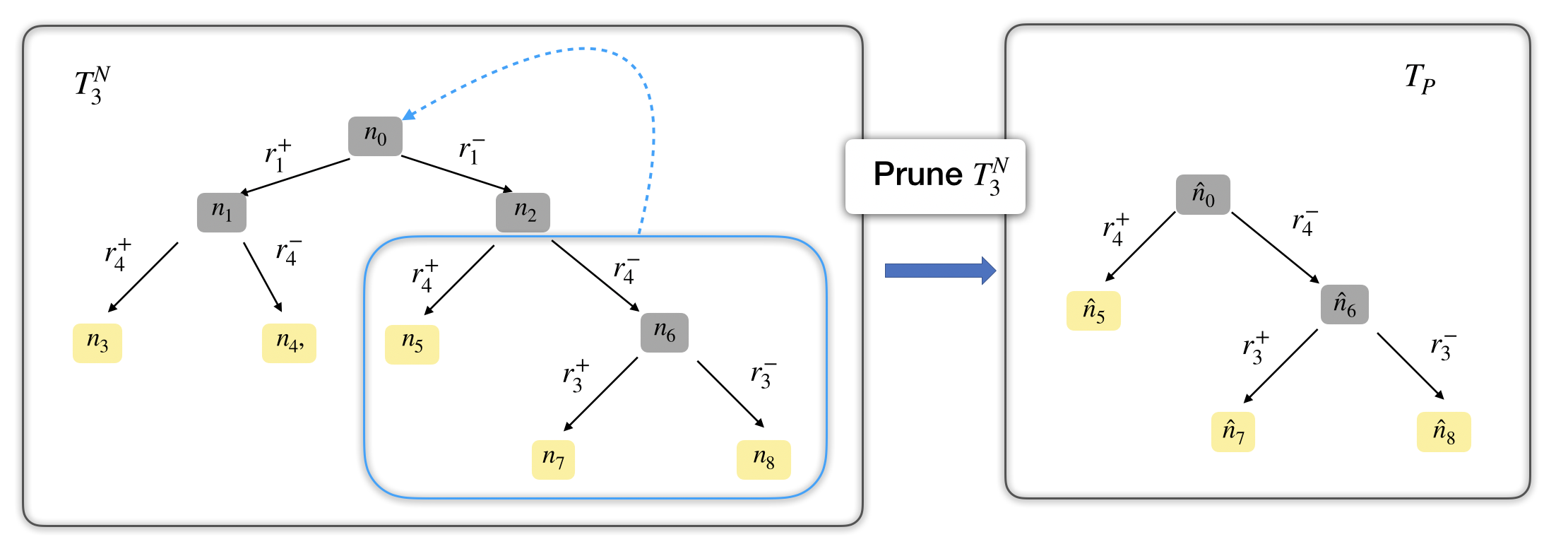}
\vspace{-0.1in}
\caption{\Tool{} removes the ineffective split $r_1$ at $n_0$ and construct a new specification tree $\Tprune$. }
\label{fig:pruneexample}
\vspace{-0.2in}
\end{figure}

We solve this problem with our novel pruning operation that removes ineffective splits from $T^N_f$ and constructs a new compact tree $\Tprune$. 
Figure~\ref{fig:pruneexample} shows the construction of pruned tree $\Tprune$ for our running example. We remove the split $r_1$ at $\node_0$ as it is less effective. Removing $r_1$ from $T^N_3$ also eliminates the nodes $\node_1$ and $\node_2$. The subtrees rooted at $\node_1$ and $\node_2$ are the result of split $r_1$. If we undo the split $r_1$ at node $\node_0$, then $\node_0$ should follow the branching decisions taken by one of its children. For this, we can choose either the subtree of $\node_0$ or $\node_1$, and attach it to $\node_0$. We describe the exact method of choosing which subtree to keep in Section~\ref{sec:ivan}. For this example, our approach chooses to keep the subtree of node $\node_2$ and eliminates the subtree at node $\node_1$. The pruning procedure leads to the discarding of entire subtrees creating a tree with fewer leaf nodes (leaf nodes $\node_3, \node_4$ are deleted in the example along with internal nodes $\node_1, \node_2$). Consequently, we obtain a more compact tree with only influential splits in the specification tree. 

We start the verification of $\perturbedNetwork$ from the leaf nodes of the pruned tree i.e. $\Tinit = \Tprune$. For our running example specification splits of all leaf nodes of $\Tprune$ are verified by the analyzer and no further splitting is needed. Figure~\ref{fig:pt4} presents the final specification tree in case we initialize the proof with the compact tree obtained from the \Tool{} algorithm. 
We show the time complexity of incremental verification in Section~\ref{sec:formulation}. 
For the running example, the incremental proof requires only 3 analyzer calls and no branching calls, and it is a significant reduction to the 9 analyzer calls and 5 node branchings performed by the baseline starting from scratch.

\section{Preliminaries}
\label{sec:pre}

In this section, we provide the necessary background on complete neural network verification. 
 
\vspace{-0.07in}
\subsection{Neural Network Verification}\label{sec:netverbg}
\noindent \textbf{Neural Networks} 
Neural networks are functions $\Norig :\mathbb{R}^{n_0} \to \mathbb{R}^{n_l}$. In this work, we focus on layered neural networks obtained by a sequential composition of $l$ layers $N_1:\mathbb{R}^{n_0} \to \mathbb{R}^{N_1}, \dots ,N_l:\mathbb{R}^{n_{l-1}} \to \mathbb{R}^{n_{l}}$. \new{Each layer $N_i$ applies an \emph{affine function} (convolution or linear function) followed by a non-linear activation function to its input. The choices for non-linear activation functions are ReLU, sigmoid, or tanh. $ReLU(x) = max(0, x)$ is most commonly used activation function}.
In Section~\ref{sec:main}, we focus on the most common BaB verifiers that partition the problems using ReLU splitting in ReLU networks. 
The $i$-th layer of each network $\OrgNetwork_{i}: \mathbb{R}^{\Dimension_i} \longrightarrow \mathbb{R}^{\Dimension_{i+1}}$ is defined as $\OrgNetwork_i(x) = \Relu(A_iX + B_i)$ where $i \in [\Layers]$.

At a high level, neural network verification involves proving that all network outputs corresponding to a chosen set of inputs satisfying the input specification $\phi$ satisfy a given logical property $\psi$. 
We first define the input and output specifications that we consider in this work:

\begin{definition}[Input specification]
For a neural network $\Norig :\mathbb{R}^{n_0} \to \mathbb{R}^{n_l}$, $\inpreg$ is a connected region and $\inpreg \subseteq \mathbb{R}^{n_0}$. \new{\textbf{Input specification} $\phi:\mathbb{R}^{n_0} \to \{true, false\}$ is a predicate over the input region $\inpreg$.} 
\end{definition}

\begin{definition}[output specification]
For a neural network with $n_l$ neurons in the output layer. \textbf{output specification} $\psi:\mathbb{R}^{n_l} \to \{true, false\}$ is a predicate over the output region. 
\end{definition}

The output property $\psi$ could be any logical statement taking a truth value true or false. In our paper, we focus on properties that can be expressed as Boolean expressions over linear forms. Most DNN verification works consider such properties. 

\vspace{-0.07in}
\begin{equation}
    \label{def:linearout}
    \psi(Y) = (C^T Y \geq 0) 
\end{equation}

We next define the verification problem solved by the verifiers: 

\begin{definition}[Verification Problem]
The \textbf{neural network verification} problem for a neural network $N$, an input specification $\phi$ and a logical property $\psi$ is to prove whether $\ \forall X \in \new{\inpreg}. \ \psi(N(X)) = \textit{true}$ or provide a counterexample otherwise.
\end{definition}

A complete verifier always verifies the property if it holds or returns a counterexample otherwise. Formally, it can be defined as:


\begin{definition} [Complete Verifier]
\label{def:sound}
A \textbf{complete verifier} $V$ for an input specification $\phi$, a neural network $\Norig$, an output property $\psi$ satisfies the following property:
\[
    V(\phi, \psi, N) = \ver  \Longleftrightarrow \forall X \in \new{\inpreg} . \psi(N(X)) = \textit{true}
\]
\end{definition}

\subsection{Branch and Bound for Verification}\label{sec:bab}
In this Section, we discuss the branch and bound techniques for complete verification of DNNs. The BaB approach in these techniques use a divide-and-conquer algorithm to compute the $\lb(C^TY)$ for proving $(C^T Y \geq 0)$ (Eq.~\ref{def:linearout}). We next discuss the bounding and branching steps in BaB techniques. 





{\noindent \bf Bounding:} 
The bounding step uses an analyzer to find a lower bound $\lb(C^TY)$. 
In complete verifiers, the analyzers are exact for linear functions (e.g., DeepZ \cite{singh2018fast}, DeepPoly \cite{singh2019abstract}).
However, they over-approximate the non-linear activation function through a convex over-approximation. 
We define these sound analyzers as:

\begin{definition} [Sound Analyzer]
\label{def:sound}
A \textbf{sound analyzer $\vbound$} on an input specification $\phi$, a DNN $\Norig$, an output property $\psi$ returns \ver, \unknown, or \counterex. It satisfies the following properties:
\begin{align*}
    & \vbound(\phi, \psi, N) = \ver  \implies \forall X \in \new{\inpreg} . \psi(N(X)) = \mathit{true} \\
    & \vbound(\phi, \psi, N) = \counterex  \implies \exists X \in \new{\inpreg} . \psi(N(X)) = \mathit{false}
\end{align*}
\end{definition}

{\noindent \bf Branching:} If the analyzer cannot prove a property, the BaB verifier partitions the problem into easier subproblems to improve analyzer precision.
\new{
Algorithm~\ref{alg:bab} presents the pseudocode for the BaB verification. 
The algorithm maintains a $\unsolvedList$ list of problems that are currently not proved or disproved. 
It initializes the list with the main verification problem.  
Line~\ref{line:bab_bound} performs the bounding step in the BaB algorithm using the analyzer $\vbound$. 
For simplicity, we abuse the notation and use $\vbound(prob, N)$ for denoting the analyzer output instead of $\vbound(\phi, \psi, N)$. Here, the $prob$ encapsulates the input and output specifications $\phi, \psi$. 
Line~\ref{line:bab_branch} partitions the unsolved problem into subproblems. 
The algorithm halts when either the $\vbound$ finds a counterexample on one of the subproblems or the list of unsolved problems is empty. 
}
There are two common branching strategies for BaB verification, input splitting and ReLU splitting, which we describe next. 

\new{
\begin{algorithm}
\small
\caption{\new{Branch and Bound}}\label{alg:bab}
\begin{algorithmic}[1]
\new{
\Function{BaB}{$N, problem$}
\State $\unsolvedList \gets [(problem)]$
\While{$\unsolvedList$ is not empty}
\For {$prob \in \unsolvedList$}
\State $status[prob] = \vbound(prob, N)$ \label{line:bab_bound} \Comment{Bounding step}
\EndFor

\For {$prob \in \unsolvedList$}

\If {$status[prob] =$ \ver}
\State $\unsolvedList.remove(prob)$ \label{line:bab_ver} \Comment{Remove verified subproblems}
\EndIf

\If {$status[prob] =$ \counterex} \label{line:bab_counter}
\State \textbf{return} \counterex \text{ } for $prob$ \Comment{Return if a counterexample is found}
\EndIf

\If {$status[prob] =$ \unknown} \label{line:bab_unknown}
\State $\unsolvedList.remove(prob)$
\State $[\text{subprob}_1, \text{subprob}_2] \gets \text{split}(\text{prob})$ \Comment{Branching step} \label{line:bab_branch}
\State $\unsolvedList.insert(\text{subprob}_1, \text{subprob}_2)$
\EndIf
\EndFor
\EndWhile
\State \textbf{return} $\text{\ver}$
\EndFunction
}
\end{algorithmic}
\end{algorithm}
}

{\noindent \bf Input Splitting:} In input splitting, the input region $\inpreg$ for verification is partitioned. The typical choice is to cut a selected input dimension in half while the rest of the dimensions are unchanged. The dimension to cut is decided by the branching strategy used. This technique is known to be $\delta$-complete for any activation function \cite{Pailoor19}, but does not scale for high-dimensional input space. In many computer vision tasks, the input is an image with 1000s of pixels. Thus, a high-dimensional perturbation region on such input cannot be \mbox{branched efficiently for fast verification.}

{\noindent \bf ReLU Splitting:} State-of-the-art techniques that focus on verifying DNNs with high-dimensional input and ReLU activation, use ReLU splitting. 
We denote a ReLU unit for $i$-th layer and $j$-th index as a function $x_{i, j} = \max(\hat{x}_{i, j}, 0)$, where $\hat{x}_{i, j}$ and $x_{i, j}$ are the pre-activation and post-activation values respectively. The analyzer computes lower bounds $lb$ and upper bounds $ub$ for each intermediate variable in the DNN. If $lb(\hat{x}_{i,j}) \geq 0$, then the ReLU unit simply acts as the identify  function $x_{i, j} = \hat{x}_{i, j}$. If $ub(\hat{x}_{i,j}) \leq 0$, then the ReLU unit operates as a constant function $x_{i, j} = 0$. In both of these cases, the ReLU unit is a linear function. However, if $lb(\hat{x}_{i,j}) < 0 < ub(\hat{x}_{i,j})$, we cannot linearize the ReLU function exactly. We call such ReLU units ambiguous ReLUs. In ReLU splitting, the unsolved problem is partitioned into two subproblems such that one subproblem assumes $\hat{x}_{i,j} < 0$ and the other assumes $\hat{x}_{i,j} \geq 0$. This partition allows us to linearize the ReLU unit in both subproblems leading to a boost in the overall precision of the analyzer. The heuristic used for selecting which ReLU to split significantly impacts the verifier speed.

\new{
{\noindent \bf BaB for Other Activation Functions:}
BaB-based verification can work with the most commonly used activation functions (tanh, sigmoid, leaky ReLU).
\begin{enumerate}
    \item For piecewise linear activation functions such as leaky ReLU, activation splitting approaches (e.g, ReLU splitting)  can be used for complete verification. 
    \item For other activation functions (tanh, sigmoid), BaB with activation splitting cannot yield complete verification but can be used to improve the precision of sound and incomplete verification \cite{muller2020neural, DBLP:journals/corr/abs-1709-09130}.
    \item Although input splitting is less efficient in the aforementioned cases for high dimensional DNN inputs, it can be applied with any activation function (tanh, sigmoid, ReLU, leaky ReLU). 
\end{enumerate}
}






\vspace{-0.07in}
\section{Incremental Verification}
\label{sec:main}
In this section, we describe our main technical contributions and the \Tool{} algorithm. We first formally define the specification tree structure used for incremental verification (Section~\ref{sec:prooftree}). Next, we formulate the problem of incremental verification (Section ~\ref{sec:formulation}). In Section~\ref{sec:ivan}, we illustrate the techniques used in our algorithm. We characterize the effectiveness of our technique by computing a class of networks for which our incremental verification is efficiently applicable in  Section~\ref{sec:perturb}. 

\vspace{-0.07in}

\subsection{Specification Tree for BaB} 
\label{sec:prooftree}
\Tool{} uses the specification tree to store the trace of splitting decisions that the BaB verifier makes on its execution. 
A specification tree can be used for any BaB branching method (e.g, input splitting), but without loss of generality, our discussion focuses on ReLU splitting. Let $\mathcal{N}$ denote the class of networks with the same architecture, and let $\mathcal{R}$ denote the set of ReLUs in this architecture. The specification tree captures the ReLU splitting decisions and the split specifications in the execution of BaB for a property $(\phi, \psi)$, where we define $(\phi,\psi):=\phi \to \psi$.

For a ReLU $r_i$ with input $\hat{x}_i$, let $r_i^+ := (\hat{x}_i \geq 0)$ and $r_i^- := (\hat{x}_i < 0)$. We define a split decision as:
\vspace{-0.07in}

\begin{definition} [Split Decision]
For a  ReLU $r \in \mathcal{R}$, a split decision is $r^? \in \{r^+, r^-\}$ where $r^?$ is assigned the predicate $r^+$ or $r^-$.
\end{definition}
\vspace{-0.07in}

A specification split of $(\phi, \psi)$ is a specification stronger than $(\phi, \psi)$ parameterized by the subset of ReLUs in $\Reluset$ and the corresponding split decisions. Formally,
\begin{definition} [Specification Split]
For a set of ReLUs $\Reluset' = \{r_1, r_2 \dots r_k\} \subseteq \Reluset$, and ReLU split decision $r_i^? \in \{r_i^-, r_i^+\}$ for each $r_i$, the corresponding specification split of $(\phi,\psi)$ is $(\phi \land r_{1}^? \land r_{2}^? \land \dots r_{k}^?, \psi)$. 
\end{definition}
Since $\emptyset \subseteq \Reluset$, $(\phi, \psi)$ is a split specification of itself.
Let $\mathcal{S}$ denote the set of specification splits that can be obtained from $(\phi, \psi)$.  
Each node $\node$ in the tree encodes a specification split in $\mathcal{S}$. 
Each edge in the specification tree is labeled with a ReLU split decision $r^?$.
Let $\nodes{T}$ denote the nodes of the tree $T$ and $\leaves{T}$ denote the leaves of the tree $T$.

\noindent \textbf{Mapping Nodes to Specification Splits:}
The specification associated with the root node is $(\phi, \psi)$.
The function $\children{\node}$ maps a node $\node$ to either the pair of its children or $\emptyset$ if $\node$ has no children.
If $\node_l$ and $\node_r$ are the children of node $\node$ and $\spec{n} = (\phi', \psi)$ is the specification split at $\node$, then $\spec{\node_l} = (\phi' \land r^+, \psi)$ and $\spec{\node_r} = (\varphi' \land r^-, \psi)$. For the specifications $\spec{n}, \spec{\node_l}, \spec{\node_r}$ the \mbox{following statement holds:}
\begin{equation}
\label{eq:split}
    (\spec{\node_l} \land \spec{\node_r}) \Longleftrightarrow  \spec{\node}
\end{equation}
This relationship implies that verifying the parent node specification is equivalent to verifying the two children node's specifications.
Formally, we can now define the specification tree as:
\begin{definition} [specification tree]
\label{def:proof_tree}
Given a set of ReLU $\mathcal{R}$, a rooted full binary tree $T$ is a~\textbf{specification tree}, if for a node $n \in \nodes{T}$, and nodes $n_l, n_r \in \children{n}$, edge $(n, n_l)$ is labeled with predicate $r^+$ and edge $(n, n_r)$ is labeled with predicate $r^-$, for $r \in \mathcal{R}$. 
\end{definition}
%



\begin{wrapfigure}{L}{0.5\textwidth}
\vspace{-0.4in}

\begin{minipage}{0.5\textwidth}
\small
\setstretch{1.1}
\begin{algorithm}[H]
    \small
    \caption{$\add$ operation}
    \label{alg:tree_operation1}
    \centering
    \begin{algorithmic}[1]
        \Function{\add}{$T, n, r$} \\
        \hspace*{\algorithmicindent} \textbf{Input:} Specification tree $T$, a leaf node $n \in \leaves{T}$, a ReLU $r \in \mathcal{R}$ for splitting the node\\
        \hspace*{\algorithmicindent} \textbf{Output:} returns newly added nodes
        \State $n_l \gets$ Add\_Child$(n, r^+)$
        \State $n_r \gets$ Add\_Child$(n, r^-)$
        \State \textbf{return} $n_l, n_r$
        \EndFunction
    \end{algorithmic}
\end{algorithm}
\vspace{-0.25in}
\end{minipage}
\vspace{-0.2in}

\end{wrapfigure}

 BaB uses a branching function $\hbranch$ for choosing the ReLU to split. We define this branching function in terms of the node $\node$ of the specification tree as:
\begin{definition} [Branching Heuristic]
\label{def:branch}
Given a set of ReLU $\mathcal{R}$, a network $N$, and a node $n$ in the specification tree, if $\mathcal{P} \subseteq \mathcal{R}$ denote the set of ReLUs split in the path from the root node of the specification tree to $n$ then the branching heuristic $\hbranch(N, n, r)$ computes a score $h \in \mathbb{R}$ estimating the effectiveness of ReLU $r \in \mathcal{R}/\mathcal{P}$ for splitting the specification ($\spec{n}$) of the node $n$.
\end{definition}


We next state the split operation on a specification tree. Algorithm~\ref{alg:tree_operation1} presents the steps in the split operation.

\noindent$\bullet$ \textit{Split Operation:} Every ReLU split adds two nodes to the specification tree at a given leaf node $n$. The BaB algorithm chooses the ReLU $\argmax_{r \in \Reluset/\mathcal{P}} \hbranch(N, \node, r)$ to split at node $\node$ using the heuristic function.




\vspace{-0.07in}
\subsection{Incremental Verification: Problem Formulation}
\label{sec:formulation}
Give a set of networks $\mathcal{N}$ with the same architecture with a set of ReLUs $\mathcal{R}$, $\treeset$ be the set of all specification trees defined over $\mathcal{R}$.
There exists a partial order $(<)$ on $\treeset$ through standard subgraph relation. 
BaB execution on a network $N \in \mathcal{N}$ traces a sequence of trees $T_0, T_1 \dots T_f \in \treeset$ such that $T_i < T_{i+1}$. 
It halts with the final tree $T_f$ when it either verifies the property or finds a counterexample. The construction of $T_{i+1}$ from $T_i$ depends on the branching function $\hbranch$ (Definition~\ref{def:branch}). 



\noindent {\bf Incremental Verification:} The incremental verification problem is to efficiently reuse the information from the execution of verification of network $N$ for the faster verification of its updated version $\perturbedNetwork$. 
Standard BaB for verification of $\perturbedNetwork$ starts with a single node tree while the incremental verifier starts with a tree $\Tinit \in \treeset$ that is not restricted to be a tree with a single node. 
We modify the final specification tree $T^N_f$ from the verification of $N$ to construct $\Tinit$.
%
The branching heuristic $\hbranch_\Delta$ for incremental verification is derived from the branching heuristic $\hbranch$ based on the efficacy of various branching decisions made during the proof for $N$. 
Formally, the complete incremental verifier we propose is defined as:
\begin{algorithm}[h]
\small
\setstretch{1.1}
\caption{Verifying Perturbed Network}
\label{alg:incver}
\begin{flushleft}
\textbf{Input:} $\perturbedNetwork$, property ($\phi$,$\psi$), Initial specification tree $\Tinit$, branching heuristic $\hbranch_{\Delta}$ \\
\textbf{Output:}
 \ver if the specification ($\phi$,$\psi$) is verified, otherwise a \counterex
\end{flushleft}

\begin{algorithmic}[1] 
\State $T^{\perturbedNetwork} \gets$ Initialize $T^{\perturbedNetwork}$ as $\Tinit$
\State $\activeList =$ $\leaves{\Tinit}$ \label{line:active} \Comment{Initialize active list as $\leaves{\Tinit}$}

\While{$\activeList$ is not empty} \label{line:loop} 
\For {$n \in \activeList$}
\State $status[n] \gets \vbound(n)$ \label{line:analyzer} \Comment{Bounding step}
\EndFor

\For {$n \in \activeList$}

\If {$status[n] =$ \ver}
\State $\activeList.remove(n)$ \label{line:ver} \Comment{Remove verified nodes}
\EndIf

\If {$status[n] =$ \counterex} \label{line:counter}
\State $\activeList.empty()$
\State \textbf{return} \counterex for $n$ \Comment{Return if a counterexample is found}
\EndIf

\If {$status[n] =$ \unknown} \label{line:unknown}
\State $\activeList.remove(n)$
\State $r_{chosen} \gets \argmax_{r \in \Reluset} \hbranch_{\Delta}(N, n, r)$ \Comment{Use $\hbranch_{\Delta}$ to choose the split ReLU}
\State $n_l, n_r \gets$ \add$(T^{\perturbedNetwork}, n, r_{chosen})$ \Comment{Branching step}
\State $\activeList.insert(n_l, n_r)$
\EndIf

\EndFor
\EndWhile

\State \textbf{return} \ver
\end{algorithmic}
\end{algorithm}
\vspace{-0.1in}

\begin{definition} [Complete and Incremental Verifier]
\label{def:incverifier}
A \textbf{Complete and Incremental Verifier} $V_{\Delta}$ takes a neural network $\perturbedNetwork$, an input specification $\phi$, an output property $\psi$, analyzer $\vbound$, the branching heuristic $\hbranch_{\Delta}$ and the initial tree $\Tinit$. $V_\Delta(\perturbedNetwork,  \phi, \psi, \Tinit, \hbranch_{\Delta})$ returns $\ver$ if $\perturbedNetwork$ satisfies the property $(\phi, \psi)$, otherwise, it returns a \counterex.
\end{definition}

Algorithm~\ref{alg:incver} presents the incremental verifier algorithm for verifying the perturbed network. It takes $\hbranch_{\Delta}$ and $\Tinit$ as input. 
It maintains a list of active nodes which are the nodes corresponding to the specifications that are yet to be checked by the analyzer. It initializes the list of active nodes with leaves of tree $\Tinit$ (line~\ref{line:active}). The main loop runs until the active list is empty (line~\ref{line:loop}) or it discovers a counterexample (line~\ref{line:counter}). At each iteration, it runs the analyzer on each node in the active list (line~\ref{line:analyzer}). The nodes that are $\ver$ are removed from the list (line~\ref{line:ver}), whereas the nodes that result in $\unknown$ are split. The new children are added to the active list (line~\ref{line:unknown}).

{\noindent \bf Optimal Incremental Verification:} 
We define the partial function $\Timeb: \treeset \times \treeset \rightharpoondown
 \mathbb{R}$, $\Timeb(\Tinit, T^{\perturbedNetwork}_{f})$ for a fixed complete incremental verifier $V_\Delta$ as the time taken by $V_\Delta$ that starts from $\Tinit$ and halts with the final tree $T^{\perturbedNetwork}_{f}$.  
 $\Time_{h}(\hbranch,\hbranch_{\Delta})$ and $\Time_{t}(T^{N}_f,\Tinit)$ are the time for constructing $\hbranch_{\Delta}$ from \hbranch,     and $\Tinit$ from $T^{N}_f$ respectively.
We pose the optimal incremental verification problem as an optimization problem of finding the best $\hbranch_\Delta, \Tinit$ such that the time of incremental verification is minimized. 
Formally, we state the problem as:
%
\begin{equation}
\label{eq:opt_inc}
    \argmin_{\hbranch_{\Delta}, \Tinit} \big[ \Timeb(\Tinit, T^{\perturbedNetwork}_f) + \Time_{h}(\hbranch,\hbranch_{\Delta}) + \Time_{t}(T^{N}_f,\Tinit) \big]
\end{equation}
The search space for $\Tinit$ is exponential in terms of $\Reluset$, and the search space for $\hbranch_{\Delta}$ is infinite. Further, $\Timeb$ is a complicated function of $\hbranch_{\Delta}, \Tinit$ that does not have a closed-form formulation. As a result, it is not possible to find an optimal solution. 

{\noindent \bf Simplifying Assumptions:}
To make the problem tractable we make a simplifying assumption that for all networks with the same architecture, each branching and bounding step on each invocation takes a constant time $\tbr$ and $\tbo$ respectively. 
We can now compute $\Timeb(\Tinit, T^{\perturbedNetwork}_f)$ as: 

\begin{restatable}{theorem}{timeinc}($\Timeb{}$ for incremental verification).
\label{thm:timeinc}
 If the incremenatl verifier $V_\Delta$ halts with the final tree $T^{\perturbedNetwork}_f$, then 
 $\Timeb(\Tinit, T^{\perturbedNetwork}_f) = (\tbo + \tbr) \cdot \Big(|\nodes{T^{\perturbedNetwork}_f}| + \frac{1-|\nodes{\Tinit}|}{2} \Big) - \tbr \cdot |\leaves{T^{\perturbedNetwork}_f}|$.  
\end{restatable}

The proof of the theorem is in Appendix~\ref{sec:proofs}.

In this work, we focus on a class of algorithms for which the preprocessing times $\Time_{h}(\hbranch,\hbranch_{\Delta})$ and  $\Time_{t}(T^{N}_f,\Tinit)$ are $<< \Timeb(\Tinit, T^{\perturbedNetwork}_f)$. Furthermore, we also focus on branching heuristics used in practice where $\tbr << \tbo$. 
Equation~\ref{eq:opt_inc} simplifies to finding $\hbranch_{\Delta}$ and $\Tinit$ such that the following expression $\Timeb(\Tinit, T^{\perturbedNetwork}_f) = \tbo \cdot \Big(|\nodes{T^{\perturbedNetwork}_f}| + \frac{1-|\nodes{\Tinit}|}{2}\Big)$
is minimized. Rewriting and ignoring the constant term we get 
\begin{align}
\label{eq:simp}
    \Timeb(\Tinit, T^{\perturbedNetwork}_f) =\tbo \cdot \Bigg(\frac{|\nodes{T^{\perturbedNetwork}_f}|-|\nodes{\Tinit}|}{2} + \frac{|\nodes{T^{\perturbedNetwork}_f}|}{2} \Bigg)
\end{align}






\subsection{\Tool{} Algorithm for Incremental Verification}
\label{sec:ivan}
We describe the novel components of our algorithm and present the full workflow in Algorithm~\ref{alg:algorithm_main}. 
Our first technique called reuse focuses on minimizing $|\nodes{T^{\perturbedNetwork}_f}| - |\nodes{\Tinit}|$ in Equation~\ref{eq:simp}.
Our second reorder technique focuses on minimizing $|\nodes{T^{\perturbedNetwork}_f}|$. 
The $\hbranch_{\Delta}, \Tinit$ obtained by reuse and reorder are distinct. \Tool{} algorithm combines these distinct solutions, to reduce $\Timeb(\Tinit, T^{\perturbedNetwork}_f)$. 

\noindent{\bf Reuse:} This technique is based on the observation that the BaB specification trees should be similar for small perturbations in the network. Accordingly, in the method, we use the final specification tree for $N$ as the initial tree for the verification of $\perturbedNetwork$ i.e. $\Tinit = T^{N}_{f}$, and keep the $\hbranch_{\Delta} = \hbranch$ unchanged.  
We formally characterize a set of networks obtained by small perturbation for which $\Tinit = T^{N}_{f}$ is sufficient for verifying $\perturbedNetwork$ without any further splitting in Section~\ref{sec:perturb}. 

\noindent{\bf Reorder:} 
Reorder technique focuses on improving the branching heuristic $\hbranch$ such that it reduces $|\nodes{T^{\perturbedNetwork}_f}|$, and $\Tinit$ is single node tree with $\node_0$ encoding the specification $(\phi, \psi)$. 
If we start $\Tinit=T^{N}_{f}$, $|\nodes{T^{\perturbedNetwork}_f}|$ is at least $|\nodes{T^{N}_{f}}|$, and thus, we start $\Tinit$ from scratch allowing the technique to minimize $|\nodes{T^{\perturbedNetwork}_f}|$. 
We create a branching function $\hbranch_\Delta$ from $\hbranch$ with the following two changes. (i) The splits that worked effectively for the verification of the $N$ should be prioritized. (ii) The splits that were not effective should be deprioritized.
To formalize the effectiveness of splits, we define the $\lb_{N}(n)$ as the lower bound computed by the analyzer $\vbound$ on the network $N$ for proving the property $\spec{n}$ encoded by the node $n$.
Further, using the function $\lb_{N}$ we define an improvement function $I_N$ represents the effectiveness of a ReLU split at a specific node as: 
\begin{equation}
\label{eq:improve}
    I_N(n, r) = \min(\lb_{N}(n_r)-\lb_{N}(n), \lb_{N}(n_l)-\lb_{N}(n)) 
\end{equation}
where $n_l, n_r \in \children{\node}$ in the specification tree $T^N_f$.
We use $I_N$ to define the observed effectiveness $\hbranch_{obs}(r)$ from a split $r$ on the entire specification tree for $N$.
It is defined as the mean of the improvement over each node where split $r$ was made. 
Let $Q \subset \nodes{T^N_f}$ denote a set of nodes where split $r$ was made. Then, 
\begin{equation}
    \hbranch_{obs}(r) = \frac{\sum_{n \in Q}  I_N(n, r)}{|Q|}.
\end{equation}

Using the $\hbranch_{obs}(r)$ score we update the existing branching function as:
\begin{equation}
\label{eq:heuristic}
    \hbranch_{\Delta} (n, r) = \hratio \cdot \hbranch(n, r) + (1-\hratio) \cdot ( \hbranch_{obs}(r) - \threshold ).
\end{equation}
Here, we introduce two hyperparameters $\hratio$ and $\threshold$. The hyperparameter $\hratio \in [0, 1]$ controls the importance given to the actual heuristic score and the observed improvement from the verification on $N$.
If $\hratio = 1$, then $\hbranch_{\Delta}$ depends only on the original branching heuristic score. 
If $\hratio = 0$, then it fully relies on observed split scores. 
The hyperparameter $\threshold$ ensures that our score positively changes score for $r$ that have $\hbranch_{obs}(r) > \threshold$ and negatively change scores for $\hbranch_{obs}(r) < \threshold$.

\noindent{\textbf{Constructing a Pruned Specification Tree:}} 
%
%
The two reordering goals of prioritizing and deprioritizing effective and ineffective splits are difficult to combine with reuse. 
However, instead of starting from scratch, we can construct a specification tree $\Tprune$ from $T^N_f$ excluding the ineffective splits. 
For $\node \in \nodes{T^N_f}$, where ReLU $r$ splits $\node$, we denote the set of bad splits as the set $\mathcal{B}(T^N_f)$ of the pairs $(n, r)$ such that the improvement score $I_N(n, r) \leq \threshold$. 
For $(n, r) \in \mathcal{B}(T^N_f)$ while constructing the pruned tree our algorithm chooses a child $n_k$ of $n$. If a ReLU $r_k$ is split at $n_k$ in $T^N_f$, it performs a split $r_k$ in the corresponding node in $\Tprune$, and skips over the bad split $r$. 
The subtree corresponding to the other child $n_{k'}$ is eliminated and not added to our pruned tree. 
We choose $n_k$ such that:
\begin{equation}
    \label{eq:chosen}
    n_k = \argmin_{n_u \in \children{n}} \lb_{N}(n_u) - \lb_{N}(n)
\end{equation}
We choose such $n_k$ over $n_{k'}$ since $\lb_{N}(n)$ is closer to $\lb_{N}(n_k)$ than $\lb_{N}(n_{k'})$. Further, combining Equation ~\ref{eq:improve}~and~\ref{eq:chosen}, we can show $(\lb_{N}(n_k) - \lb_{N}(n)) < \threshold$, i.e. their difference is bounded.  
We anticipate that on the omission of the split $r$, the subtree corresponding to $n_k$ is a better match to the necessary branching decisions following $n$ than $n_{k'}$.

Algorithm~\ref{alg:prune} presents the top-down construction of $\Tprune$. The algorithm starts from the root of $T^N_f$ and recursively traverses through the children constructing $\Tprune$. 
It maintains a queue $\queue$ of nodes yet to be explored and a map $\nodeMap$ that maps nodes from the tree $T^N_f$ to the corresponding new nodes in $\Tprune$. 
At a node $n$, if $(n, r)$ is not a bad split, it performs the split $r$ at the corresponding mapping $\hat{n}$. Otherwise, if $r_k$ is the split at $n_k$, it skips over $r$ and performs the split of $r_k$ at $\hat{n}$. The newly created children from a split of $\hat{n}$ are associated with children of $n_k$ using $\nodeMap$. The children of $n_k$ are added in the $\queue$ and they are recursively processed in the next iteration for further constructing $\Tprune$.  
$\Tprune$ is still a specification tree satisfying the Definition~\ref{sec:prooftree} by construction. The specifications $\spec{n}$ of a node $n$ in $\Tprune$ can be constructed using a path from the root to $n$. 
%

\begin{minipage}{0.57\textwidth}
\begin{algorithm}[H]
\small
\setstretch{1.1}
\caption{Creating a Pruned Tree}
\label{alg:prune}
\begin{flushleft}
\textbf{Input:} specification tree $T^{\perturbedNetwork}_f$, hyperparameter $\threshold$ \\
\textbf{Output:}
 Pruned tree $\Tprune$
\end{flushleft}
\begin{algorithmic}[1] 
\State $n_{\textit{root}} \gets $ root of $T^{\perturbedNetwork}_f$, $\hat{n}_{\textit{root}} \gets$ copy of $n_{\textit{root}} $
\State $\Tprune \gets$ Initialize a new tree with $\hat{n}_{\textit{root}}$
\State $\queue \gets$ Initialize list with $n_{\textit{root}}$ \State $\nodeMap \gets$  Initialize an empty map
\label{line:init} 
\State $\nodeMap[n_{\textit{root}}] \gets \hat{n}_{\textit{root}}$

\While{$\queue$ is not empty} \label{line:loop2} 
\State $n \gets \queue.\textit{pop()}$; $r \gets $ split at node $n$; $\hat{n} \gets \nodeMap[n]$ 
\If {$I_N(n, r) < \threshold$}

\State $n_k \gets \argmin_{n_k \in \children{n}} \lb_{N}(n_k) - \lb_{N}(n)$
\State $r_k \gets $ split at node $n_k$
\State $n_l, n_r \gets n_k.\textit{children}$; $\hat{n}_l, \hat{n}_r \gets$ \add$(\Tprune, \hat{n}, r_{k})$  
\State $\nodeMap[n_l] \gets \hat{n}_l; \nodeMap[n_r] \gets \hat{n}_r$
\State $\queue.\textit{push}(n_l); \queue.\textit{push}(n_r)$ 
\Else 
\State $n_l, n_r \gets n.\textit{children}$; $\hat{n}_l, \hat{n}_r \gets$ \add$(\Tprune, \hat{n}, r)$  
\State $\nodeMap[n_l] \gets \hat{n}_l;\nodeMap[n_r] \gets \hat{n}_r$
\State $\queue.\textit{push}(n_l); \queue.\textit{push}(n_r)$  
\EndIf
\EndWhile
\State \textbf{return} $\Tprune$

\end{algorithmic}
\end{algorithm}


\end{minipage}
\hfill
\begin{minipage}{0.40\textwidth}

\begin{algorithm}[H]
\small
\setstretch{1.1}
\caption{Incremental Verification Algorithm}
\label{alg:algorithm_main}
\begin{flushleft}
\textbf{Input:} Original network $N$, \\
Perturbed network $\perturbedNetwork$, \\
property ($\phi$, $\psi$), \\
analyzer $\vbound$, \\
branching heuristic $\hbranch$,\\
hyperparameters \\
$\alpha$ and $\threshold$,\\
incremental verifier $V_\Delta$\\
\textbf{Output:}
 Verification result for $N$ and $\perturbedNetwork$
\end{flushleft}

\begin{algorithmic}[1] 
\State $\mathit{resultN}$, $\T{N}{f} \gets$ $V(N, \phi, \psi, \hbranch)$
\State $\Tinit \gets$ PrunedTree$(T^{N}_f, \threshold)$ \label{line:prune}
\State $\hbranch_{\Delta} \gets$  UpdateH$(\hbranch, \T{N}{f}, \threshold, \alpha)$ \label{line:updateh}
\State $\mathit{resultN}^a \gets$ $V_\Delta(\perturbedNetwork,  \phi, \psi, \Tinit, \hbranch_{\Delta})$  \label{line:verinc} \Comment{Incremental verification step calls Algorithm~\ref{alg:incver}}
\State \textbf{return} $\mathit{resultN}, \mathit{resultN}^a$
\end{algorithmic}
\end{algorithm}

\end{minipage}

\noindent{\bf Main algorithm:} Algorithm~\ref{alg:algorithm_main} presents \Tool's main algorithm for incremental verification that combines all the aforementioned techniques. It takes as inputs the original network $N$, a perturbed network $\perturbedNetwork$, input specification $\phi$, and an output property $\psi$. It prunes the final tree $T^N_f$ obtained in the verification of $N$ and constructs $\Tinit$ (line~\ref{line:prune}). It computes the updated branching heuristic $\hbranch_{\Delta}$ using Equation~\ref{eq:heuristic} (line~\ref{line:updateh}). It uses $\Tinit$ and $\hbranch_{\Delta}$ for performing fast incremental verification of networks $\perturbedNetwork$ (line~\ref{line:verinc}).

We next state the following lemma that states - verifying the property $(\phi, \psi)$ is equivalent to verifying the specifications for all the leaves.

\begin{restatable}{lemma}{invariance}
\label{lemma:invariance} 
The specifications encoded by the leaf nodes of a specification tree $T$ maintain the following invariance. 
\[
    \Bigg( \bigwedge_{n \in leaves(T)} \spec{n} \Bigg) \Longleftrightarrow  (\phi \to \psi)
\]
\end{restatable}

We next use the lemma to prove the soundness and completeness of our algorithm. All the proofs are in Appendix~\ref{sec:proofs}.

\begin{restatable}{theorem}{sound}(Soundness of Verification Algorithm).
\label{theorem:sound} If Algorithm~\ref{alg:algorithm_main} verifies the property $(\phi, \psi)$ for the network $\perturbedNetwork$, then the property must hold. 
\end{restatable}

\vspace{-0.07in}

\begin{restatable}{theorem}{complete}(Completeness of Verification Algorithm).
If for the network $\perturbedNetwork$, the property  $(\phi, \psi)$ holds then Algorithm~\ref{alg:algorithm_main} always terminates and produces $\ver$ as output. 
\label{theorem:complete}
\end{restatable}

\vspace{-0.07in}

\new{
{\noindent \bf Scope of \Tool{}:} \Tool{} utilizes the specification tree to store the trace of the BaB proof. The IVAN algorithm enhances this tree by reusing and refining it to enable faster BaB proof of updated networks. Our paper focuses on using \Tool{} to verify ReLU networks with BaB that implements ReLU splitting. However, we expect that 
\Tool{}'s principles can be extended to networks with other activation functions (tanh, sigmoid, leaky ReLU) for which BaB has been applied for verification.
}

\vspace{-0.07in}
\subsection{Network Perturbation Bounds}
\label{sec:perturb}
In this section, we formally characterize a class of perturbations on a network $N$ where our proposed "Reuse" technique \new{attains maximum possible speed-up.} 
Specifically, we focus on modifications affecting only the last layer which represent many practical network perturbations (e.g, transfer learning, fine-tuning). \new{The last layer modification assumption is only for our theoretical results in this section. Our experiments make no such assumption and consider perturbations applied across the original network.}

We leave the derivation of perturbation bounds corresponding to the full \Tool{} to future work as it requires theoretically modeling the effect of arbitrary network perturbations on DNN output as well as complex interactions between "Reuse" and "Reorder" techniques. 
Given a specification tree $T$ and network architecture $\mathcal{N}$, we identify a set of neural networks $\distri{T}$ such that any network $\perturbedNetwork \in \distri{T}$ can be verified by reusing $T$.

%
%
%
%
We assume the weights are changed by the weight perturbation matrix $\Eps$. 
If $\OrgNetwork_{\Layers} = \Relu(A_{\Layers} \cdot X + B_{\Layers})$ then last layer of $\perturbedNetwork$ is $\perturbedNetwork_{\Layers} = \Relu((A_{\Layers} + \Eps) \cdot X + B_{\Layers})$.

\vspace{-0.07in}
\begin{definition}[Last Layer Perturbed Network]
Given a network $\OrgNetwork$ with architecture $\mathcal{N}$, the set of last layer perturbed networks is $\mathcal{M}(N, \EpsNorm) \subseteq \mathcal{N}$, such that if $\perturbedNetwork \in \mathcal{M}(N, \EpsNorm)$ then $(\forall i \in [\Layers -1]) \cdot \OrgNetwork_i = \perturbedNetwork_i$, $\OrgNetwork_{\Layers} = \Relu(A_{\Layers} \cdot X + B_{\Layers})$,    $\perturbedNetwork_{\Layers} = \Relu((A_{\Layers} + \Eps) \cdot X + B_{\Layers})$ and $\|\Eps\|_{F} \leq \EpsNorm$.
\footnote{$\|\cdot\|_F$ denotes the Frobenius norm of a matrix} 
\end{definition}
\vspace{-0.07in}

\noindent We next compute the upper bound of $\delta$, for which if the property can be proved/disproved using specification tree $T$ in $\OrgNetwork$ then the same property can be proved/disproved in $\perturbedNetwork$ using the same $T$.
Therefore, once we have the proof tree $T$ that verifies the property in $N$ we can reuse $T$ for verifying any perturbed network $\perturbedNetwork \in \mathcal{M}(N, \EpsNorm)$.
Assuming the property $(\phi, \psi)$ and the analyzer $\vbound$ are the same for any perturbed network $\perturbedNetwork \in \mathcal{M}(N, \EpsNorm)$ the upper bound of $\EpsNorm$ only depends on $\OrgNetwork$ and $T$. 

\noindent We next introduce some useful notations that help us explicitly compute the upper bound of $\EpsNorm$.
Given $T$ let $\FeasibleReg(\GenericNet_{i}, T)$ be the over-approximated region computed by the analyzer $A$ that contains all feasible outputs $N_{i}$ of the $i$-th layer of the original network.
Note $\FeasibleReg(\GenericNet_{i}, T)$ depends on the $\phi$ and analyzer $A$ but we omit them to simplify the notation.
Let $\SolvFunc(N, T)$ denote whether the property $(\phi, \psi)$ can be verifed on network $\OrgNetwork$ with $T$. 
Proof of Theorem~\ref{thm:perturb1} is presented in Section~\ref{sec:proofs2}
\begin{align}
&\ProblemMin(\FeasibleReg(\GenericNet_{\Layers}, T)) = \min_{Y\;\in\;\FeasibleReg(\GenericNet_{\Layers}, T)} \Lpc^TY \\
&\SolvFunc(N, T) = (\ProblemMin(\FeasibleReg(\GenericNet_{\Layers}, T)) \geq 0) \\
&\MaxNorm(N, T) = \max_{Y \in \FeasibleReg(\OrgNetwork_{l-1}, T)} \|Y\|_2
\end{align}
\vspace{-0.07in}

\begin{restatable}{theorem}{perturba}
\label{thm:perturb1}
If $\EpsNorm \leq \frac{|\ProblemMin(\FeasibleReg(\OrgNetwork_{\Layers}, T))|}{\|\Lpc\|_2 \cdot \MaxNorm(N, T)}$ then for any perturbed network $\perturbedNetwork \in \mathcal{M}(N, \EpsNorm)$ $\SolvFunc(\OrgNetwork, T) \iff \SolvFunc(\perturbedNetwork, T)$.
\end{restatable}
The proof of the theorem is in Appendix~\ref{sec:proofs2}.

\vspace{-0.07in}
\section{Methodology}

\begin{table*}[!t]
\centering\tablesize
\caption{Models and the perturbation $\epsilon$ used for the evaluation for incremental verification.}
\vspace{-.15in}

\resizebox{.98\columnwidth}{!}{

\begin{tabular}{l l l l l l}
    \toprule
    Model & Architecture & Dataset & \#Neurons & Training Method & $\epsilon$\\ 
    \midrule
    ACAS-XU Networks & $6 \times 50$ linear layers  & ACAS-XU & 300 & Standard \cite{Julian_2019} &  - \\
    \neta & $2 \times 256$ linear layers  & MNIST & 512 & Standard  & 0.02\\
    \netb & 2 Conv, 2 linear layers & MNIST & 9508 & Certified Robust \cite{balunovic2020Adversarial} & 0.1\\
    \netc & 2 Conv, 2 linear layers & CIFAR10 & 4852 & Empirical Robust \cite{Dong_2018_CVPR} & $\frac{2}{255}$\\
    \netd & 2 Conv, 2 linear layers & CIFAR10 & 6244 & Certified Robust \cite{wong2018provable} & $\frac{4}{255}$\\
    \nete & 4 Conv, 2 linear layers & CIFAR10 & 6756 &
    Certified Robust \cite{wong2018provable} & $\frac{4}{255}$\\
    
    \bottomrule
\end{tabular}
}
\vspace{-0.1in}
\label{table:models}
\end{table*}

\noindent{\bf Networks and Properties.} 
We evaluate \Tool{} on models with various architectures that are trained with different training methods. 
Similar to most of the previous literature, we verify $L_\infty$-based local robustness properties for MNIST and CIFAR10 networks and choose standard $\epsilon$ values used for evaluating complete verifiers. 
For the verification of global properties in Section~\ref{sec:acas} we use the standard set of ACAS-XU properties that are part of the VNN-COMP benchmarks \cite{DBLP:journals/corr/abs-2109-00498}. 
Table~\ref{table:models} presents the evaluated models and the choice of $\epsilon$ for the local robustness properties.


\noindent{\bf Network Perturbation.} Similar to previous works \cite{DBLP:conf/icse/PaulsenWW20, DBLP:journals/pacmpl/UgareSM22}, we use quantization to generate perturbed networks. Specifically, we use int8 and int16 post-training quantizations. The quantization scheme \mbox{has the form \cite{tf_quantization}:}
%
  $  r =  s (q - zp) $.  
%
Here, $q$ is the quantized value and $r$ is the real value; $s$ which is the scale and $zp$ which is the zero point are the parameters of quantization. Our experiments use symmetric quantization with $zp = 0$.


\noindent\textbf{Baseline.}
\new{
We use the following baseline BaB verifiers:} 
\begin{itemize}
    \item For proving the local robustness properties, we use LP-based triangle relaxation for bounding \cite{ehlers2017formal, bunel2020branch}, and we use the estimation based on coefficients of the zonotopes for choosing the ReLU splitting \cite{ijcai2021p351}.
    \item  For the verification of ACAS-XU global properties, we use RefineZono \cite{singh2019boosting}. \new{RefineZono uses DeepZ \cite{singh2018fast} analyzer with input splitting. This baseline is used only for experiments in Section~\ref{sec:acas}.} 
\end{itemize}
 
\noindent{\bf Experimental Setup.} We use 64 cores of an AMD Ryzen Threadripper CPU with the main memory of 128 GB running the Linux operating system. The code for our tool is written in Python. We use the GUROBI \cite{gurobi2018} solver for our LP-based analyzer. 

\noindent{\bf Hyperparameters.} We use Optuna tuner \cite{optuna_2019} for tuning the hyperparameters. We present more details and sensitivity analysis of the hyperparameters in Section~\ref{sec:hyperparam}.

\vspace{-0.07in}
\section{Experimental Evaluation}
\label{sec:eval}

We evaluate the effectiveness of \Tool{} in verifying the local robustness properties of the quantized networks. We then analyze how various tool components contribute to the overall result. We further show the sensitivity of speedup obtained by \Tool{} to the hyperparameters. 
\new{We also stress-test \Tool{} on large random perturbation to the network.}
Finally, we evaluate the effectiveness of \Tool{} on global property verification with input splitting. 

\vspace{-0.07in}

\subsection{Effectiveness of \Tool{}}
\label{sec:local}
\begin{figure}[!htbp]
\centering\vspace{-.2in}
\begin{subfigure}[b]{0.4\textwidth}
 \includegraphics[width=\textwidth]{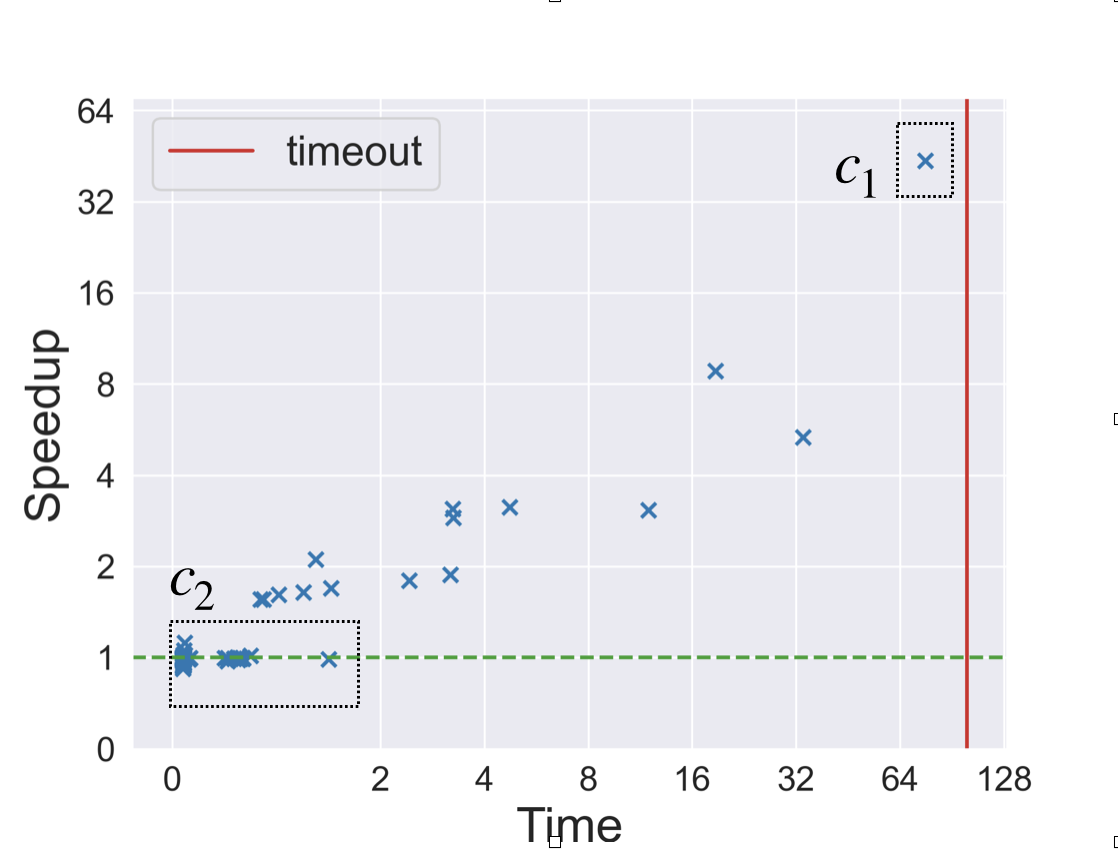}
 \caption{\neta with INT16 quantization}
 \label{fig:scatter1a}
\end{subfigure}
\hspace{10mm}
\begin{subfigure}[b]{0.4\textwidth}
 \includegraphics[width=\textwidth]{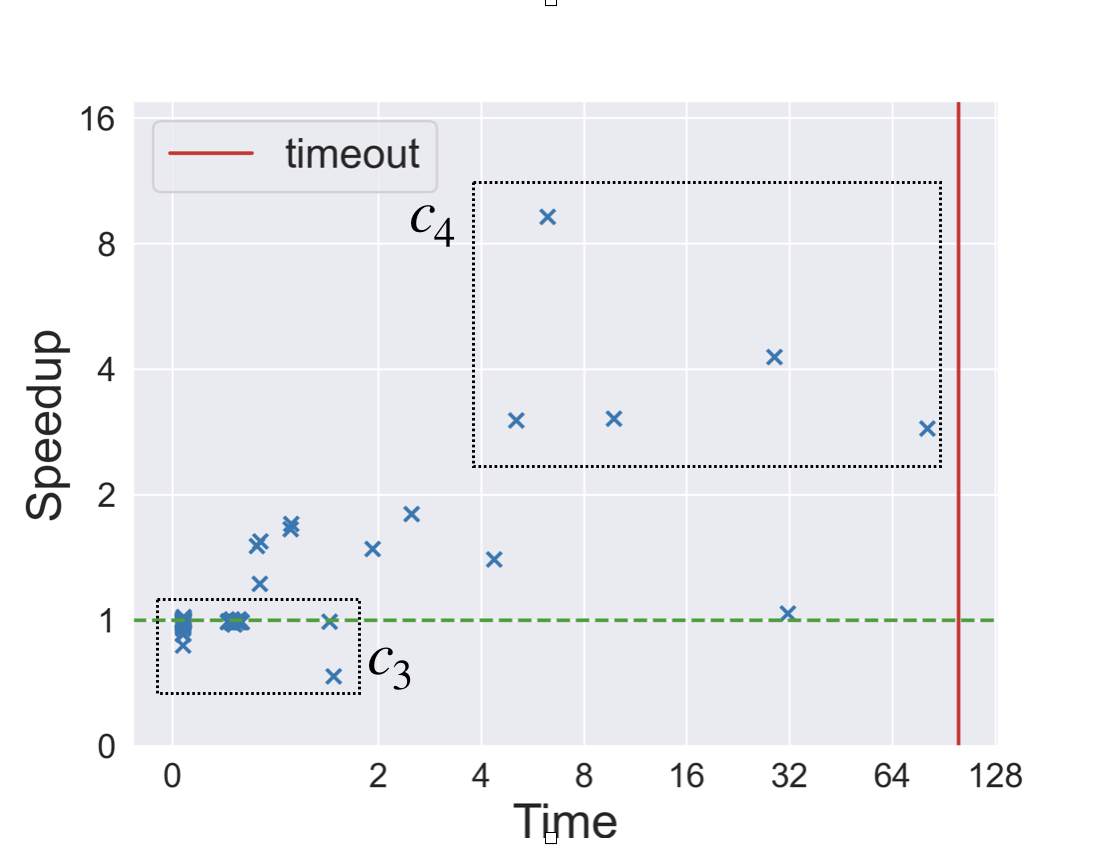}
 \caption{\neta with INT8 quantization}
 \label{fig:scatter1b}
\end{subfigure}
\vspace{-.1in}
\hspace{10mm}
\hfill
\caption{\Tool{} speedup for the verification of local robustness properties on \neta.}
\label{fig:scatter1}\vspace{-.1in}
\end{figure} 
\vspace{-0.1in}


Figure~\ref{fig:scatter1} presents the speedup obtained by \Tool{} on \neta. The x-axis displays the time taken by the baseline verifier for the verification in Seconds. The y-axis denotes the speedup obtained by \Tool{} over the baseline on a specific verification instance. Each cross in the plot shows results for a specific verification property. The vertical line denotes the timeout for the experiment and the dashed line is to separate instances that have a speedup greater than 1x. 

We observe that \Tool{} gets higher speedup on hard instances that take more time for verification on the baseline. \Tool{} has a small overhead for storing the specification tree compared to the baseline. For hard specifications that result in large specification trees, this overhead is insignificant compared to the improvement in the verification time. Our techniques that reuse and refine the tree focus on speeding up such hard specifications. However, for specifications that are easy to prove with small specification trees, we see a slight slowdown in verification time. Since these easy specifications are verified quickly by both \Tool{} and the baseline, they are irrelevant in overall verification time over all the specifications. For instance, the box labeled by $c_2$ in Figure~\ref{fig:scatter1a} contains all of the 83 cases with low $(<1.2x)$ speedup on int16 quantized network. Despite low speedup, all of them take 16.27s to verify with \Tool{}. Whereas the case labeled by $c_1$ alone takes 75.54s on the baseline and 1.73s on \Tool{}, leading to a 43x speedup -- caused by reducing BaB tree size from 345 nodes to 28 nodes on pruning, out of which only 14 leaf nodes are active and lead to analyzer calls. 

We observe a similar pattern in the case of the int8 quantized network in Figure~\ref{fig:scatter1b}. It shows that the cases confined in box $c_3$, despite having lower speedup, take relatively less time. The cases included in box $c_4$ in Figure~\ref{fig:scatter1b} have a much higher impact on the overall verification time. Box $c_3$ includes the majority of the low-speedup 83 cases that take a total of 18.44s time for verification with \Tool{}. Whereas for 5 cases in box $c_4$ with higher speedups, take 401 analyzer calls with baseline and 118 analyzer calls with \Tool{}. Accordingly, solving them takes 130.6s with the baseline and 40.26s with \Tool{}, leading to a 3.3x speedup. 

\begin{figure}[!htbp]
\vspace{-.1in}
\hspace{5mm}
\centering

\begin{subfigure}[b]{0.41\textwidth}
 \centering
 \includegraphics[width=\textwidth]{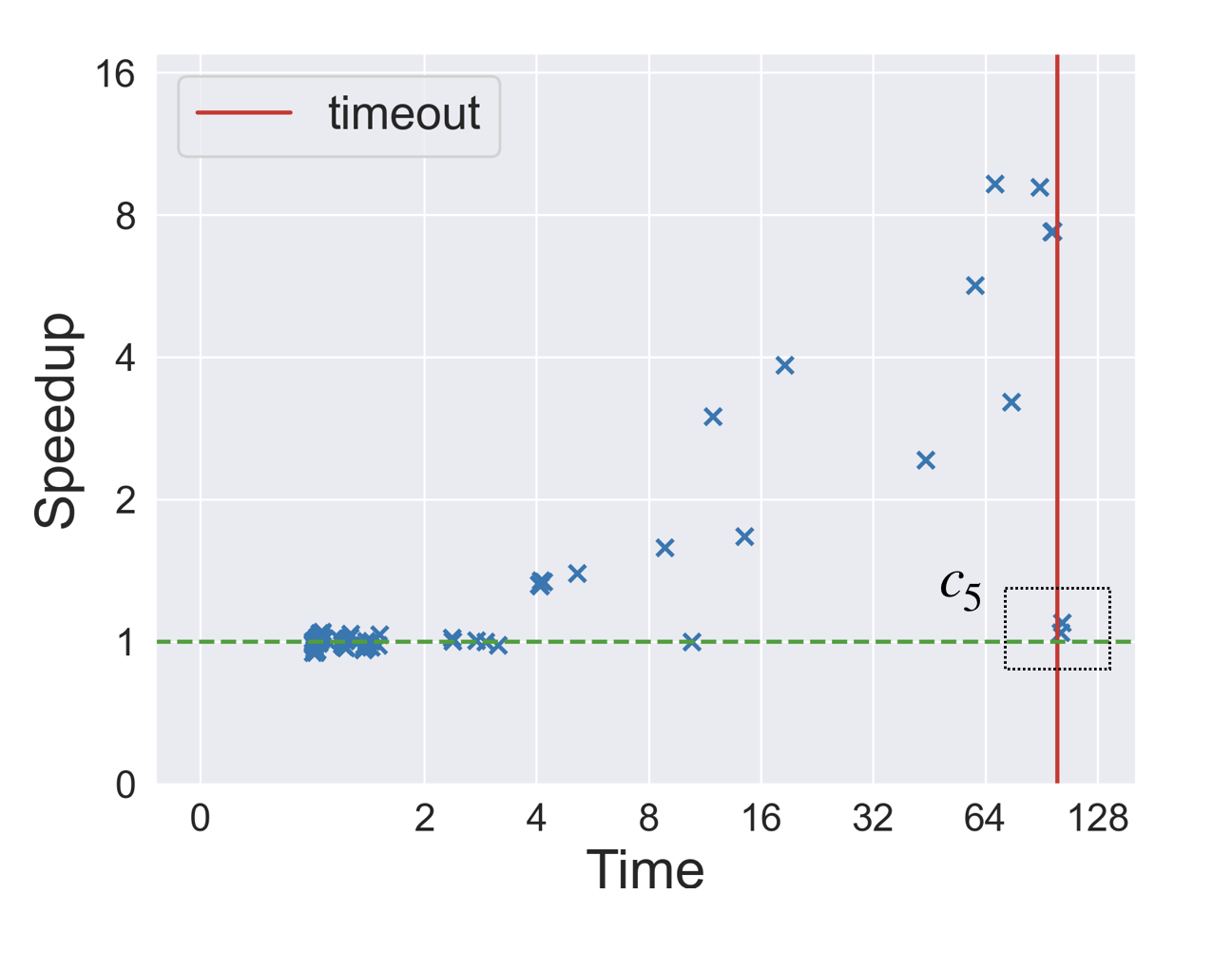}
 \caption{\netb with INT16 quantization}
 \label{fig:scatter2a}
\end{subfigure}
\hspace{10mm}
\begin{subfigure}[b]{0.41\textwidth}
 \centering
 \includegraphics[width=\textwidth]{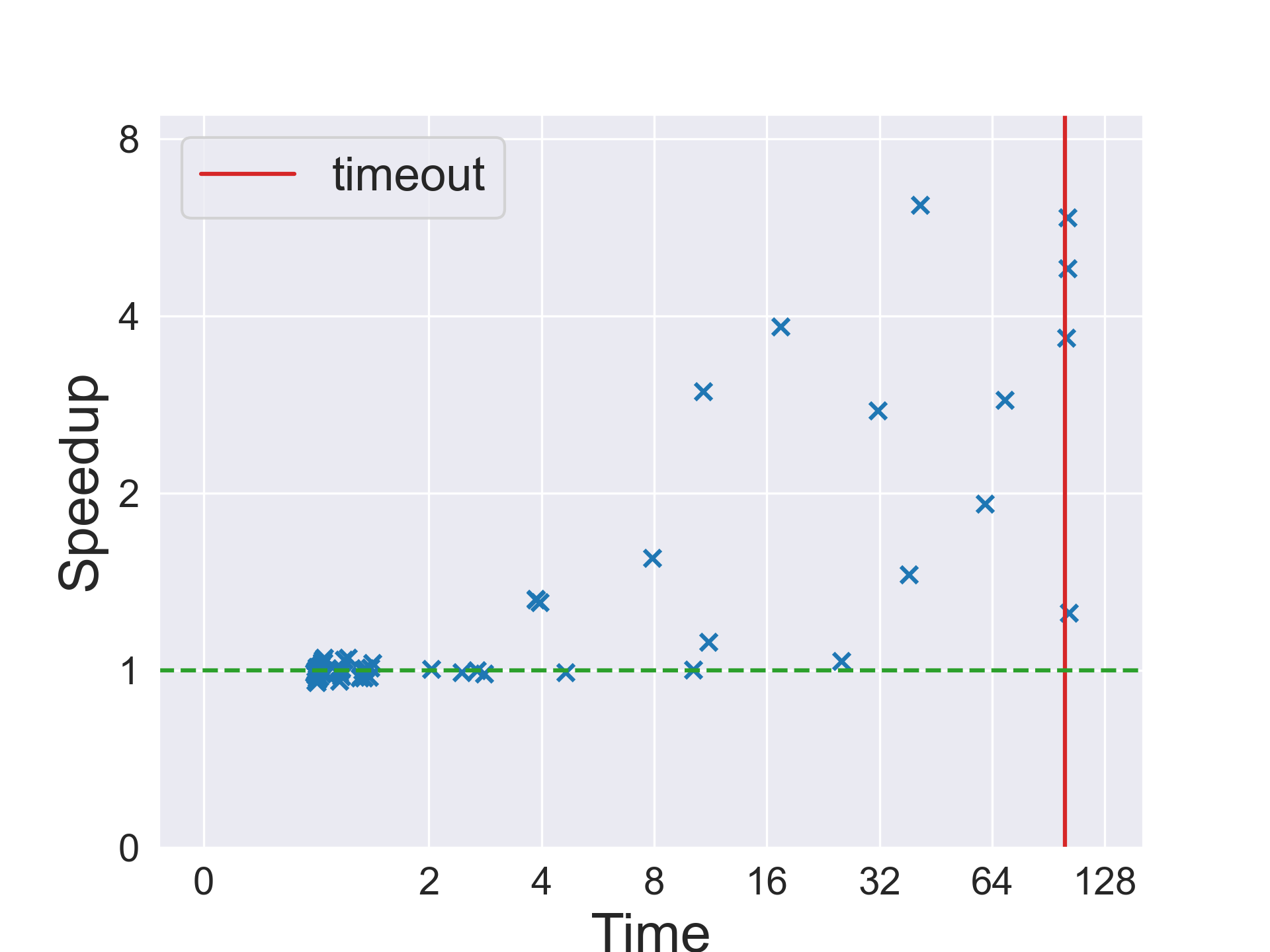}
 \caption{\netb with INT8 quantization}
\end{subfigure}

\begin{subfigure}[b]{0.41\textwidth}
 \centering
 \includegraphics[width=\textwidth]{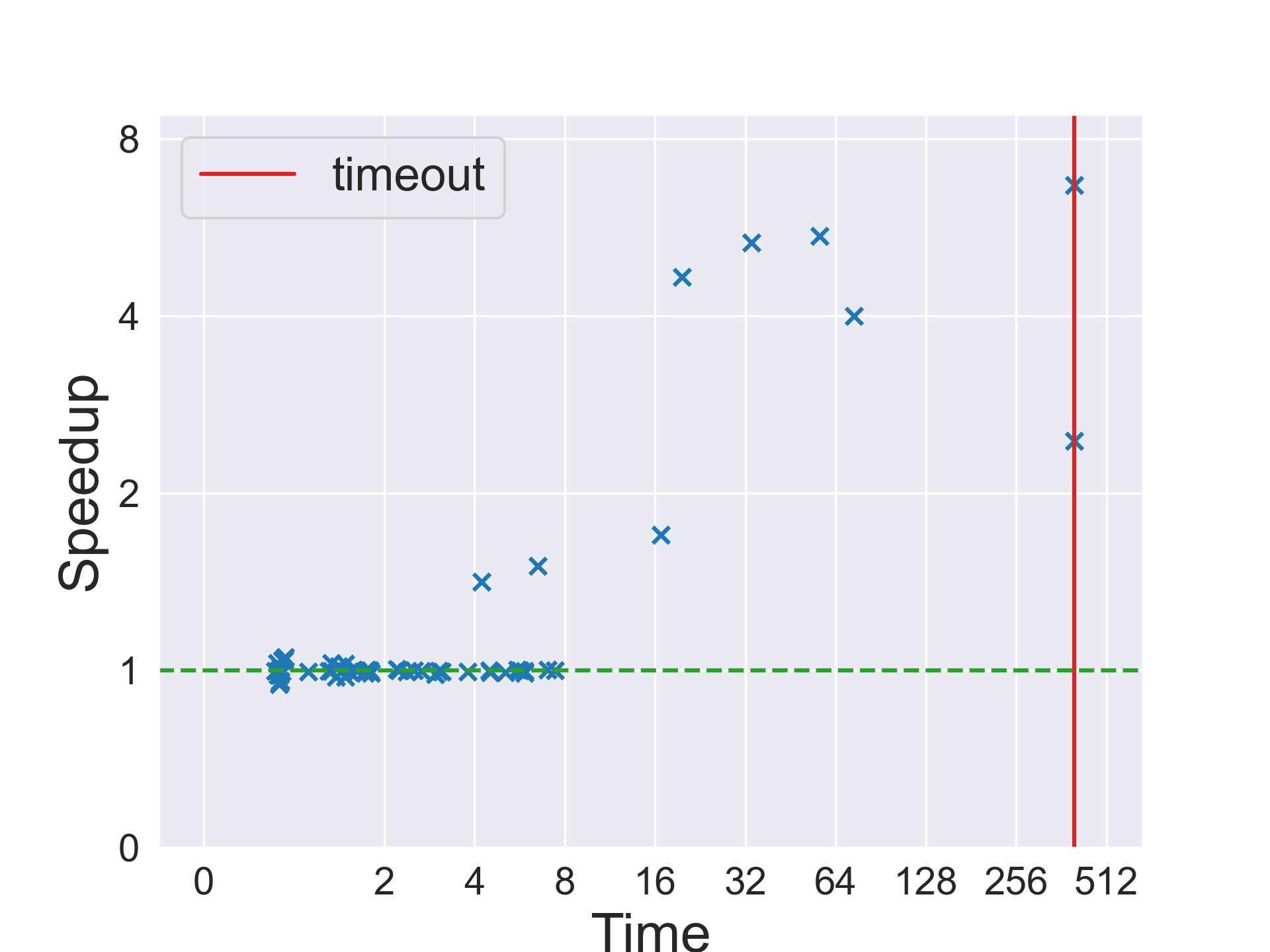}
 \caption{\netd with INT16 quant.}
\end{subfigure}
\hspace{10mm}
\begin{subfigure}[b]{0.41\textwidth}
 \centering
 \includegraphics[width=\textwidth]{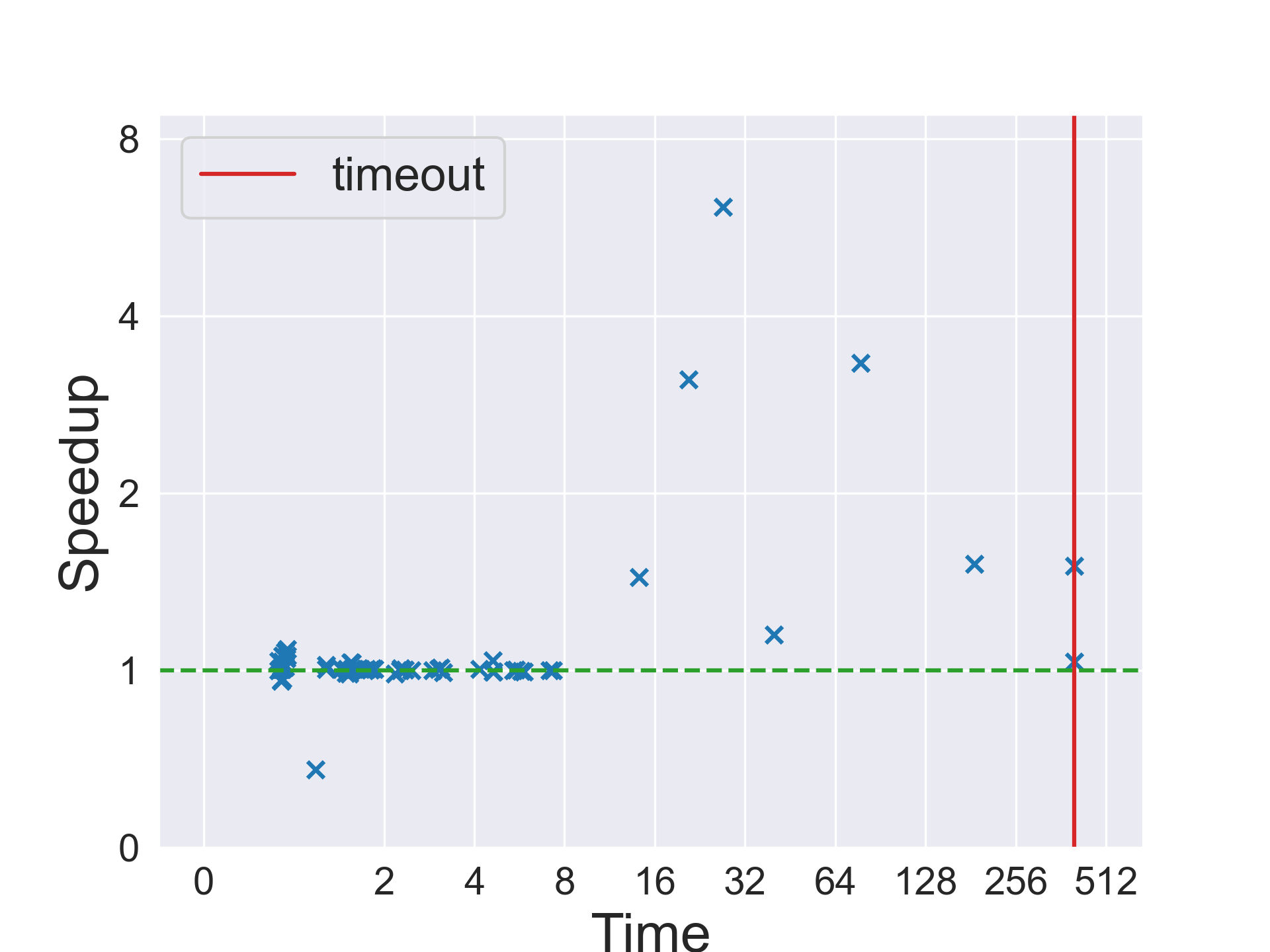}
 \caption{\netd with INT8 quant.}
\end{subfigure}
\caption{\Tool{} speedup for the verification of local robustness properties.}
\label{fig:scatter2}\vspace{-.1in}
\end{figure} 

Figure~\ref{fig:scatter2} presents speedup for several other networks. \Tool{} is notably more effective on hard-to-verify specifications that take more than 10s to verify using the baseline. It achieves \geomeanhard{} geomean speedup on such cases. In many cases, we see more solved cases by \Tool{} over the baseline. For instance, the box $c_5$ in Figure~\ref{fig:scatter2a} contains 2 cases that baseline does not solve within the timeout of 100s, but \Tool{} solves them in 90.6s and 95.8s each. We show speedup vs. time plots for other networks (\netc, \nete{}) and \mbox{more statistics of our evaluation in Appendix~\ref{sec:stats}.}

\vspace{-0.07in}
\subsection{Overall Speedup}
\label{sec:ablation}
We observe no cases when the baseline verifies the property and \Tool{} exceeds the timeout. 
We cannot compute the speedup for the cases where the baseline exceeds the timeout. 
Therefore, we compute the overall speedup over the set $S$ that denotes all the cases that are solved by the baseline within the time limit. 
$\timebase(c)$ and $\timetool(c)$ denote the time taken by baseline and \Tool{} on the case $c$ respectively, then we compute the overall speedup as $\speedup = \frac{\sum_{c \in S}{\timebase(c)}}{\sum_{c \in S}{\timetool(c)}}$.

Table~\ref{table:ablation} presents the comparison of the contribution of each technique used in \Tool{} for each model. Column $\solved$ in each case displays the number of extra verification problems solved by the technique in comparison to the baseline. 
Columns in \Tool{}[Reuse] present results on only using the reuse technique. 
Columns in \Tool{}[Reorder] show results on using the reorder technique. 
Columns in \Tool{} present the results on using all techniques from Section~\ref{sec:main}. 
Column $\speedup$ for each technique demonstrates the overall speedup obtained compared to the baseline. We observe that in most case combination of all techniques performs better than reuse and reordering. We see that reorder performs better than reuse except for one case (\neta on int8). 
\begin{table*}[!t]
\centering\tablesize
\caption{Ablation study for overall speedup across all properties for different techniques in \Tool{}.}
\begin{tabular}{l l | l l | l l | l l}
    \toprule
    Model & Approximation & \multicolumn{2}{|c|}{\Tool{}[Reuse]} & \multicolumn{2}{|c|}{\Tool{}[Reorder]} & \multicolumn{2}{|c}{\Tool{}} \\
    & & $\speedup$ & $\solved$ & $\speedup$ & $\solved$ & $\speedup$ & $\solved$ \\
    \midrule
    \neta & int16 & 2.51x & 0 & 2.71x & 0 & \textbf{4.43x} & 0 \\
    & int8 & 1.07x & 0 & 1.64x & 0 & \textbf{2.02x} & 0 \\
    \netb & int16 & 1.62x & 0 & 2.15x & 0 & \textbf{3.09x} & 2 \\
    & int8 & 1.27x & 2 & 1.34x & 3 & \textbf{1.71x} & 4 \\
    \netc & int16 & 1.02x & 0 & 1.57x & 2 & \textbf{2.52x} & 2 \\
    &int8 & 1.08x & 0 & 1.53x & 0 & \textbf{1.78x} & 0 \\
    \netd & int16 & 1.43x & 1 & 1.51x & 0 & \textbf{1.87x} & 2\\
    & int8 & 0.75x & 0 & \textbf{1.62x} & 1 & 1.53x & \textbf{2} \\
    \nete & int16 & 1.64x & 0 & 2.29x & 0 & \textbf{3.21x} & 0\\
    & int8 & 1.15x & 0 & 1.13x & 1 & \textbf{1.25x} & 1 \\
    \bottomrule
\end{tabular}
\label{table:ablation}
\vspace{-0.1in}
\end{table*}

\vspace{-0.07in}
\subsection{Hyperparameter Sensitivity Analysis}
\label{sec:hyperparam}

\begin{figure}[!htbp]
\vspace{-0.2in}
\centering
\begin{subfigure}[b]{0.41\textwidth}
 \includegraphics[width=\textwidth]{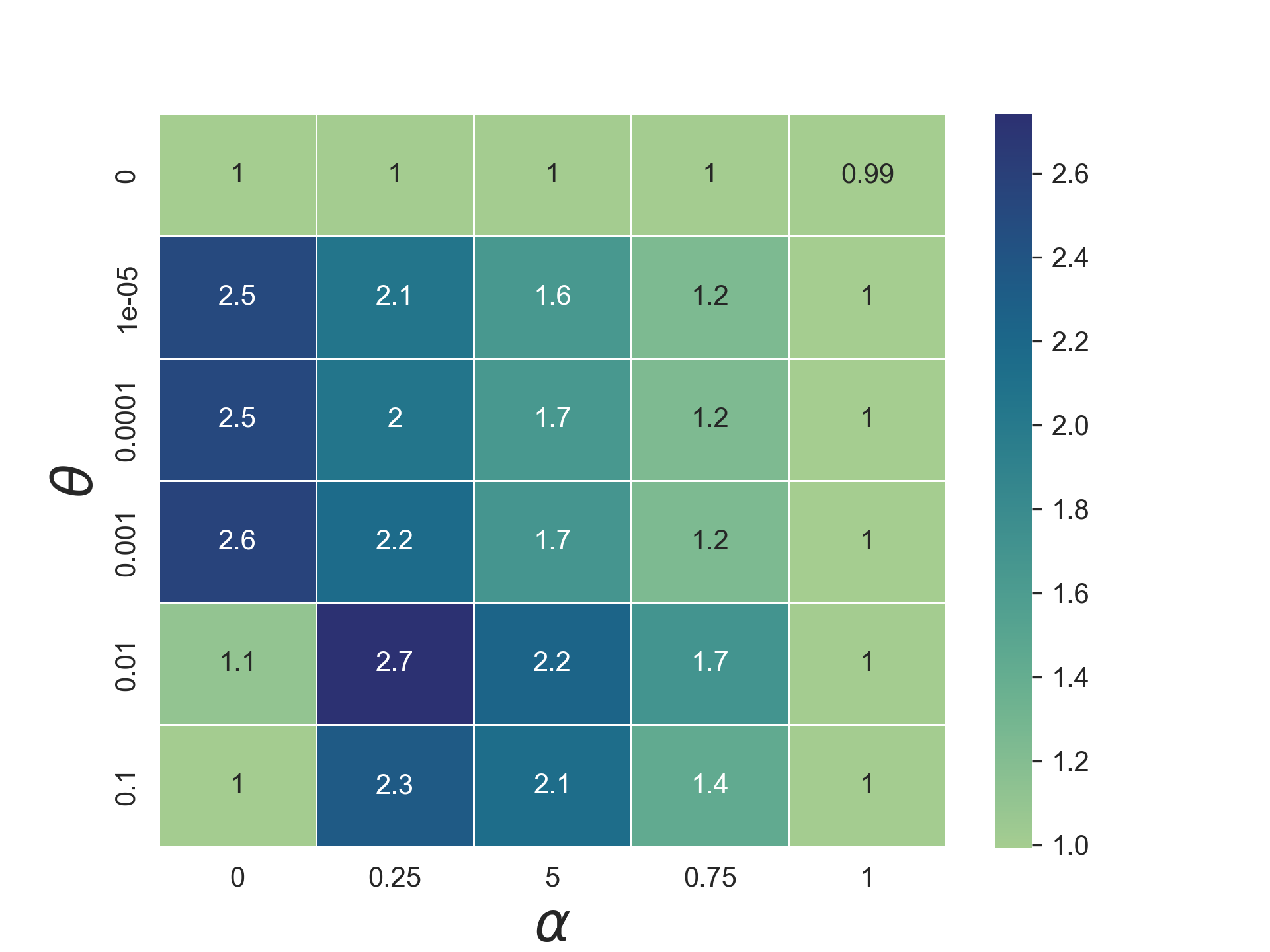}
 \caption{\Tool{}[Reorder]}
 \label{fig:sensitivity1}
\end{subfigure}
\hspace{5mm}
\begin{subfigure}[b]{0.41\textwidth}
 \includegraphics[width=\textwidth]{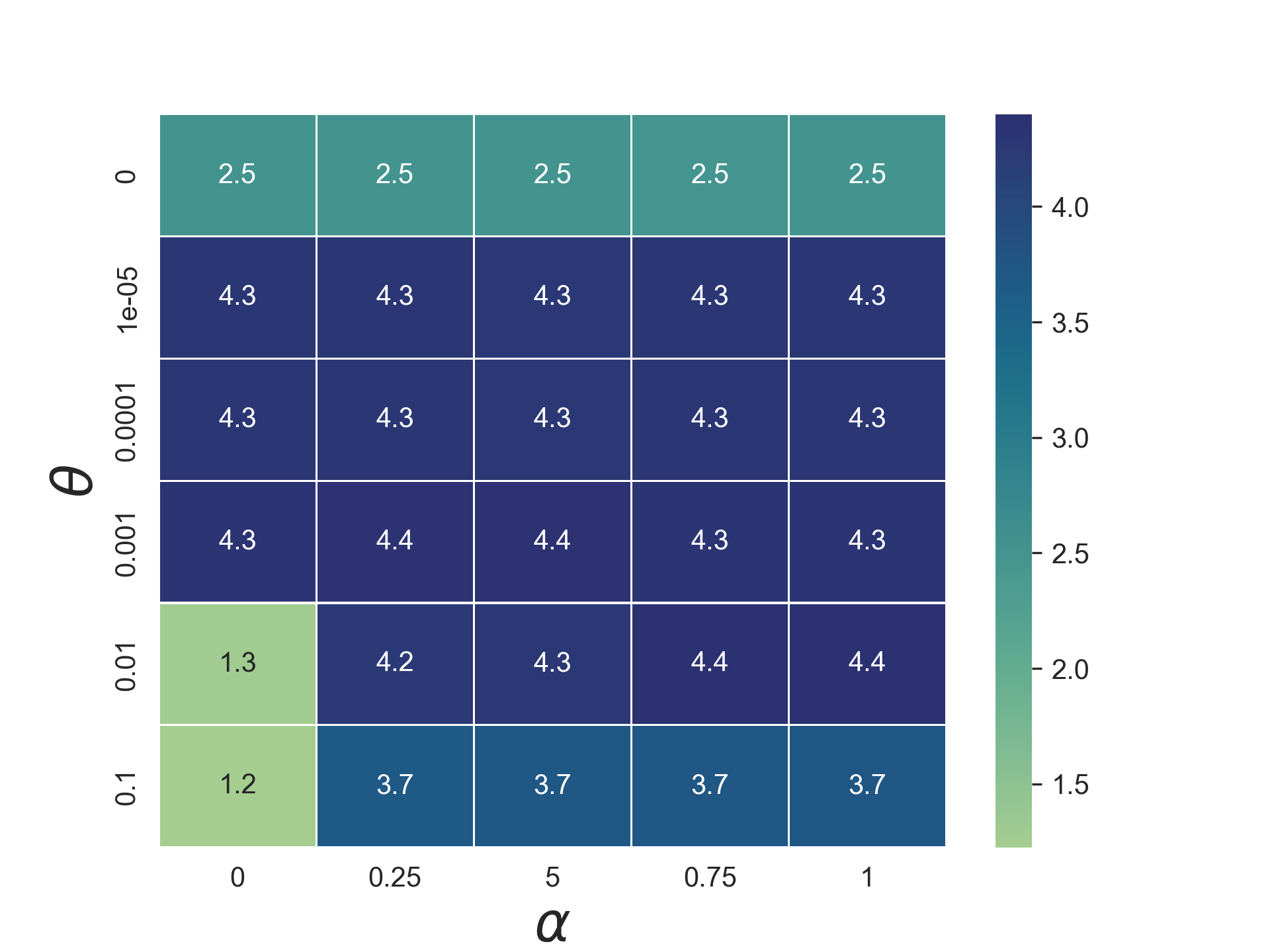}
 \caption{\Tool{}}
 \label{fig:sensitivity2}
\end{subfigure}
\hfill
\caption{Speedup for the combination of hyperparameter values on \neta with int16 quantization.}
\label{fig:sensitivity}
\end{figure} 
\vspace{-0.1in}


Figure~\ref{fig:sensitivity} plots the heatmap for \Tool{} speedup on various hyperparameter values. The x-axis shows the hyperparameter $\alpha$ value and the y-axis shows the $\threshold$ value. Each point in the greed is annotated with the observed $\speedup$ on setting the corresponding hyperparameter values. Figure~\ref{fig:sensitivity1} presents the plot for \Tool{} with on reorder technique. $(\alpha, \threshold) = (0.25, 0.01)$ is the highest speedup point. Choosing $\threshold = 0$ implies that are not deprioritizing the splitting decisions that did not work well. In that case, we observe no speedup with reordering, showing the necessity of $\threshold$ in our $\hbranch_\Delta$ formulation. Figure~\ref{fig:sensitivity2} presents the same plot for our main algorithm that also reuses the pruned tree. We observe that the speedup is less sensitive to hyperparameter value changes in this plot. This is expected since reordering starts from single node $\Tinit$ and purely relied on $\hbranch_\Delta$ formulation for the speedup. While our main technique also reuses the tree, even when $\threshold = 0$ it can get $\sim$2.5x speedup. 


\subsection{Global Properties with Input Splitting}
\label{sec:acas}

\begin{figure}[!htbp]
\centering
\vspace{-0.2in}
\begin{subfigure}[b]{0.49\textwidth}
 \includegraphics[width=0.8\textwidth]{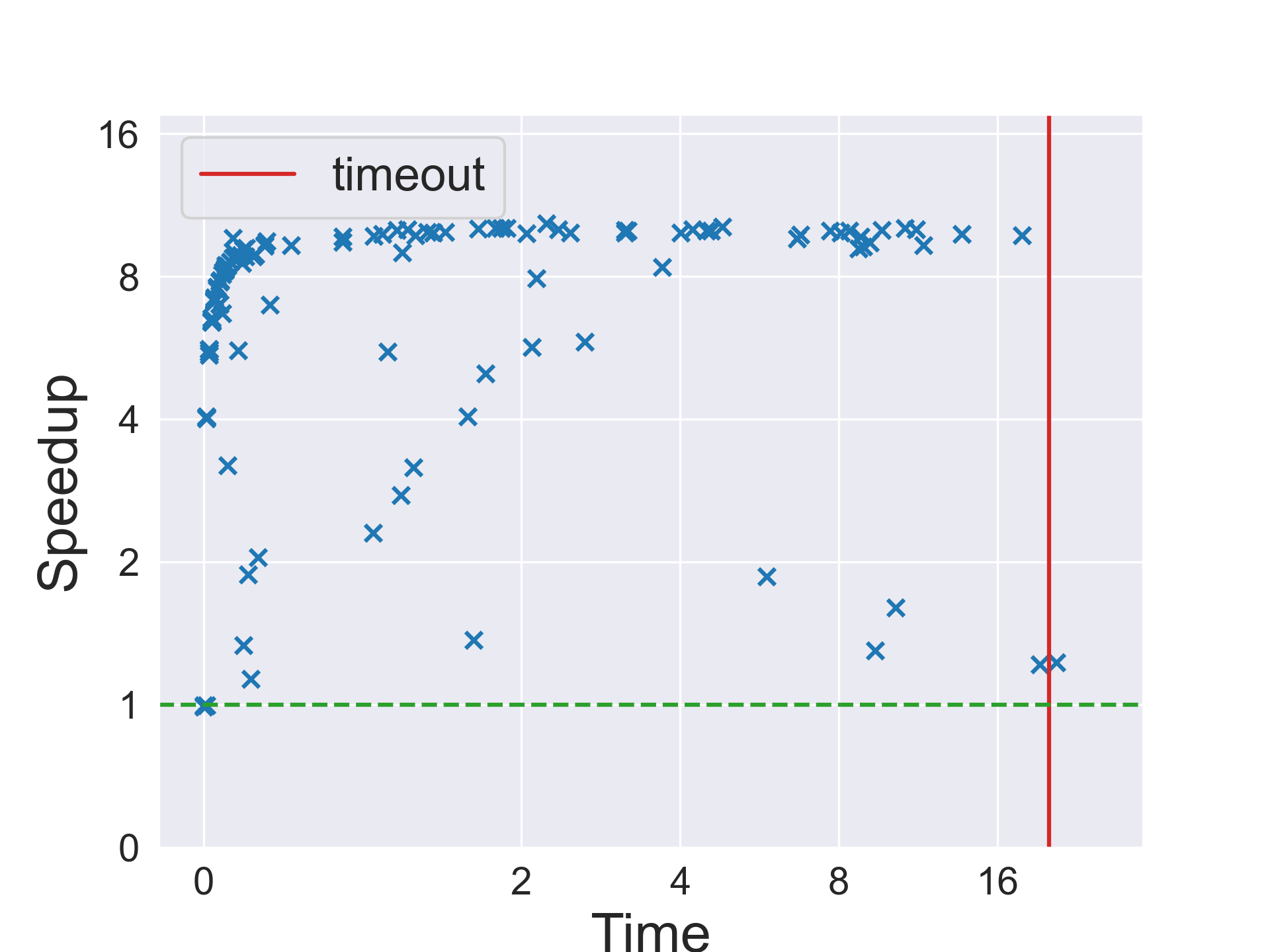}
 \caption{ACAS-XU networks with INT16 quantization}
 \label{fig:scatter11a}
\end{subfigure}
\hspace{1mm}
\begin{subfigure}[b]{0.49\textwidth}
 \includegraphics[width=0.8\textwidth]{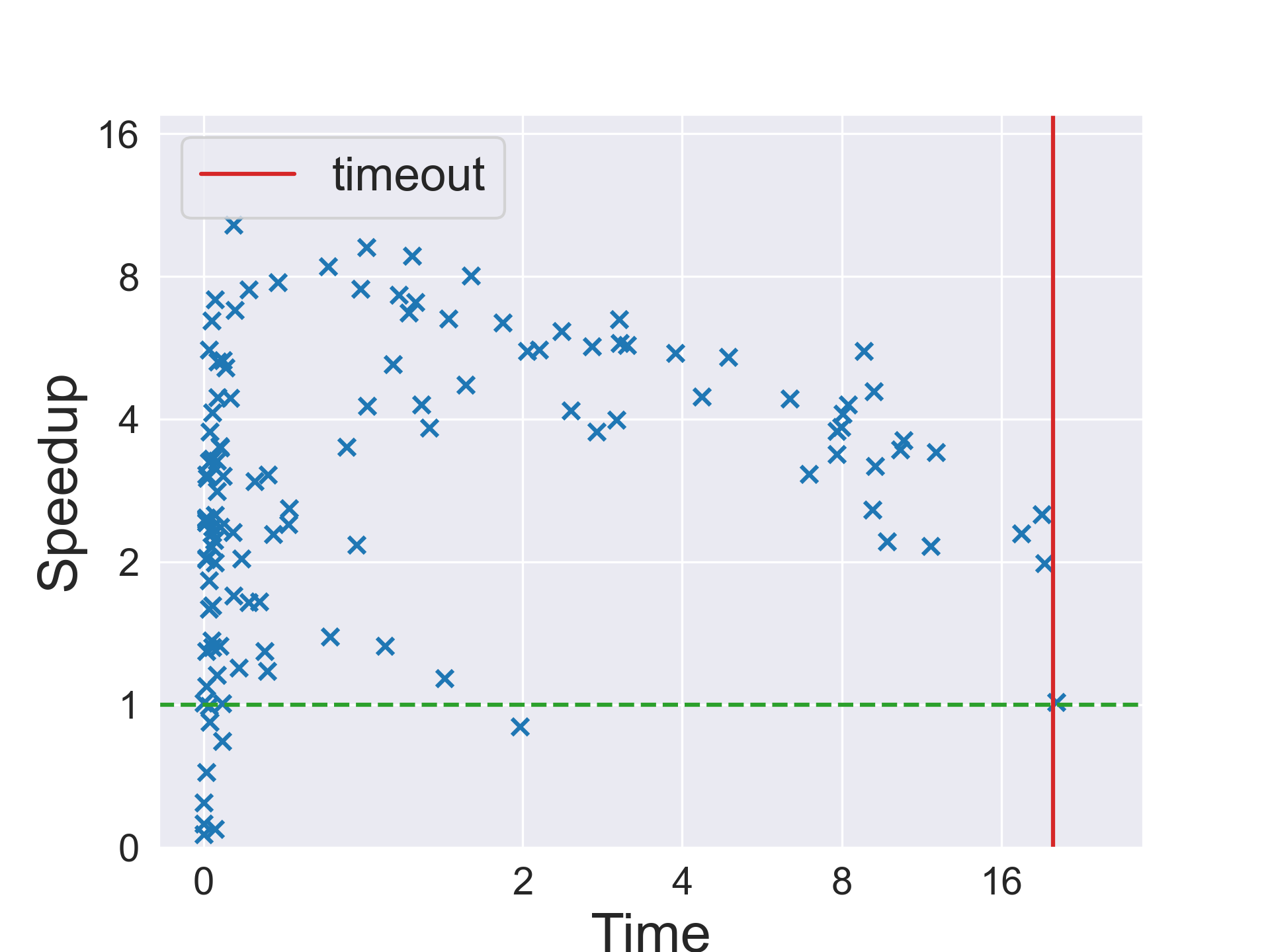}
 \caption{ACAS-XU networks with INT8 quantization}
 \label{fig:scatter11b}
\end{subfigure}

\caption{\Tool{} speedup for the verification of global ACAS-XU properties.}
\vspace{-0.2in}
\label{fig:scatter12}
\end{figure} 


We show that \Tool{} is effective in speeding up the state-of-the-art verifier RefineZono \cite{singh2019boosting} when verifying global properties. This baseline employs input splitting based on a strong branching strategy. 
Figure~\ref{fig:scatter12} presents the speedup achieved by \Tool{} over this baseline. Overall, \Tool{} achieves a 9.5x speedup in the int16 quantization case and a 3.1x speedup in the int8 quantization case. Previous work has observed that ACAS-XU properties take a large number of splits with most analyzers. For the int16 case, the average value of $|T^{\perturbedNetwork}_f|$ with our baseline is 285.4. The baseline takes a total of 305s time for verifying cases that have large tree $|T^{\perturbedNetwork}_f| > 5$. \Tool{} verifies those properties in 32s. 

\new{
\vspace{-0.07in}

\subsection{Random Weight Perturbations}
 \begin{wraptable}{r}{70mm}
\centering\tablesize
\caption{\Tool{} speedup on uniform random weight perturbations}
\vspace{-.15in}

{

\begin{tabular}
{c|c|c|c}
\toprule
& \multicolumn{3}{|c}{Weight perturbation}\\
Model & 2\% & 5\% & 10\% \\
\midrule
\neta & 1.65x & 1.57x & 0.87x \\
\netb & 1.97x & 0.57x & 0.57x \\
\netc & 1.29x & 1.09x & 0.69x  \\
\netd & 1.42x & 1.08x & 0.96x \\
\nete & 1.32x & 1.06x & 1.05x \\
\bottomrule
\end{tabular}
}
\label{table:random}
\end{wraptable}

In this experiment, we stress-test IVAN for incremental verification by applying uniform random perturbation on the DNN weights. Here, we perturbed each weight in the network by 2\%, 5\%, and 10\%. Even the smallest of these perturbations (2\%) to each of the weights already induces larger overall changes in the network than those caused by practical methods such as quantization, pruning, and fine-tuning that often non-uniformly affect specific layers of the network. For each network and perturbation, we run \Tool{} and the baseline to verify 100 properties and compute the average speedup of \Tool{} over the baseline.

Table~\ref{table:random} presents the average speedups obtained by \Tool{}. Each row shows the \Tool{} speedup under various weight perturbations for a particular network. We see that in most cases \Tool{} speedup reduces as the perturbations to the weights increase. It is because the specification tree for the perturbed network is no longer similar to the one for the original network. If \Tool{} is used in such cases, it uses suboptimal splits leading to higher verification time.
}
 


\vspace{-0.07in}
\section{Related Work}
\label{sec:rel}
\noindent \textbf{\bf Neural Network Verification:} Recent works introduced several techniques for verifying properties of neural networks~\cite{tjeng2017evaluating, bunel2020branch, ehlers2017formal, DBLP:conf/cav/KatzBDJK17, anderson2020strong, Pailoor19, wang2018neurify, wang2021beta, https://doi.org/10.48550/arxiv.2209.05446, 10.1145/3563324, yang2022provable}. 
For BaB-based complete verification, previous works used distinct strategies for ReLU splitting. \citet{ehlers2017formal} and \citet{katz2017reluplex} used random ReLU selection for splitting. \citet{wang2018neurify} computes scores based on gradient information to rank ambiguous ReLU nodes. Similarly, \citet{bunel2020branch} compute scores based on a formula based on the estimation equations in \cite{wong18convex}. \citet{ijcai2021p351} use coefficients of zonotopes for these scores.

\noindent \textbf{\bf Incremental Neural Network Verification:} \citet{DBLP:conf/cav/FischerSDSV22} presented the concept of sharing certificates between specifications. They reuse the proof for $L_\infty$ specification computed with abstract interpretation-based analyzers based on the notion of proof templates, for faster verification of patch and geometric perturbations.  \citet{DBLP:journals/pacmpl/UgareSM22} showed that the reusing of proof is possible between networks. It uses a similar concept of network adaptable proof templates. It is limited to certain properties (patch, geometric, $L_0$) and works with abstract interpretation-based incomplete verifiers. \citet{DBLP:journals/corr/abs-2106-12732} considers incremental incomplete verification of relatively small DNNs with last-layer perturbation. All of these works cannot handle incremental and complete verification of diverse specifications, which is the focus of our work

\new{\noindent \textbf{\bf Differential Neural Network Verification:} } ReluDiff \cite{DBLP:conf/icse/PaulsenWW20} presented the concept of differential neural network verification. The follow-up work of \cite{DBLP:conf/kbse/PaulsenWWW20} made it more scalable. ReluDiff can be used for bounding the difference in the output of an original network and a perturbed network corresponding to an input region. ReluDiff uses input splitting to perform complete differential verification. Our method is complementary to ReluDiff and can be used to speed up the complete differential verification with multiple perturbed networks, performing it incrementally. 
\citet{cheng2020continuous} reuse previous interval analysis results for the verification of the fully-connected networks where the specifications are only defined over the last linear layer of an updated network. In contrast, \Tool{} performs end-to-end verification and operates on a more general class of networks, specifications, and perturbations. 

\noindent \textbf{\bf Warm Starting Mixed Integer Linear Programming (MILP) Solvers:} State-of-the-art MILP solvers such as GUROBI \cite{gurobi2018} and CPLEX \cite{cplex2009v12} support warm starting that can accelerate the optimization performance. MILP can warm start based on initial solutions that are close to the optimal solution. This allows MILP solvers to avoid exploring paths that do not improve on the provided initial solution and can help the solver to converge faster. The exact implementation details of these closed-sourced commercial solvers are unavailable. Regardless, our experiments with MILP warm starting of GUROBI for incremental DNN verification showed insignificant speedup.  


\noindent \textbf{\bf Incremental Program Verification:} Incremental verification has improved the scalability of  traditional program verification to an industrial scale \cite{10.1145/2465449.2465456, Lachnech-etal:TACAS01, DBLP:conf/lics/OHearn18, DBLP:conf/pldi/0002CS21}. Incremental program analysis tasks reuse partial results \cite{5306334}, constraints \cite{10.1145/2393596.2393665} and precision information \cite{10.1145/2491411.2491429} from previous runs for faster analysis of individual commits.  
\new{
Frequently, the changes made by the program are limited to a small portion of the overall program (and its analysis requires significant  attention to the impact on control flow). whereas DNN updates typically alter the weights of multiple layers throughout the network (but with no impact on control flow). Therefore, incremental DNN verification presents a distinct challenge compared to the incremental verification of programs.
}

\noindent \textbf{\bf Incremental SMT Solvers:} Modern SMT solvers such as Z3 \cite{10.5555/1792734.1792766} and CVC5 \cite{DBLP:conf/tacas/BarbosaBBKLMMMN22} during constraint solving learn lemmas, which are later reused to solve similar problems. The incrementality of these solvers is restricted to the addition or deletion of constraints. They do not consider reuse in cases when the constraints are perturbed as in our case. 

\vspace{-0.1in}
\section{Conclusion}\label{sec:concl}

Current complete approaches for DNN verification re-run the verification every time the network is modified.
In this paper, we presented \Tool{}, the first general, incremental, and complete DNN verifier.
\Tool{} captures the trace of the BaB-based complete verification through the specification tree. 
We evaluated our \Tool{} on combinations of networks, properties, and updates. 
\Tool{} achieves up to \upto{} speedup and geometric mean speedup of \geomean{} in verifying DNN properties.

\section*{ACKNOWLEDGMENTS}
We thank the anonymous reviewers for their comments. This research was supported in part by NSF Grants No. CCF-1846354, CCF-1956374, CCF-2008883, CCF-2217144, CCF-2238079, CNS-2148583, USDA NIFA Grant No. NIFA-2024827 and Qualcomm innovation fellowship.

\bibliographystyle{ACM-Reference-Format}
\bibliography{paper}


\begin{thebibliography}{70}


\ifx \showCODEN    \undefined \def \showCODEN     #1{\unskip}     \fi
\ifx \showDOI      \undefined \def \showDOI       #1{#1}\fi
\ifx \showISBNx    \undefined \def \showISBNx     #1{\unskip}     \fi
\ifx \showISBNxiii \undefined \def \showISBNxiii  #1{\unskip}     \fi
\ifx \showISSN     \undefined \def \showISSN      #1{\unskip}     \fi
\ifx \showLCCN     \undefined \def \showLCCN      #1{\unskip}     \fi
\ifx \shownote     \undefined \def \shownote      #1{#1}          \fi
\ifx \showarticletitle \undefined \def \showarticletitle #1{#1}   \fi
\ifx \showURL      \undefined \def \showURL       {\relax}        \fi
\providecommand\bibfield[2]{#2}
\providecommand\bibinfo[2]{#2}
\providecommand\natexlab[1]{#1}
\providecommand\showeprint[2][]{arXiv:#2}

\bibitem[Akiba et~al\mbox{.}(2019)]%
        {optuna_2019}
\bibfield{author}{\bibinfo{person}{Takuya Akiba}, \bibinfo{person}{Shotaro
  Sano}, \bibinfo{person}{Toshihiko Yanase}, \bibinfo{person}{Takeru Ohta},
  {and} \bibinfo{person}{Masanori Koyama}.} \bibinfo{year}{2019}\natexlab{}.
\newblock \showarticletitle{Optuna: A Next-generation Hyperparameter
  Optimization Framework}. In \bibinfo{booktitle}{\emph{Proceedings of the 25rd
  {ACM} {SIGKDD} International Conference on Knowledge Discovery and Data
  Mining}}.
\newblock


\bibitem[Albarghouthi(2021)]%
        {albarghouthi-book}
\bibfield{author}{\bibinfo{person}{Aws Albarghouthi}.}
  \bibinfo{year}{2021}\natexlab{}.
\newblock \bibinfo{booktitle}{\emph{Introduction to Neural Network
  Verification}}.
\newblock \bibinfo{publisher}{verifieddeeplearning.com}.
\newblock
\showeprint[arxiv]{2109.10317}~[cs.LG]
\newblock
\shownote{\url{http://verifieddeeplearning.com}}.


\bibitem[Amato et~al\mbox{.}(2013)]%
        {AMATO201347}
\bibfield{author}{\bibinfo{person}{Filippo Amato}, \bibinfo{person}{Alberto
  López}, \bibinfo{person}{Eladia~María Peña-Méndez}, \bibinfo{person}{Petr
  Vaňhara}, \bibinfo{person}{Aleš Hampl}, {and} \bibinfo{person}{Josef
  Havel}.} \bibinfo{year}{2013}\natexlab{}.
\newblock \showarticletitle{Artificial neural networks in medical diagnosis}.
\newblock \bibinfo{journal}{\emph{Journal of Applied Biomedicine}}
  \bibinfo{volume}{11}, \bibinfo{number}{2} (\bibinfo{year}{2013}).
\newblock


\bibitem[Anderson et~al\mbox{.}(2019)]%
        {Pailoor19}
\bibfield{author}{\bibinfo{person}{Greg Anderson}, \bibinfo{person}{Shankara
  Pailoor}, \bibinfo{person}{Isil Dillig}, {and} \bibinfo{person}{Swarat
  Chaudhuri}.} \bibinfo{year}{2019}\natexlab{}.
\newblock \showarticletitle{Optimization and Abstraction: A Synergistic
  Approach for Analyzing Neural Network Robustness}. In
  \bibinfo{booktitle}{\emph{Proc. Programming Language Design and
  Implementation (PLDI)}}.
\newblock


\bibitem[Anderson et~al\mbox{.}(2020)]%
        {anderson2020strong}
\bibfield{author}{\bibinfo{person}{Ross Anderson}, \bibinfo{person}{Joey
  Huchette}, \bibinfo{person}{Will Ma}, \bibinfo{person}{Christian
  Tjandraatmadja}, {and} \bibinfo{person}{Juan~Pablo Vielma}.}
  \bibinfo{year}{2020}\natexlab{}.
\newblock \showarticletitle{Strong mixed-integer programming formulations for
  trained neural networks}.
\newblock \bibinfo{journal}{\emph{Mathematical Programming}}
  (\bibinfo{year}{2020}).
\newblock


\bibitem[Bak et~al\mbox{.}(2021)]%
        {DBLP:journals/corr/abs-2109-00498}
\bibfield{author}{\bibinfo{person}{Stanley Bak}, \bibinfo{person}{Changliu
  Liu}, {and} \bibinfo{person}{Taylor~T. Johnson}.}
  \bibinfo{year}{2021}\natexlab{}.
\newblock \showarticletitle{The Second International Verification of Neural
  Networks Competition {(VNN-COMP} 2021): Summary and Results}.
\newblock \bibinfo{journal}{\emph{CoRR}}  \bibinfo{volume}{abs/2109.00498}
  (\bibinfo{year}{2021}).
\newblock
\showeprint[arXiv]{2109.00498}
\urldef\tempurl%
\url{https://arxiv.org/abs/2109.00498}
\showURL{%
\tempurl}


\bibitem[Bak et~al\mbox{.}(2020)]%
        {DBLP:conf/cav/BakTHJ20}
\bibfield{author}{\bibinfo{person}{Stanley Bak}, \bibinfo{person}{Hoang{-}Dung
  Tran}, \bibinfo{person}{Kerianne Hobbs}, {and} \bibinfo{person}{Taylor~T.
  Johnson}.} \bibinfo{year}{2020}\natexlab{}.
\newblock \showarticletitle{Improved Geometric Path Enumeration for Verifying
  ReLU Neural Networks}. In \bibinfo{booktitle}{\emph{Computer Aided
  Verification - 32nd International Conference, {CAV} 2020, Los Angeles, CA,
  USA, July 21-24, 2020, Proceedings, Part {I}}}
  \emph{(\bibinfo{series}{Lecture Notes in Computer Science},
  Vol.~\bibinfo{volume}{12224})},
  \bibfield{editor}{\bibinfo{person}{Shuvendu~K. Lahiri} {and}
  \bibinfo{person}{Chao Wang}} (Eds.). \bibinfo{publisher}{Springer},
  \bibinfo{pages}{66--96}.
\newblock
\urldef\tempurl%
\url{https://doi.org/10.1007/978-3-030-53288-8\_4}
\showDOI{\tempurl}


\bibitem[Balunovic and Vechev(2020)]%
        {balunovic2020Adversarial}
\bibfield{author}{\bibinfo{person}{Mislav Balunovic} {and}
  \bibinfo{person}{Martin Vechev}.} \bibinfo{year}{2020}\natexlab{}.
\newblock \showarticletitle{Adversarial Training and Provable Defenses:
  Bridging the Gap}. In \bibinfo{booktitle}{\emph{International Conference on
  Learning Representations}}.
\newblock
\urldef\tempurl%
\url{https://openreview.net/forum?id=SJxSDxrKDr}
\showURL{%
\tempurl}


\bibitem[Barbosa et~al\mbox{.}(2022)]%
        {DBLP:conf/tacas/BarbosaBBKLMMMN22}
\bibfield{author}{\bibinfo{person}{Haniel Barbosa}, \bibinfo{person}{Clark~W.
  Barrett}, \bibinfo{person}{Martin Brain}, \bibinfo{person}{Gereon Kremer},
  \bibinfo{person}{Hanna Lachnitt}, \bibinfo{person}{Makai Mann},
  \bibinfo{person}{Abdalrhman Mohamed}, \bibinfo{person}{Mudathir Mohamed},
  \bibinfo{person}{Aina Niemetz}, \bibinfo{person}{Andres N{\"{o}}tzli},
  \bibinfo{person}{Alex Ozdemir}, \bibinfo{person}{Mathias Preiner},
  \bibinfo{person}{Andrew Reynolds}, \bibinfo{person}{Ying Sheng},
  \bibinfo{person}{Cesare Tinelli}, {and} \bibinfo{person}{Yoni Zohar}.}
  \bibinfo{year}{2022}\natexlab{}.
\newblock \showarticletitle{cvc5: {A} Versatile and Industrial-Strength {SMT}
  Solver}. In \bibinfo{booktitle}{\emph{Tools and Algorithms for the
  Construction and Analysis of Systems - 28th International Conference, {TACAS}
  2022, Held as Part of the European Joint Conferences on Theory and Practice
  of Software, {ETAPS} 2022, Munich, Germany, April 2-7, 2022, Proceedings,
  Part {I}}} \emph{(\bibinfo{series}{Lecture Notes in Computer Science},
  Vol.~\bibinfo{volume}{13243})}, \bibfield{editor}{\bibinfo{person}{Dana
  Fisman} {and} \bibinfo{person}{Grigore Rosu}} (Eds.).
  \bibinfo{publisher}{Springer}, \bibinfo{pages}{415--442}.
\newblock
\urldef\tempurl%
\url{https://doi.org/10.1007/978-3-030-99524-9\_24}
\showDOI{\tempurl}


\bibitem[Beyer et~al\mbox{.}(2013)]%
        {10.1145/2491411.2491429}
\bibfield{author}{\bibinfo{person}{Dirk Beyer}, \bibinfo{person}{Stefan
  L\"{o}we}, \bibinfo{person}{Evgeny Novikov}, \bibinfo{person}{Andreas
  Stahlbauer}, {and} \bibinfo{person}{Philipp Wendler}.}
  \bibinfo{year}{2013}\natexlab{}.
\newblock \showarticletitle{Precision Reuse for Efficient Regression
  Verification}. In \bibinfo{booktitle}{\emph{Proceedings of the 2013 9th Joint
  Meeting on Foundations of Software Engineering}} (Saint Petersburg, Russia)
  \emph{(\bibinfo{series}{ESEC/FSE 2013})}. \bibinfo{publisher}{Association for
  Computing Machinery}, \bibinfo{address}{New York, NY, USA},
  \bibinfo{pages}{389–399}.
\newblock
\showISBNx{9781450322379}
\urldef\tempurl%
\url{https://doi.org/10.1145/2491411.2491429}
\showDOI{\tempurl}


\bibitem[Blalock et~al\mbox{.}(2020)]%
        {DBLP:conf/mlsys/BlalockOFG20}
\bibfield{author}{\bibinfo{person}{Davis~W. Blalock}, \bibinfo{person}{Jose
  Javier~Gonzalez Ortiz}, \bibinfo{person}{Jonathan Frankle}, {and}
  \bibinfo{person}{John~V. Guttag}.} \bibinfo{year}{2020}\natexlab{}.
\newblock \showarticletitle{What is the State of Neural Network Pruning?}. In
  \bibinfo{booktitle}{\emph{Proceedings of Machine Learning and Systems 2020,
  MLSys 2020, Austin, TX, USA, March 2-4, 2020}}.
\newblock


\bibitem[Bojarski et~al\mbox{.}(2016)]%
        {bojarski2016end}
\bibfield{author}{\bibinfo{person}{Mariusz Bojarski}, \bibinfo{person}{Davide
  Del~Testa}, \bibinfo{person}{Daniel Dworakowski}, \bibinfo{person}{Bernhard
  Firner}, \bibinfo{person}{Beat Flepp}, \bibinfo{person}{Prasoon Goyal},
  \bibinfo{person}{Lawrence~D Jackel}, \bibinfo{person}{Mathew Monfort},
  \bibinfo{person}{Urs Muller}, \bibinfo{person}{Jiakai Zhang},
  {et~al\mbox{.}}} \bibinfo{year}{2016}\natexlab{}.
\newblock \showarticletitle{End to end learning for self-driving cars}.
\newblock \bibinfo{journal}{\emph{arXiv preprint arXiv:1604.07316}}
  (\bibinfo{year}{2016}).
\newblock


\bibitem[Bunel et~al\mbox{.}(2020b)]%
        {bunel2020branch}
\bibfield{author}{\bibinfo{person}{Rudy Bunel}, \bibinfo{person}{Jingyue Lu},
  \bibinfo{person}{Ilker Turkaslan}, \bibinfo{person}{Pushmeet Kohli},
  \bibinfo{person}{P Torr}, {and} \bibinfo{person}{P Mudigonda}.}
  \bibinfo{year}{2020}\natexlab{b}.
\newblock \showarticletitle{Branch and bound for piecewise linear neural
  network verification}.
\newblock \bibinfo{journal}{\emph{Journal of Machine Learning Research}}
  \bibinfo{volume}{21}, \bibinfo{number}{2020} (\bibinfo{year}{2020}).
\newblock


\bibitem[Bunel et~al\mbox{.}(2020a)]%
        {bunel2020efficient}
\bibfield{author}{\bibinfo{person}{Rudy~R Bunel}, \bibinfo{person}{Oliver
  Hinder}, \bibinfo{person}{Srinadh Bhojanapalli}, {and}
  \bibinfo{person}{Krishnamurthy Dvijotham}.} \bibinfo{year}{2020}\natexlab{a}.
\newblock \showarticletitle{An efficient nonconvex reformulation of stagewise
  convex optimization problems}.
\newblock \bibinfo{journal}{\emph{Advances in Neural Information Processing
  Systems}}  \bibinfo{volume}{33} (\bibinfo{year}{2020}).
\newblock


\bibitem[Chen et~al\mbox{.}(2022)]%
        {chen2022robust}
\bibfield{author}{\bibinfo{person}{Jiefeng Chen}, \bibinfo{person}{Yixuan Li},
  \bibinfo{person}{Xi Wu}, \bibinfo{person}{Yingyu Liang}, {and}
  \bibinfo{person}{Somesh Jha}.} \bibinfo{year}{2022}\natexlab{}.
\newblock \showarticletitle{Robust Out-of-distribution Detection for Neural
  Networks}. In \bibinfo{booktitle}{\emph{AAAI-22 Workshop on Adversarial
  Machine Learning and Beyond}}.
\newblock


\bibitem[Cheng and Yan(2020)]%
        {cheng2020continuous}
\bibfield{author}{\bibinfo{person}{Chih-Hong Cheng} {and}
  \bibinfo{person}{Rongjie Yan}.} \bibinfo{year}{2020}\natexlab{}.
\newblock \bibinfo{title}{Continuous Safety Verification of Neural Networks}.
\newblock
\newblock
\showeprint[arxiv]{2010.05689}~[cs.LG]


\bibitem[Cplex(2009)]%
        {cplex2009v12}
\bibfield{author}{\bibinfo{person}{IBM~ILOG Cplex}.}
  \bibinfo{year}{2009}\natexlab{}.
\newblock \showarticletitle{V12. 1: User’s Manual for CPLEX}.
\newblock \bibinfo{journal}{\emph{International Business Machines Corporation}}
  \bibinfo{volume}{46}, \bibinfo{number}{53} (\bibinfo{year}{2009}),
  \bibinfo{pages}{157}.
\newblock


\bibitem[De~Moura and Bj\o{}rner(2008)]%
        {10.5555/1792734.1792766}
\bibfield{author}{\bibinfo{person}{Leonardo De~Moura} {and}
  \bibinfo{person}{Nikolaj Bj\o{}rner}.} \bibinfo{year}{2008}\natexlab{}.
\newblock \showarticletitle{Z3: An Efficient SMT Solver}. In
  \bibinfo{booktitle}{\emph{Proceedings of the Theory and Practice of Software,
  14th International Conference on Tools and Algorithms for the Construction
  and Analysis of Systems}} (Budapest, Hungary)
  \emph{(\bibinfo{series}{TACAS'08/ETAPS'08})}.
  \bibinfo{publisher}{Springer-Verlag}, \bibinfo{address}{Berlin, Heidelberg},
  \bibinfo{pages}{337–340}.
\newblock
\showISBNx{3540787992}


\bibitem[Dong et~al\mbox{.}(2018)]%
        {Dong_2018_CVPR}
\bibfield{author}{\bibinfo{person}{Yinpeng Dong}, \bibinfo{person}{Fangzhou
  Liao}, \bibinfo{person}{Tianyu Pang}, \bibinfo{person}{Hang Su},
  \bibinfo{person}{Jun Zhu}, \bibinfo{person}{Xiaolin Hu}, {and}
  \bibinfo{person}{Jianguo Li}.} \bibinfo{year}{2018}\natexlab{}.
\newblock \showarticletitle{Boosting Adversarial Attacks With Momentum}. In
  \bibinfo{booktitle}{\emph{Proceedings of the IEEE Conference on Computer
  Vision and Pattern Recognition (CVPR)}}.
\newblock


\bibitem[Dutta et~al\mbox{.}(2017)]%
        {DBLP:journals/corr/abs-1709-09130}
\bibfield{author}{\bibinfo{person}{Souradeep Dutta}, \bibinfo{person}{Susmit
  Jha}, \bibinfo{person}{Sriram Sankaranarayanan}, {and}
  \bibinfo{person}{Ashish Tiwari}.} \bibinfo{year}{2017}\natexlab{}.
\newblock \showarticletitle{Output Range Analysis for Deep Neural Networks}.
\newblock \bibinfo{journal}{\emph{CoRR}}  \bibinfo{volume}{abs/1709.09130}
  (\bibinfo{year}{2017}).
\newblock
\showeprint[arXiv]{1709.09130}
\urldef\tempurl%
\url{http://arxiv.org/abs/1709.09130}
\showURL{%
\tempurl}


\bibitem[Ehlers(2017)]%
        {ehlers2017formal}
\bibfield{author}{\bibinfo{person}{Ruediger Ehlers}.}
  \bibinfo{year}{2017}\natexlab{}.
\newblock \showarticletitle{Formal verification of piece-wise linear
  feed-forward neural networks}. In \bibinfo{booktitle}{\emph{International
  Symposium on Automated Technology for Verification and Analysis}}.
\newblock


\bibitem[Ferrari et~al\mbox{.}(2022)]%
        {ferrari2022complete}
\bibfield{author}{\bibinfo{person}{Claudio Ferrari},
  \bibinfo{person}{Mark~Niklas Mueller}, \bibinfo{person}{Nikola
  Jovanovi{\'c}}, {and} \bibinfo{person}{Martin Vechev}.}
  \bibinfo{year}{2022}\natexlab{}.
\newblock \showarticletitle{Complete Verification via Multi-Neuron Relaxation
  Guided Branch-and-Bound}. In \bibinfo{booktitle}{\emph{International
  Conference on Learning Representations}}.
\newblock
\urldef\tempurl%
\url{https://openreview.net/forum?id=l_amHf1oaK}
\showURL{%
\tempurl}


\bibitem[Fischer et~al\mbox{.}(2022)]%
        {DBLP:conf/cav/FischerSDSV22}
\bibfield{author}{\bibinfo{person}{Marc Fischer}, \bibinfo{person}{Christian
  Sprecher}, \bibinfo{person}{Dimitar~I. Dimitrov}, \bibinfo{person}{Gagandeep
  Singh}, {and} \bibinfo{person}{Martin~T. Vechev}.}
  \bibinfo{year}{2022}\natexlab{}.
\newblock \showarticletitle{Shared Certificates for Neural Network
  Verification}. In \bibinfo{booktitle}{\emph{Computer Aided Verification -
  34th International Conference, {CAV} 2022, Haifa, Israel, August 7-10, 2022,
  Proceedings, Part {I}}} \emph{(\bibinfo{series}{Lecture Notes in Computer
  Science}, Vol.~\bibinfo{volume}{13371})},
  \bibfield{editor}{\bibinfo{person}{Sharon Shoham} {and}
  \bibinfo{person}{Yakir Vizel}} (Eds.). \bibinfo{publisher}{Springer},
  \bibinfo{pages}{127--148}.
\newblock
\urldef\tempurl%
\url{https://doi.org/10.1007/978-3-031-13185-1\_7}
\showDOI{\tempurl}


\bibitem[Fromherz et~al\mbox{.}(2021)]%
        {fromherz2021fast}
\bibfield{author}{\bibinfo{person}{Aymeric Fromherz}, \bibinfo{person}{Klas
  Leino}, \bibinfo{person}{Matt Fredrikson}, \bibinfo{person}{Bryan Parno},
  {and} \bibinfo{person}{Corina Pasareanu}.} \bibinfo{year}{2021}\natexlab{}.
\newblock \showarticletitle{Fast Geometric Projections for Local Robustness
  Certification}. In \bibinfo{booktitle}{\emph{International Conference on
  Learning Representations}}.
\newblock
\urldef\tempurl%
\url{https://openreview.net/forum?id=zWy1uxjDdZJ}
\showURL{%
\tempurl}


\bibitem[Fu and Li(2022)]%
        {fu2022sound}
\bibfield{author}{\bibinfo{person}{Feisi Fu} {and} \bibinfo{person}{Wenchao
  Li}.} \bibinfo{year}{2022}\natexlab{}.
\newblock \showarticletitle{Sound and Complete Neural Network Repair with
  Minimality and Locality Guarantees}. In
  \bibinfo{booktitle}{\emph{International Conference on Learning
  Representations}}.
\newblock
\urldef\tempurl%
\url{https://openreview.net/forum?id=xS8AMYiEav3}
\showURL{%
\tempurl}


\bibitem[Gehr et~al\mbox{.}(2018)]%
        {gehr2018ai2}
\bibfield{author}{\bibinfo{person}{Timon Gehr}, \bibinfo{person}{Matthew
  Mirman}, \bibinfo{person}{Dana Drachsler-Cohen}, \bibinfo{person}{Petar
  Tsankov}, \bibinfo{person}{Swarat Chaudhuri}, {and} \bibinfo{person}{Martin
  Vechev}.} \bibinfo{year}{2018}\natexlab{}.
\newblock \showarticletitle{Ai2: Safety and robustness certification of neural
  networks with abstract interpretation}. In \bibinfo{booktitle}{\emph{2018
  IEEE Symposium on Security and Privacy (SP)}}.
\newblock


\bibitem[Gholami et~al\mbox{.}(2021)]%
        {DBLP:journals/corr/abs-2103-13630}
\bibfield{author}{\bibinfo{person}{Amir Gholami}, \bibinfo{person}{Sehoon Kim},
  \bibinfo{person}{Zhen Dong}, \bibinfo{person}{Zhewei Yao},
  \bibinfo{person}{Michael~W. Mahoney}, {and} \bibinfo{person}{Kurt Keutzer}.}
  \bibinfo{year}{2021}\natexlab{}.
\newblock \showarticletitle{A Survey of Quantization Methods for Efficient
  Neural Network Inference}.
\newblock \bibinfo{journal}{\emph{CoRR}}  \bibinfo{volume}{abs/2103.13630}
  (\bibinfo{year}{2021}).
\newblock
\showeprint[arxiv]{2103.13630}


\bibitem[Gokhale et~al\mbox{.}(2021)]%
        {GokhaleAKTBY:21}
\bibfield{author}{\bibinfo{person}{Tejas Gokhale}, \bibinfo{person}{Rushil
  Anirudh}, \bibinfo{person}{Bhavya Kailkhura}, \bibinfo{person}{Jayaraman~J.
  Thiagarajan}, \bibinfo{person}{Chitta Baral}, {and} \bibinfo{person}{Yezhou
  Yang}.} \bibinfo{year}{2021}\natexlab{}.
\newblock \showarticletitle{Attribute-Guided Adversarial Training for
  Robustness to Natural Perturbations}. In \bibinfo{booktitle}{\emph{{AAAI}}}.
  \bibinfo{publisher}{{AAAI} Press}, \bibinfo{pages}{7574--7582}.
\newblock


\bibitem[{Gurobi Optimization, LLC}(2018)]%
        {gurobi2018}
\bibfield{author}{\bibinfo{person}{{Gurobi Optimization, LLC}}.}
  \bibinfo{year}{2018}\natexlab{}.
\newblock \bibinfo{title}{Gurobi Optimizer Reference Manual}.
\newblock
\newblock


\bibitem[Henriksen and Lomuscio(2021)]%
        {ijcai2021p351}
\bibfield{author}{\bibinfo{person}{Patrick Henriksen} {and}
  \bibinfo{person}{Alessio Lomuscio}.} \bibinfo{year}{2021}\natexlab{}.
\newblock \showarticletitle{DEEPSPLIT: An Efficient Splitting Method for Neural
  Network Verification via Indirect Effect Analysis}. In
  \bibinfo{booktitle}{\emph{Proceedings of the Thirtieth International Joint
  Conference on Artificial Intelligence, {IJCAI-21}}},
  \bibfield{editor}{\bibinfo{person}{Zhi-Hua Zhou}} (Ed.).
  \bibinfo{publisher}{International Joint Conferences on Artificial
  Intelligence Organization}, \bibinfo{pages}{2549--2555}.
\newblock
\urldef\tempurl%
\url{https://doi.org/10.24963/ijcai.2021/351}
\showDOI{\tempurl}
\newblock
\shownote{Main Track}.


\bibitem[Johnson et~al\mbox{.}(2013)]%
        {10.1145/2465449.2465456}
\bibfield{author}{\bibinfo{person}{Kenneth Johnson}, \bibinfo{person}{Radu
  Calinescu}, {and} \bibinfo{person}{Shinji Kikuchi}.}
  \bibinfo{year}{2013}\natexlab{}.
\newblock \showarticletitle{An Incremental Verification Framework for
  Component-Based Software Systems}. In \bibinfo{booktitle}{\emph{Proceedings
  of the 16th International ACM Sigsoft Symposium on Component-Based Software
  Engineering}} (Vancouver, British Columbia, Canada)
  \emph{(\bibinfo{series}{CBSE '13})}. \bibinfo{publisher}{Association for
  Computing Machinery}, \bibinfo{address}{New York, NY, USA},
  \bibinfo{pages}{33–42}.
\newblock
\showISBNx{9781450321228}
\urldef\tempurl%
\url{https://doi.org/10.1145/2465449.2465456}
\showDOI{\tempurl}


\bibitem[Julian et~al\mbox{.}(2018)]%
        {acasxu:18}
\bibfield{author}{\bibinfo{person}{Kyle~D. Julian}, \bibinfo{person}{Mykel~J.
  Kochenderfer}, {and} \bibinfo{person}{Michael~P. Owen}.}
  \bibinfo{year}{2018}\natexlab{}.
\newblock \showarticletitle{Deep Neural Network Compression for Aircraft
  Collision Avoidance Systems}.
\newblock \bibinfo{journal}{\emph{CoRR}}  \bibinfo{volume}{abs/1810.04240}
  (\bibinfo{year}{2018}).
\newblock


\bibitem[Julian et~al\mbox{.}(2019)]%
        {Julian_2019}
\bibfield{author}{\bibinfo{person}{Kyle~D. Julian}, \bibinfo{person}{Mykel~J.
  Kochenderfer}, {and} \bibinfo{person}{Michael~P. Owen}.}
  \bibinfo{year}{2019}\natexlab{}.
\newblock \showarticletitle{Deep Neural Network Compression for Aircraft
  Collision Avoidance Systems}.
\newblock \bibinfo{journal}{\emph{Journal of Guidance, Control, and Dynamics}}
  \bibinfo{volume}{42}, \bibinfo{number}{3} (\bibinfo{date}{mar}
  \bibinfo{year}{2019}), \bibinfo{pages}{598--608}.
\newblock
\urldef\tempurl%
\url{https://doi.org/10.2514/1.g003724}
\showDOI{\tempurl}


\bibitem[Kabaha and Drachsler-Cohen(2022)]%
        {https://doi.org/10.48550/arxiv.2209.05446}
\bibfield{author}{\bibinfo{person}{Anan Kabaha} {and} \bibinfo{person}{Dana
  Drachsler-Cohen}.} \bibinfo{year}{2022}\natexlab{}.
\newblock \bibinfo{title}{Boosting Robustness Verification of Semantic Feature
  Neighborhoods}.
\newblock
\newblock
\urldef\tempurl%
\url{https://doi.org/10.48550/ARXIV.2209.05446}
\showDOI{\tempurl}


\bibitem[Katz et~al\mbox{.}(2017a)]%
        {katz2017reluplex}
\bibfield{author}{\bibinfo{person}{Guy Katz}, \bibinfo{person}{Clark Barrett},
  \bibinfo{person}{David~L Dill}, \bibinfo{person}{Kyle Julian}, {and}
  \bibinfo{person}{Mykel~J Kochenderfer}.} \bibinfo{year}{2017}\natexlab{a}.
\newblock \showarticletitle{Reluplex: An efficient SMT solver for verifying
  deep neural networks}. In \bibinfo{booktitle}{\emph{International Conference
  on Computer Aided Verification}}.
\newblock


\bibitem[Katz et~al\mbox{.}(2017b)]%
        {DBLP:conf/cav/KatzBDJK17}
\bibfield{author}{\bibinfo{person}{Guy Katz}, \bibinfo{person}{Clark~W.
  Barrett}, \bibinfo{person}{David~L. Dill}, \bibinfo{person}{Kyle Julian},
  {and} \bibinfo{person}{Mykel~J. Kochenderfer}.}
  \bibinfo{year}{2017}\natexlab{b}.
\newblock \showarticletitle{Reluplex: An Efficient {SMT} Solver for Verifying
  Deep Neural Networks}. In \bibinfo{booktitle}{\emph{Computer Aided
  Verification - 29th International Conference, {CAV} 2017, Heidelberg,
  Germany, July 24-28, 2017, Proceedings, Part {I}}}
  \emph{(\bibinfo{series}{Lecture Notes in Computer Science},
  Vol.~\bibinfo{volume}{10426})}.
\newblock
\urldef\tempurl%
\url{https://doi.org/10.1007/978-3-319-63387-9\_5}
\showDOI{\tempurl}


\bibitem[Lakhnech et~al\mbox{.}(2001)]%
        {Lachnech-etal:TACAS01}
\bibfield{author}{\bibinfo{person}{Yassine Lakhnech}, \bibinfo{person}{Saddek
  Bensalem}, \bibinfo{person}{Sergey Berezin}, {and} \bibinfo{person}{Sam
  Owre}.} \bibinfo{year}{2001}\natexlab{}.
\newblock \showarticletitle{Incremental Verification by Abstraction}. In
  \bibinfo{booktitle}{\emph{Tools and Algorithms for the Construction and
  Analysis of Systems: 7th International Conference, {TACAS} 2001}},
  \bibfield{editor}{\bibinfo{person}{T.~Margaria} {and}
  \bibinfo{person}{W.~Yi}} (Eds.), Vol.~\bibinfo{volume}{2031}.
  \bibinfo{publisher}{Springer-Verlag}, \bibinfo{address}{Genova, Italy},
  \bibinfo{pages}{98--112}.
\newblock


\bibitem[Laurel et~al\mbox{.}(2021)]%
        {9586276}
\bibfield{author}{\bibinfo{person}{Jacob Laurel}, \bibinfo{person}{Rem Yang},
  \bibinfo{person}{Atharva Sehgal}, \bibinfo{person}{Shubham Ugare}, {and}
  \bibinfo{person}{Sasa Misailovic}.} \bibinfo{year}{2021}\natexlab{}.
\newblock \showarticletitle{Statheros: Compiler for Efficient Low-Precision
  Probabilistic Programming}. In \bibinfo{booktitle}{\emph{Design Automation
  Conference (DAC)}}. \bibinfo{pages}{787--792}.
\newblock


\bibitem[Laurel et~al\mbox{.}(2022)]%
        {10.1145/3563324}
\bibfield{author}{\bibinfo{person}{Jacob Laurel}, \bibinfo{person}{Rem Yang},
  \bibinfo{person}{Shubham Ugare}, \bibinfo{person}{Robert Nagel},
  \bibinfo{person}{Gagandeep Singh}, {and} \bibinfo{person}{Sasa Misailovic}.}
  \bibinfo{year}{2022}\natexlab{}.
\newblock \showarticletitle{A General Construction for Abstract Interpretation
  of Higher-Order Automatic Differentiation}.
\newblock \bibinfo{journal}{\emph{Proc. ACM Program. Lang.}}
  \bibinfo{volume}{6}, \bibinfo{number}{OOPSLA2}, Article
  \bibinfo{articleno}{161} (\bibinfo{date}{oct} \bibinfo{year}{2022}),
  \bibinfo{numpages}{29}~pages.
\newblock
\urldef\tempurl%
\url{https://doi.org/10.1145/3563324}
\showDOI{\tempurl}


\bibitem[Madry et~al\mbox{.}(2017)]%
        {madry2017towards}
\bibfield{author}{\bibinfo{person}{Aleksander Madry},
  \bibinfo{person}{Aleksandar Makelov}, \bibinfo{person}{Ludwig Schmidt},
  \bibinfo{person}{Dimitris Tsipras}, {and} \bibinfo{person}{Adrian Vladu}.}
  \bibinfo{year}{2017}\natexlab{}.
\newblock \showarticletitle{Towards deep learning models resistant to
  adversarial attacks}.
\newblock \bibinfo{journal}{\emph{arXiv preprint arXiv:1706.06083}}
  (\bibinfo{year}{2017}).
\newblock


\bibitem[M{\"u}ller et~al\mbox{.}(2021)]%
        {muller2020neural}
\bibfield{author}{\bibinfo{person}{Christoph M{\"u}ller},
  \bibinfo{person}{Francois Serre}, \bibinfo{person}{Gagandeep Singh},
  \bibinfo{person}{Markus P{\"u}schel}, {and} \bibinfo{person}{Martin Vechev}.}
  \bibinfo{year}{2021}\natexlab{}.
\newblock \showarticletitle{Scaling Polyhedral Neural Network Verification on
  GPUs}.
\newblock \bibinfo{journal}{\emph{Proc. Machine Learning and Systems (MLSys)}}
  (\bibinfo{year}{2021}).
\newblock


\bibitem[O'Hearn(2018)]%
        {DBLP:conf/lics/OHearn18}
\bibfield{author}{\bibinfo{person}{Peter~W. O'Hearn}.}
  \bibinfo{year}{2018}\natexlab{}.
\newblock \showarticletitle{Continuous Reasoning: Scaling the impact of formal
  methods}. In \bibinfo{booktitle}{\emph{Proceedings of the 33rd Annual
  {ACM/IEEE} Symposium on Logic in Computer Science, {LICS} 2018, Oxford, UK,
  July 09-12, 2018}}, \bibfield{editor}{\bibinfo{person}{Anuj Dawar} {and}
  \bibinfo{person}{Erich Gr{\"{a}}del}} (Eds.). \bibinfo{publisher}{{ACM}},
  \bibinfo{pages}{13--25}.
\newblock
\urldef\tempurl%
\url{https://doi.org/10.1145/3209108.3209109}
\showDOI{\tempurl}


\bibitem[Palma et~al\mbox{.}(2021)]%
        {depalma2021scaling}
\bibfield{author}{\bibinfo{person}{Alessandro~De Palma},
  \bibinfo{person}{Harkirat~S. Behl}, \bibinfo{person}{Rudy~R. Bunel},
  \bibinfo{person}{Philip H.~S. Torr}, {and} \bibinfo{person}{M.~Pawan Kumar}.}
  \bibinfo{year}{2021}\natexlab{}.
\newblock \showarticletitle{Scaling the Convex Barrier with Active Sets}. In
  \bibinfo{booktitle}{\emph{9th International Conference on Learning
  Representations, {ICLR} 2021, Virtual Event, Austria, May 3-7, 2021}}.
\newblock


\bibitem[Paulsen et~al\mbox{.}(2020a)]%
        {DBLP:conf/icse/PaulsenWW20}
\bibfield{author}{\bibinfo{person}{Brandon Paulsen}, \bibinfo{person}{Jingbo
  Wang}, {and} \bibinfo{person}{Chao Wang}.} \bibinfo{year}{2020}\natexlab{a}.
\newblock \showarticletitle{ReluDiff: differential verification of deep neural
  networks}. In \bibinfo{booktitle}{\emph{{ICSE} '20: 42nd International
  Conference on Software Engineering, Seoul, South Korea, 27 June - 19 July,
  2020}}.
\newblock
\urldef\tempurl%
\url{https://doi.org/10.1145/3377811.3380337}
\showDOI{\tempurl}


\bibitem[Paulsen et~al\mbox{.}(2020b)]%
        {DBLP:conf/kbse/PaulsenWWW20}
\bibfield{author}{\bibinfo{person}{Brandon Paulsen}, \bibinfo{person}{Jingbo
  Wang}, \bibinfo{person}{Jiawei Wang}, {and} \bibinfo{person}{Chao Wang}.}
  \bibinfo{year}{2020}\natexlab{b}.
\newblock \showarticletitle{{NEURODIFF:} Scalable Differential Verification of
  Neural Networks using Fine-Grained Approximation}. In
  \bibinfo{booktitle}{\emph{35th {IEEE/ACM} International Conference on
  Automated Software Engineering, {ASE} 2020, Melbourne, Australia, September
  21-25, 2020}}.
\newblock
\urldef\tempurl%
\url{https://doi.org/10.1145/3324884.3416560}
\showDOI{\tempurl}


\bibitem[Salman et~al\mbox{.}(2019)]%
        {DBLP:conf/nips/SalmanY0HZ19}
\bibfield{author}{\bibinfo{person}{Hadi Salman}, \bibinfo{person}{Greg Yang},
  \bibinfo{person}{Huan Zhang}, \bibinfo{person}{Cho{-}Jui Hsieh}, {and}
  \bibinfo{person}{Pengchuan Zhang}.} \bibinfo{year}{2019}\natexlab{}.
\newblock \showarticletitle{A Convex Relaxation Barrier to Tight Robustness
  Verification of Neural Networks}. In \bibinfo{booktitle}{\emph{Advances in
  Neural Information Processing Systems 32: Annual Conference on Neural
  Information Processing Systems 2019, NeurIPS 2019, December 8-14, 2019,
  Vancouver, BC, Canada}}.
\newblock


\bibitem[Singh et~al\mbox{.}(2019a)]%
        {singh2019beyond}
\bibfield{author}{\bibinfo{person}{Gagandeep Singh}, \bibinfo{person}{Rupanshu
  Ganvir}, \bibinfo{person}{Markus P{\"u}schel}, {and} \bibinfo{person}{Martin
  Vechev}.} \bibinfo{year}{2019}\natexlab{a}.
\newblock \showarticletitle{Beyond the single neuron convex barrier for neural
  network certification}. In \bibinfo{booktitle}{\emph{Advances in Neural
  Information Processing Systems}}.
\newblock


\bibitem[Singh et~al\mbox{.}(2018)]%
        {singh2018fast}
\bibfield{author}{\bibinfo{person}{Gagandeep Singh}, \bibinfo{person}{Timon
  Gehr}, \bibinfo{person}{Matthew Mirman}, \bibinfo{person}{Markus
  P{\"u}schel}, {and} \bibinfo{person}{Martin Vechev}.}
  \bibinfo{year}{2018}\natexlab{}.
\newblock \showarticletitle{Fast and effective robustness certification}.
\newblock \bibinfo{journal}{\emph{Advances in Neural Information Processing
  Systems}}  \bibinfo{volume}{31} (\bibinfo{year}{2018}).
\newblock


\bibitem[Singh et~al\mbox{.}(2019b)]%
        {singh2019abstract}
\bibfield{author}{\bibinfo{person}{Gagandeep Singh}, \bibinfo{person}{Timon
  Gehr}, \bibinfo{person}{Markus P{\"u}schel}, {and} \bibinfo{person}{Martin
  Vechev}.} \bibinfo{year}{2019}\natexlab{b}.
\newblock \showarticletitle{An abstract domain for certifying neural networks}.
\newblock \bibinfo{journal}{\emph{Proceedings of the ACM on Programming
  Languages}} \bibinfo{volume}{3}, \bibinfo{number}{POPL}
  (\bibinfo{year}{2019}).
\newblock


\bibitem[Singh et~al\mbox{.}(2019c)]%
        {singh2019boosting}
\bibfield{author}{\bibinfo{person}{Gagandeep Singh}, \bibinfo{person}{Timon
  Gehr}, \bibinfo{person}{Markus Püschel}, {and} \bibinfo{person}{Martin
  Vechev}.} \bibinfo{year}{2019}\natexlab{c}.
\newblock \showarticletitle{Boosting Robustness Certification of Neural
  Networks}. In \bibinfo{booktitle}{\emph{International Conference on Learning
  Representations}}.
\newblock


\bibitem[Sotoudeh and Thakur(2019)]%
        {SotoudehT19}
\bibfield{author}{\bibinfo{person}{Matthew Sotoudeh} {and}
  \bibinfo{person}{Aditya~V. Thakur}.} \bibinfo{year}{2019}\natexlab{}.
\newblock \showarticletitle{Computing Linear Restrictions of Neural Networks}.
  In \bibinfo{booktitle}{\emph{Advances in Neural Information Processing
  Systems 32: Annual Conference on Neural Information Processing Systems 2019,
  NeurIPS 2019, December 8-14, 2019, Vancouver, BC, Canada}}.
\newblock


\bibitem[Stein et~al\mbox{.}(2021)]%
        {DBLP:conf/pldi/0002CS21}
\bibfield{author}{\bibinfo{person}{Benno Stein},
  \bibinfo{person}{Bor{-}Yuh~Evan Chang}, {and} \bibinfo{person}{Manu
  Sridharan}.} \bibinfo{year}{2021}\natexlab{}.
\newblock \showarticletitle{Demanded abstract interpretation}. In
  \bibinfo{booktitle}{\emph{{PLDI} '21: 42nd {ACM} {SIGPLAN} International
  Conference on Programming Language Design and Implementation, Virtual Event,
  Canada, June 20-25, 2021}}, \bibfield{editor}{\bibinfo{person}{Stephen~N.
  Freund} {and} \bibinfo{person}{Eran Yahav}} (Eds.).
  \bibinfo{publisher}{{ACM}}, \bibinfo{pages}{282--295}.
\newblock
\urldef\tempurl%
\url{https://doi.org/10.1145/3453483.3454044}
\showDOI{\tempurl}


\bibitem[Szegedy et~al\mbox{.}(2014)]%
        {DBLP:journals/corr/SzegedyZSBEGF13}
\bibfield{author}{\bibinfo{person}{Christian Szegedy},
  \bibinfo{person}{Wojciech Zaremba}, \bibinfo{person}{Ilya Sutskever},
  \bibinfo{person}{Joan Bruna}, \bibinfo{person}{Dumitru Erhan},
  \bibinfo{person}{Ian~J. Goodfellow}, {and} \bibinfo{person}{Rob Fergus}.}
  \bibinfo{year}{2014}\natexlab{}.
\newblock \showarticletitle{Intriguing properties of neural networks}. In
  \bibinfo{booktitle}{\emph{2nd International Conference on Learning
  Representations, {ICLR} 2014, Banff, AB, Canada, April 14-16, 2014,
  Conference Track Proceedings}}.
\newblock


\bibitem[Tajbakhsh et~al\mbox{.}(2016)]%
        {tajbakhsh2016convolutional}
\bibfield{author}{\bibinfo{person}{Nima Tajbakhsh}, \bibinfo{person}{Jae~Y
  Shin}, \bibinfo{person}{Suryakanth~R Gurudu}, \bibinfo{person}{R~Todd Hurst},
  \bibinfo{person}{Christopher~B Kendall}, \bibinfo{person}{Michael~B Gotway},
  {and} \bibinfo{person}{Jianming Liang}.} \bibinfo{year}{2016}\natexlab{}.
\newblock \showarticletitle{Convolutional neural networks for medical image
  analysis: Full training or fine tuning?}
\newblock \bibinfo{journal}{\emph{IEEE transactions on medical imaging}}
  \bibinfo{volume}{35}, \bibinfo{number}{5} (\bibinfo{year}{2016}),
  \bibinfo{pages}{1299--1312}.
\newblock


\bibitem[TFLite(2017)]%
        {tf_quantization}
\bibfield{author}{\bibinfo{person}{TFLite}.} \bibinfo{year}{2017}\natexlab{}.
\newblock \bibinfo{title}{TF Lite post-training quantization}.
\newblock
  \bibinfo{howpublished}{https://www.tensorflow.org/lite/performance/post\_training\_quantization}.
\newblock


\bibitem[Tjeng et~al\mbox{.}(2017)]%
        {tjeng2017evaluating}
\bibfield{author}{\bibinfo{person}{Vincent Tjeng}, \bibinfo{person}{Kai Xiao},
  {and} \bibinfo{person}{Russ Tedrake}.} \bibinfo{year}{2017}\natexlab{}.
\newblock \showarticletitle{Evaluating robustness of neural networks with mixed
  integer programming}.
\newblock \bibinfo{journal}{\emph{arXiv preprint arXiv:1711.07356}}
  (\bibinfo{year}{2017}).
\newblock


\bibitem[Ugare et~al\mbox{.}(2022)]%
        {DBLP:journals/pacmpl/UgareSM22}
\bibfield{author}{\bibinfo{person}{Shubham Ugare}, \bibinfo{person}{Gagandeep
  Singh}, {and} \bibinfo{person}{Sasa Misailovic}.}
  \bibinfo{year}{2022}\natexlab{}.
\newblock \showarticletitle{Proof transfer for fast certification of multiple
  approximate neural networks}.
\newblock \bibinfo{journal}{\emph{Proc. {ACM} Program. Lang.}}
  \bibinfo{volume}{6}, \bibinfo{number}{{OOPSLA}} (\bibinfo{year}{2022}),
  \bibinfo{pages}{1--29}.
\newblock
\urldef\tempurl%
\url{https://doi.org/10.1145/3527319}
\showDOI{\tempurl}


\bibitem[Urban and Miné(2021)]%
        {https://doi.org/10.48550/arxiv.2104.02466}
\bibfield{author}{\bibinfo{person}{Caterina Urban} {and}
  \bibinfo{person}{Antoine Miné}.} \bibinfo{year}{2021}\natexlab{}.
\newblock \bibinfo{title}{A Review of Formal Methods applied to Machine
  Learning}.
\newblock
\newblock
\urldef\tempurl%
\url{https://doi.org/10.48550/ARXIV.2104.02466}
\showDOI{\tempurl}


\bibitem[Visser et~al\mbox{.}(2012)]%
        {10.1145/2393596.2393665}
\bibfield{author}{\bibinfo{person}{Willem Visser}, \bibinfo{person}{Jaco
  Geldenhuys}, {and} \bibinfo{person}{Matthew~B. Dwyer}.}
  \bibinfo{year}{2012}\natexlab{}.
\newblock \showarticletitle{Green: Reducing, Reusing and Recycling Constraints
  in Program Analysis}. In \bibinfo{booktitle}{\emph{Proceedings of the ACM
  SIGSOFT 20th International Symposium on the Foundations of Software
  Engineering}} (Cary, North Carolina) \emph{(\bibinfo{series}{FSE '12})}.
  \bibinfo{publisher}{Association for Computing Machinery},
  \bibinfo{address}{New York, NY, USA}, Article \bibinfo{articleno}{58},
  \bibinfo{numpages}{11}~pages.
\newblock
\showISBNx{9781450316149}
\urldef\tempurl%
\url{https://doi.org/10.1145/2393596.2393665}
\showDOI{\tempurl}


\bibitem[Wang et~al\mbox{.}(2018)]%
        {wang2018neurify}
\bibfield{author}{\bibinfo{person}{Shiqi Wang}, \bibinfo{person}{Kexin Pei},
  \bibinfo{person}{Justin Whitehouse}, \bibinfo{person}{Junfeng Yang}, {and}
  \bibinfo{person}{Suman Jana}.} \bibinfo{year}{2018}\natexlab{}.
\newblock \showarticletitle{Efficient formal safety analysis of neural
  networks}. In \bibinfo{booktitle}{\emph{Advances in Neural Information
  Processing Systems}}.
\newblock


\bibitem[Wang et~al\mbox{.}(2021)]%
        {wang2021beta}
\bibfield{author}{\bibinfo{person}{Shiqi Wang}, \bibinfo{person}{Huan Zhang},
  \bibinfo{person}{Kaidi Xu}, \bibinfo{person}{Xue Lin}, \bibinfo{person}{Suman
  Jana}, \bibinfo{person}{Cho-Jui Hsieh}, {and} \bibinfo{person}{J~Zico
  Kolter}.} \bibinfo{year}{2021}\natexlab{}.
\newblock \showarticletitle{Beta-CROWN: Efficient Bound Propagation with
  Per-neuron Split Constraints for Complete and Incomplete Neural Network
  Verification}.
\newblock \bibinfo{journal}{\emph{arXiv preprint arXiv:2103.06624}}
  (\bibinfo{year}{2021}).
\newblock


\bibitem[Wei and Liu(2021)]%
        {DBLP:journals/corr/abs-2106-12732}
\bibfield{author}{\bibinfo{person}{Tianhao Wei} {and} \bibinfo{person}{Changliu
  Liu}.} \bibinfo{year}{2021}\natexlab{}.
\newblock \showarticletitle{Online Verification of Deep Neural Networks under
  Domain or Weight Shift}.
\newblock \bibinfo{journal}{\emph{CoRR}}  \bibinfo{volume}{abs/2106.12732}
  (\bibinfo{year}{2021}).
\newblock
\showeprint[arXiv]{2106.12732}
\urldef\tempurl%
\url{https://arxiv.org/abs/2106.12732}
\showURL{%
\tempurl}


\bibitem[Weiss et~al\mbox{.}(2016)]%
        {weiss2016survey}
\bibfield{author}{\bibinfo{person}{Karl Weiss}, \bibinfo{person}{Taghi~M
  Khoshgoftaar}, {and} \bibinfo{person}{DingDing Wang}.}
  \bibinfo{year}{2016}\natexlab{}.
\newblock \showarticletitle{A survey of transfer learning}.
\newblock \bibinfo{journal}{\emph{Journal of Big data}} \bibinfo{volume}{3},
  \bibinfo{number}{1} (\bibinfo{year}{2016}), \bibinfo{pages}{1--40}.
\newblock


\bibitem[Wong and Kolter(2018a)]%
        {wong2018provable}
\bibfield{author}{\bibinfo{person}{Eric Wong} {and} \bibinfo{person}{Zico
  Kolter}.} \bibinfo{year}{2018}\natexlab{a}.
\newblock \showarticletitle{Provable defenses against adversarial examples via
  the convex outer adversarial polytope}. In
  \bibinfo{booktitle}{\emph{International Conference on Machine Learning}}.
\newblock


\bibitem[Wong and Kolter(2018b)]%
        {wong18convex}
\bibfield{author}{\bibinfo{person}{Eric Wong} {and} \bibinfo{person}{Zico
  Kolter}.} \bibinfo{year}{2018}\natexlab{b}.
\newblock \showarticletitle{Provable Defenses against Adversarial Examples via
  the Convex Outer Adversarial Polytope}. In
  \bibinfo{booktitle}{\emph{Proceedings of the 35th International Conference on
  Machine Learning}}.
\newblock


\bibitem[Xu et~al\mbox{.}(2020)]%
        {Lirpa:20}
\bibfield{author}{\bibinfo{person}{Kaidi Xu}, \bibinfo{person}{Zhouxing Shi},
  \bibinfo{person}{Huan Zhang}, \bibinfo{person}{Yihan Wang},
  \bibinfo{person}{Kai{-}Wei Chang}, \bibinfo{person}{Minlie Huang},
  \bibinfo{person}{Bhavya Kailkhura}, \bibinfo{person}{Xue Lin}, {and}
  \bibinfo{person}{Cho{-}Jui Hsieh}.} \bibinfo{year}{2020}\natexlab{}.
\newblock \showarticletitle{Automatic Perturbation Analysis for Scalable
  Certified Robustness and Beyond}.
\newblock  (\bibinfo{year}{2020}).
\newblock


\bibitem[Yang et~al\mbox{.}(2009)]%
        {5306334}
\bibfield{author}{\bibinfo{person}{Guowei Yang}, \bibinfo{person}{Matthew~B.
  Dwyer}, {and} \bibinfo{person}{Gregg Rothermel}.}
  \bibinfo{year}{2009}\natexlab{}.
\newblock \showarticletitle{Regression model checking}. In
  \bibinfo{booktitle}{\emph{2009 IEEE International Conference on Software
  Maintenance}}. \bibinfo{pages}{115--124}.
\newblock
\urldef\tempurl%
\url{https://doi.org/10.1109/ICSM.2009.5306334}
\showDOI{\tempurl}


\bibitem[Yang et~al\mbox{.}(2022)]%
        {yang2022provable}
\bibfield{author}{\bibinfo{person}{Rem Yang}, \bibinfo{person}{Jacob Laurel},
  \bibinfo{person}{Sasa Misailovic}, {and} \bibinfo{person}{Gagandeep Singh}.}
  \bibinfo{year}{2022}\natexlab{}.
\newblock \bibinfo{title}{Provable Defense Against Geometric Transformations}.
\newblock
\newblock
\showeprint[arxiv]{2207.11177}~[cs.LG]


\bibitem[Zhang et~al\mbox{.}(2022)]%
        {zhang2022general}
\bibfield{author}{\bibinfo{person}{Huan Zhang}, \bibinfo{person}{Shiqi Wang},
  \bibinfo{person}{Kaidi Xu}, \bibinfo{person}{Linyi Li}, \bibinfo{person}{Bo
  Li}, \bibinfo{person}{Suman Jana}, \bibinfo{person}{Cho-Jui Hsieh}, {and}
  \bibinfo{person}{J~Zico Kolter}.} \bibinfo{year}{2022}\natexlab{}.
\newblock \showarticletitle{General Cutting Planes for Bound-Propagation-Based
  Neural Network Verification}. In \bibinfo{booktitle}{\emph{Advances in Neural
  Information Processing Systems}}, \bibfield{editor}{\bibinfo{person}{Alice~H.
  Oh}, \bibinfo{person}{Alekh Agarwal}, \bibinfo{person}{Danielle Belgrave},
  {and} \bibinfo{person}{Kyunghyun Cho}} (Eds.).
\newblock
\urldef\tempurl%
\url{https://openreview.net/forum?id=5haAJAcofjc}
\showURL{%
\tempurl}


\bibitem[Zhang et~al\mbox{.}(2018)]%
        {zhang2018crown}
\bibfield{author}{\bibinfo{person}{Huan Zhang}, \bibinfo{person}{Tsui-Wei
  Weng}, \bibinfo{person}{Pin-Yu Chen}, \bibinfo{person}{Cho-Jui Hsieh}, {and}
  \bibinfo{person}{Luca Daniel}.} \bibinfo{year}{2018}\natexlab{}.
\newblock \showarticletitle{Efficient neural network robustness certification
  with general activation functions}. In \bibinfo{booktitle}{\emph{Advances in
  neural information processing systems}}.
\newblock


\end{thebibliography}

\section{Appendix}

\subsection{Evaluation Results}
\label{sec:stats}

\begin{figure}[!htbp]
\centering
\vspace{-.2in}
\begin{subfigure}[b]{0.41\textwidth}
 \centering
 \includegraphics[width=\textwidth]{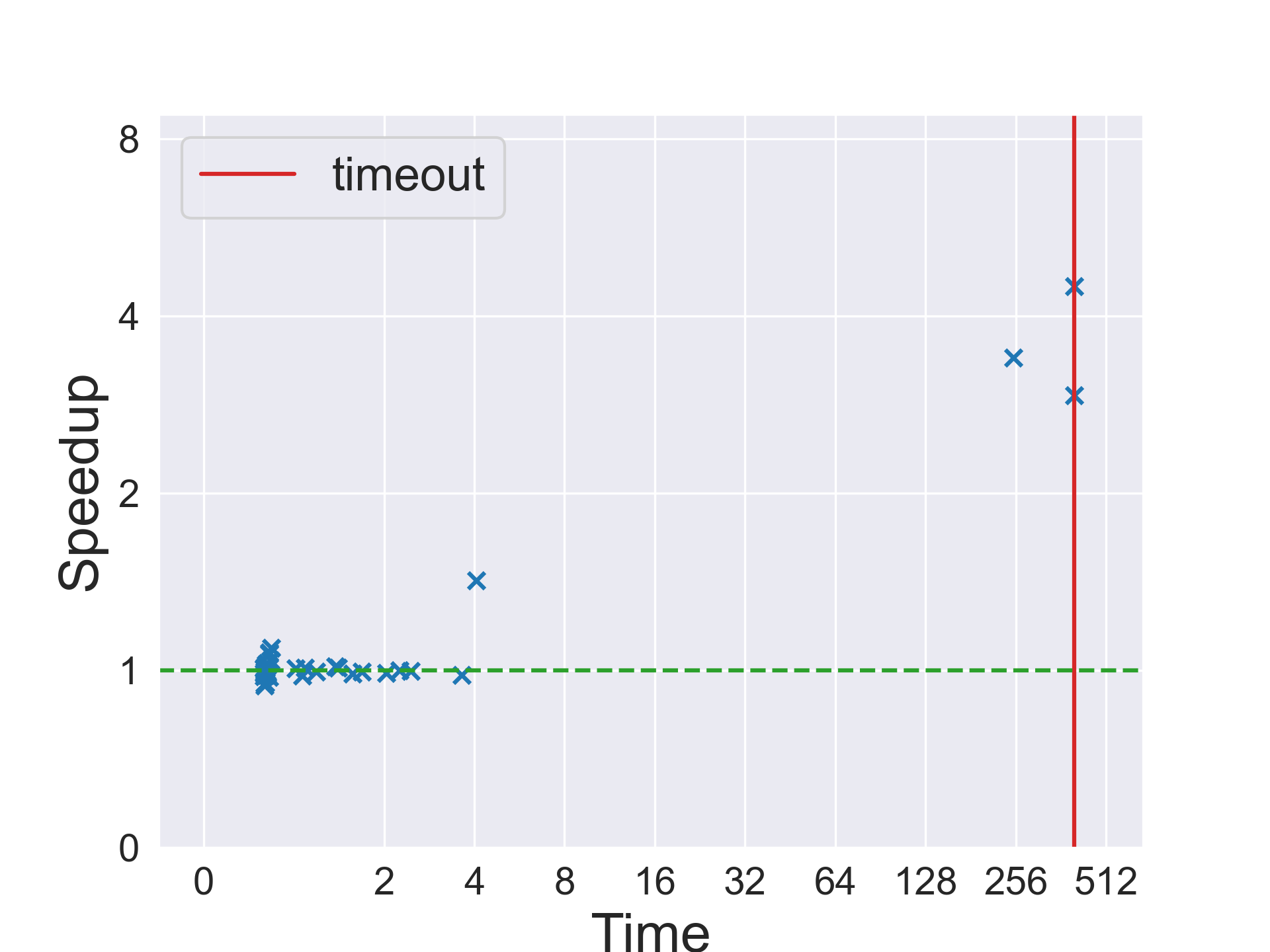}
 \caption{\netc with INT16 quantization}
\end{subfigure}
\hspace{10mm}
\begin{subfigure}[b]{0.41\textwidth}
 \centering
 \includegraphics[width=\textwidth]{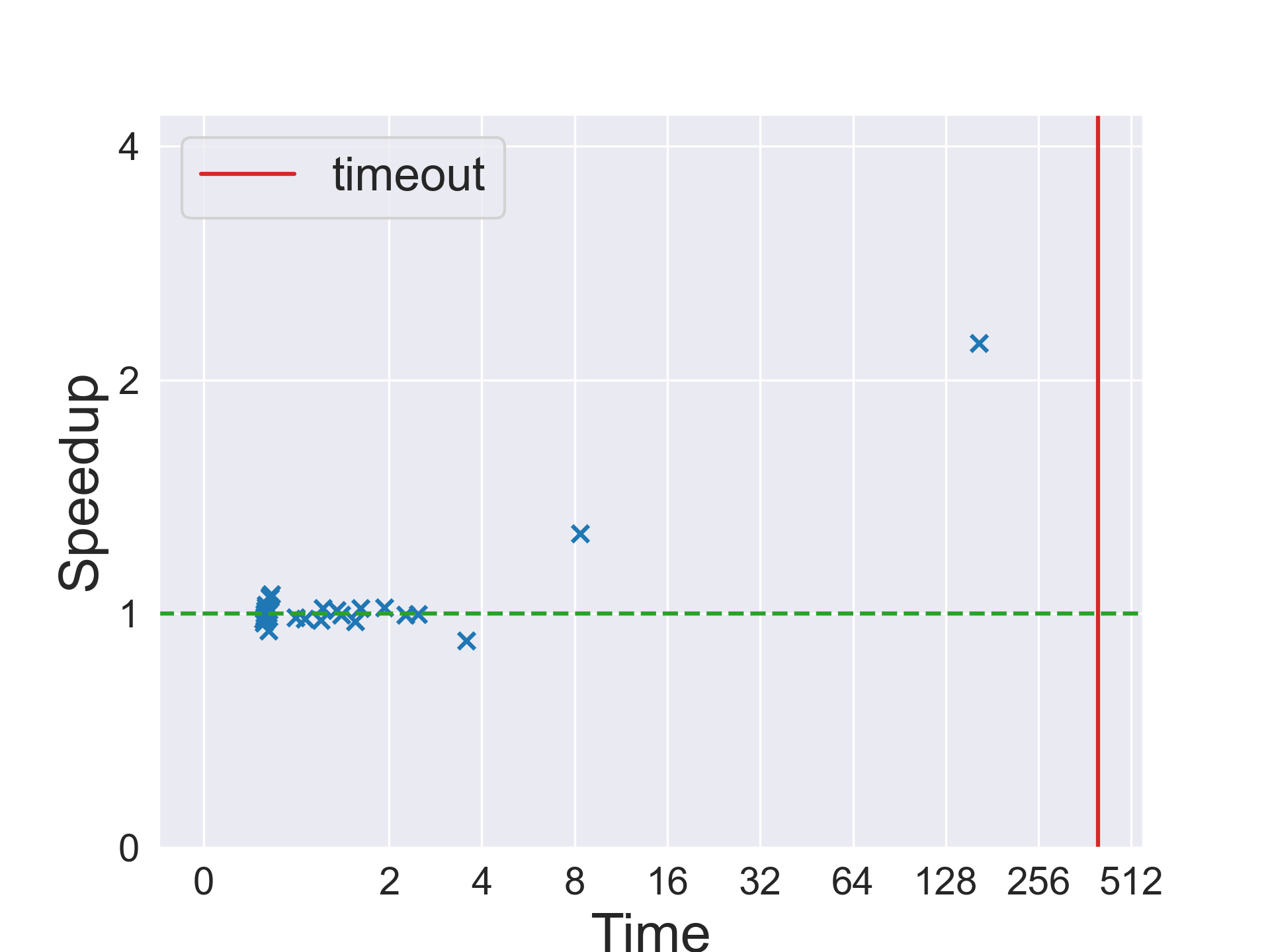}
 \caption{\netc with INT8 quantization}
\end{subfigure}
\hspace{10mm}
\hfill
\begin{subfigure}[b]{0.41\textwidth}
 \centering
 \includegraphics[width=\textwidth]{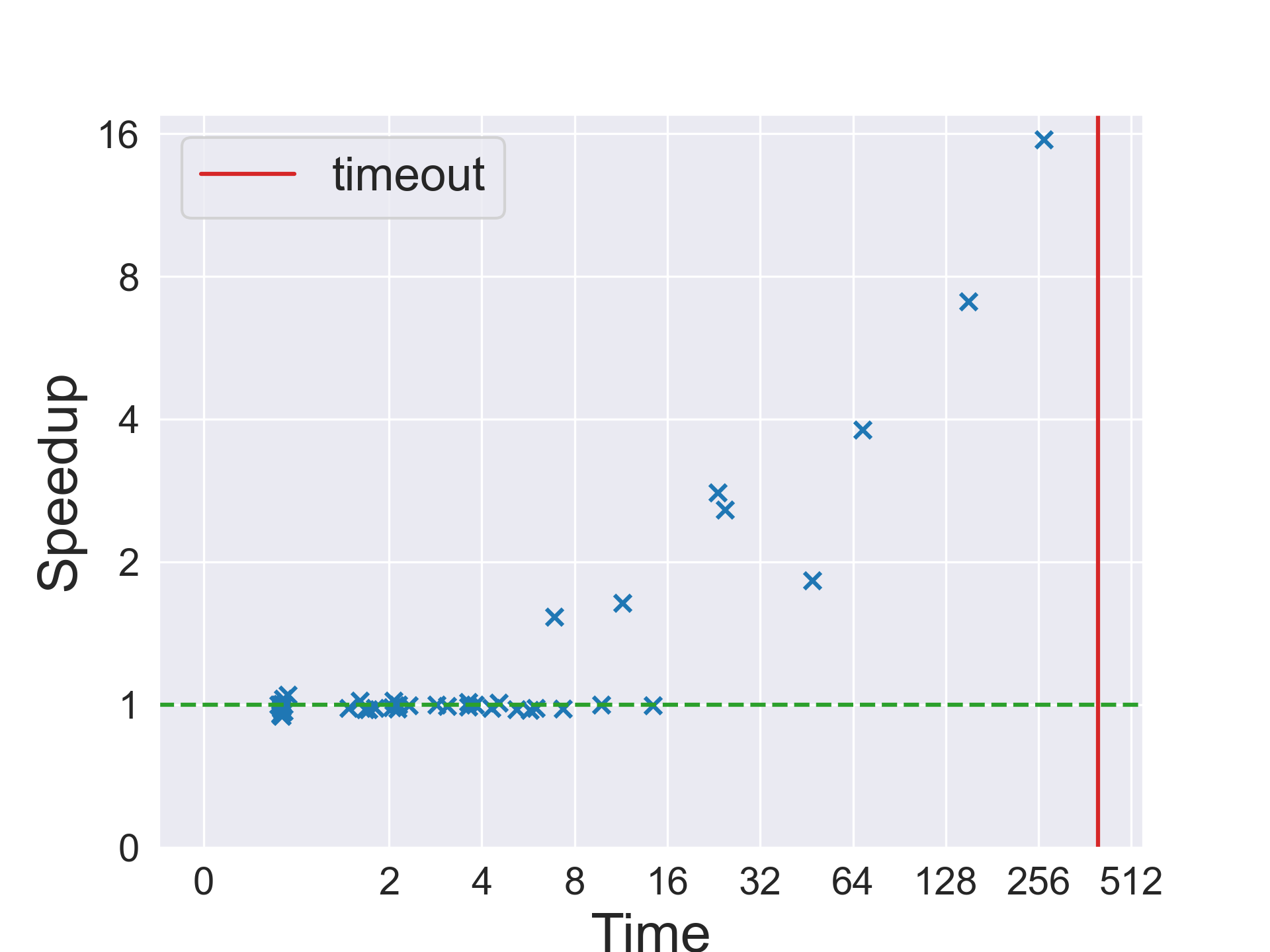}
 \caption{\nete with INT16 quantization}
\end{subfigure}
\hspace{10mm}
\begin{subfigure}[b]{0.41\textwidth}
 \centering
 \includegraphics[width=\textwidth]{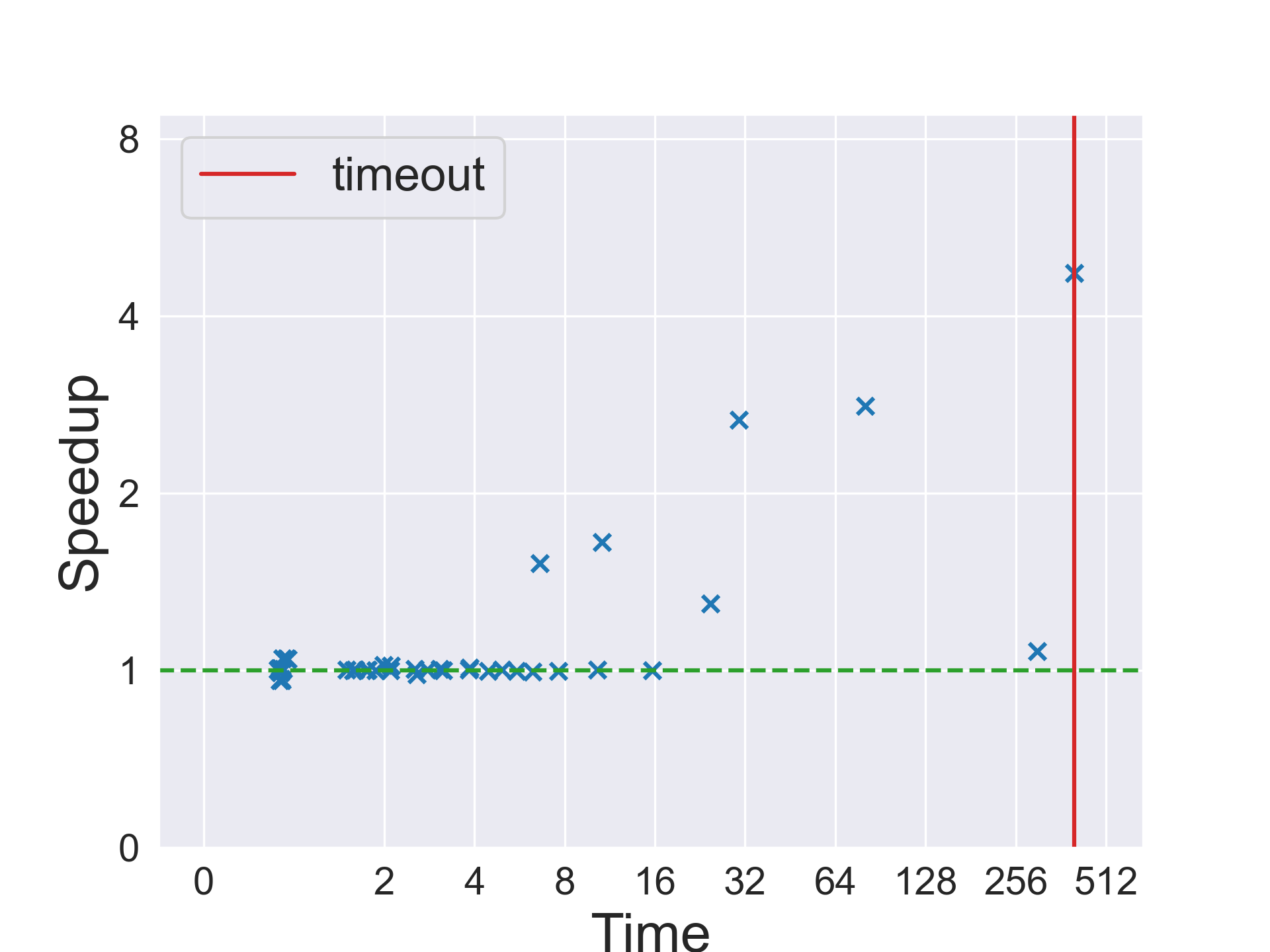}
 \caption{\nete with INT8 quantization}
\end{subfigure}
\vspace{-.15in}
\caption{\Tool{} speedup for the verification of locus robustness properties.}
\label{fig:scatter3}
\end{figure} 

We give more detailed statistics on our experiments in Table~\ref{table:summary}. We separate the results into two cases 
easy instances ($|T^N_f| \leq 5$) and hard instances ($|T^N_f| > 5$). \Tool{} focuses on hard instances, thus we observe more speedup in those cases. Column v/c/u shows the number of instances that are $\ver$, $\counterex$, and $\unknown$ respectively. Columns $Cost_\textit{base}$ and $Cost_\textit{\Tool{}}$ present the average number of analyzer calls made by baseline and \Tool{} respectively. Columns $Solved_\textit{base}$ and $Solved_\textit{\Tool{}}$ show the number of instances solved by the baseline and \Tool{} respectively. Columns $Time_\textit{base}$ and $Time_\textit{\Tool{}}$ give the time taken for verification by the baseline and \Tool{} respectively.

\begin{table*}[!htbp]
\centering
\tablesize
\resizebox{\columnwidth}{!}{
\begin{tabular}{l l |l l l | l l | l l l l | l l l l}
\toprule
& & & & & & & \multicolumn{4}{c|}{$|T^N_f| \leq 5$} & \multicolumn{4}{c}{$|T^N_f| > 5$} \\
Model & Perturbation & Cases & v/c/u & v/c/u & $Cost_\textit{base}$ & $Cost_\textit{\Tool{}}$ & $Solved_\textit{base}$ & $Solved_\textit{\Tool{}}$ & $Time_\textit{base}$ & $Time_\textit{\Tool{}}$ & $Solved_\textit{base}$ & $Solved_\textit{\Tool{}}$ & $Time_\textit{base}$ & $Time_\textit{\Tool{}}$\\ 
\midrule
\neta & int16 & 100 & 85/13/2              & 85/13/2              & 7.53  & 1.94  & 86 & 86 & 19.06  & 18.06  & 12 & 12 & 160.39  & 22.81  \\
 & int8 & 100 & 85/13/2              & 85/13/2              & 7.24  & 4.04  & 86 & 86 & 20.73  & 20.14  & 12 & 12 & 174.87  & 79.84  \\
\netb & int16 & 96  & 74/8/14              & 76/8/12              & 8.64  & 3.74  & 72 & 72 & 123.22 & 113.51 & 10 & 12 & 779.4   & 298.91 \\
 & int8 & 96  & 71/8/17              & 75/8/13              & 8.95  & 3.71  & 70 & 70 & 110.59 & 104.65 & 9  & 13 & 713.35  & 284.01 \\
\netc & int16 & 56  & 33/8/15              & 35/8/13              & 30.16 & 7.12  & 40 & 40 & 44.07  & 42.85  & 1  & 3  & 1054.29 & 300.36 \\
 & int16 & 56  & 33/8/15              & 33/8/15              & 5.63  & 3.07  & 40 & 40 & 47.77  & 46.09  & 1  & 1  & 164.28  & 73.61  \\
\netd & int8 & 73  & 29/32/12             & 31/32/10             & 11.16 & 3.44  & 56 & 56 & 131.14 & 127.89 & 5  & 7  & 1002.33 & 272.6  \\
 & int16 & 73  & 30/32/11             & 32/32/9              & 17.56 & 11.19 & 56 & 56 & 123.08 & 123.98 & 6  & 8  & 1168.83 & 829.82 \\
\nete & int16 & 59  & 27/23/9              & 27/23/9              & 6.32  & 1.76  & 45 & 45 & 152.37 & 131.44 & 5  & 5  & 557.32  & 90.24  \\
 & int8 & 59  & 25/23/11             & 26/23/10             & 9.04  & 4.47  & 45 & 45 & 151.32 & 137.61 & 3  & 4  & 812.87  & 396.85 \\
\hline
\end{tabular}
}
\caption{Summary of statistics on verifying all models with baseline and \Tool{}}
\label{table:summary}
\vspace{-.25in}
\end{table*}

We observe that \Tool{} significantly improves the verifier's performance for verification of specifications that need specification trees of size greater than 5. These proofs contribute to a substantial portion of overall time, and thus \Tool{} offers a large overall speedup. We observe that \Tool{} offers insignificant improvement in the verification time for cases with a small specification tree $T^N_f$. These cases are verified by the baseline in less time. There is not too much our techniques such as reordering and pruning and can do for the already compact trees. 

\subsection{Proofs for Theorems}
\label{sec:proofs}


\timeinc*
\begin{proof}
The incremental verifier starts by bounding $\leaves{\Tinit}$ to check the property, and if needed, it recursively branches the nodes further. 

Consequently, the verifier performs the bounding step for $\leaves{\Tinit}$ and all the new nodes in $T^{\perturbedNetwork}_f$ added to $\Tinit$. 

It performs the branching step for all newly added internal nodes in the specification tree. The number of new internal nodes can be computed as $(|\nodes{T^{\perturbedNetwork}_f}| - |\leaves{T^{\perturbedNetwork}_f}| - |\nodes{\Tinit}| + |\leaves{\Tinit}|)$

Accordingly, we can compute:

 \begin{align*}
\Timeb(\Tinit, T^{\perturbedNetwork}_f) &= \tbo \cdot (|\nodes{T^{\perturbedNetwork}_f}| - |\nodes{\Tinit}| + |\leaves{\Tinit}|)\\
&+ \tbr \cdot (|\nodes{T^{\perturbedNetwork}_f}| - |\leaves{T^{\perturbedNetwork}_f}| - |\nodes{\Tinit}| + |\leaves{\Tinit}|)
\\
&= (\tbo + \tbr) \cdot (|\nodes{T^{\perturbedNetwork}_f}| - |\nodes{\Tinit}| + |\leaves{\Tinit}|) - \tbr \cdot |\leaves{T^{\perturbedNetwork}_f}| \\
\end{align*}

\end{proof}

\invariance*
\begin{proof}
We prove this claim using structural induction on the specification tree. \\
It is trivially true for the specification tree of a single node. Since $\spec{\nroot} = \phi \to \psi$, we can conclude $\spec{\nroot} \Longleftrightarrow  \phi \to \psi$. \\ 
For the inductive step, we assume that for a tree $T_i$ with $i$ splits, the hypothesis is true. If we split node $n_s \in \leaves{T_i}$, we get a tree $T_{i+1}$ with leaf nodes $n_s^+$ and $n_s^-$.  \\
\begingroup
\allowdisplaybreaks
\begin{align*}
   & \Bigg( \bigwedge_{n \in leaves(T_i)} \spec{n} \Bigg) \Longleftrightarrow  (\phi \to \psi)
   && \text{(Induction hypothesis)}\\
   & \Bigg( \bigwedge_{n \in leaves(T_i)/n_s} \spec{n} \Bigg) \land \spec{n_s}  \Longleftrightarrow  (\phi \to \psi) && \\
    & \Bigg( \bigwedge_{n \in \leaves{T_i}/n_s} \spec{n} \Bigg) \land \spec{n_s^+} \land \spec{n_s^-} \Longleftrightarrow  (\phi \to \psi) && \text{(From equation~\ref{eq:split})}\\
    & \Bigg( \bigwedge_{n \in \leaves{T_{i+1}}} \spec{n} \Bigg) \Longleftrightarrow  (\phi \to \psi) && \text{(Combining previous equations)}\\
\end{align*}
\endgroup

Hence, the invariance is true for the specification tree $T_{i+1}$ as well. This completes our induction, and hence, our hypothesis is proved. 
\end{proof}

\begin{restatable}{lemma}{termination}(Termination).
\label{thm:termination}
Algorithm~\ref{alg:algorithm_main} always terminates. 
\label{theorem:complete}
\end{restatable}
\begin{proof}
At each specification tree node, we split a ReLU $r \in \mathcal{R}$ that was not split before. Since $|\mathcal{R}|$ is finite, the specification tree cannot have depth $> |\mathcal{R}|$. Thus, Algorithm~\ref{alg:incver} always terminates. \\
\end{proof}
\sound*
\begin{proof}
Let $\treeset$ be the set of specification trees over the architecture $\mathcal{N}$, such that $N, \perturbedNetwork \in \mathcal{N}$\\
By construction $T^N_f \in \treeset$. Algorithm~\ref{alg:algorithm_main} prunes $T^N_f$ to get the tree $\Tinit$. We see that $\Tinit \in \treeset$ since our deletion operation preserves the specification tree property of the tree. \\
Further branching from $\Tinit$ leads to the final tree $T^{\perturbedNetwork}_f$ during the incremental verification. The branching step performs the multiple $\add$ operations, and therefore $T^{\perturbedNetwork}_f \in \treeset$. \\
Algorithm~\ref{alg:incver} removes a node from the active list only when it is verified. (Line~\ref{line:active}) and $A$ is  a sound analyzer for the bounding step for verifying each node (Definition~\ref{def:sound}). \\ 
Thus, if Algorithm~\ref{alg:algorithm_main} returns $\ver$ then for each leaf node $\node \in \leaves{T}$, $\spec{n}$ holds. \\
Since $T^{\perturbedNetwork}_f \in \treeset$, we can use the  Lemma~\ref{lemma:invariance} and conclude that the property $(\phi, \psi)$ must hold.
\end{proof}

\complete*
\begin{proof}
The proof of termination is in Lemma~\ref{thm:termination}.
We prove the claim, $(\phi, \psi)$ holds then the Algorithm~\ref{alg:algorithm_main} returns $\ver$ through contradiction. Suppose Algorithm~\ref{alg:algorithm_main} does not return $\ver$. \\
Since the algorithm always terminates, it must terminate with a $\counterex$. \\
From Lemma~\ref{lemma:invariance} we know that $(\phi \to \psi) \implies \spec{n}$. This can be transformed to $\lnot \spec{n} \implies \lnot(\phi \to \psi)$ \\ 
Thus, if our algorithm returns a $\counterex$ for a specification tree node $n$ that implies $\lnot \spec{n}$ holds for some node $n$.\\
Hence, this statement implies $\lnot(\phi \to \psi)$ \\ 
This contradicts the assumption of this theorem. Hence, Algorithm~\ref{alg:algorithm_main} must return $\ver$. 
\end{proof}

\subsection{Proofs for Network Perturbation Bounds}
\label{sec:proofs2}


\perturba*
\begin{proof}
We first show that if $\EpsNorm \leq \frac{|\ProblemMin(\FeasibleReg(\OrgNetwork_{\Layers}, T))|}{\|\Lpc\|_2 \cdot \MaxNorm(N, T)}$ then $\neg \SolvFunc(\OrgNetwork, T) \implies \neg \SolvFunc(\perturbedNetwork, T)$.
The specification tree $T$ could not verify the property on $N$ then $\exists Y \in \FeasibleReg(\OrgNetwork_{\Layers}, T)$ such that $\Lpc^TY = \ProblemMin(\FeasibleReg(\OrgNetwork_{\Layers}, T)) < 0$. 
We show that the same specification tree ($T$) can not prove the property on any $\perturbedNetwork \in \mathcal{M}(N, \EpsNorm)$ by showing that $\ProblemMin(\perturbedNetwork_{\Layers}, T) < 0$.
In the following part of proof we show that $\exists Y' \in \FeasibleReg(\perturbedNetwork_{\Layers}, T)$ such that $C^TY' < 0$ which makes $\ProblemMin(\FeasibleReg(\perturbedNetwork_{\Layers}, T)) < 0$.
\begin{equation*}
\begin{split}
    \Lpc^{T} Y' &= \Lpc^T Y + \Lpc^{T}(Y' - Y) \\
             &\leq \Lpc^T Y + \|\Lpc\|_2\|(Y' - Y)\|_2 \\
             &\leq \Lpc^T Y + \|\Lpc\|_2 \cdot \EpsNorm \cdot \MaxNorm(\OrgNetwork_{\Layers}, T) \;\; \text{(Using Lemma~\ref{lem:feasiblesolution})}\\
             &< \ProblemMin(\FeasibleReg(N_{\Layers}, T))  + |\ProblemMin(\FeasibleReg(N_{\Layers}, T))| \leq 0 \;\; \text{given}\;\;\ProblemMin(\FeasibleReg(N_{\Layers}, T)) < 0 
\end{split}
\end{equation*}
We now show that $\EpsNorm \leq \frac{|\ProblemMin(\FeasibleReg(\OrgNetwork_{\Layers}, T))|}{\|\Lpc\|_2 \cdot \MaxNorm(N, T)}$ then $ \SolvFunc(\OrgNetwork, T) \implies \SolvFunc(\perturbedNetwork, T)$. 
We prove this by contradiction. Suppose $\ProblemMin(\perturbedNetwork_{\Layers}, T) < 0$ then $\exists Y' \in \FeasibleReg(\perturbedNetwork_{\Layers}, T)$ such that $\Lpc^{T}Y' < 0$. 
Swapping $\OrgNetwork$ with $\perturbedNetwork$ in lemma~\ref{lem:feasiblesolution} we can show $Y \in \FeasibleReg(\OrgNetwork_{\Layers}, T)$ such that $\|Y - Y'\|_2 \leq \EpsNorm \cdot \MaxNorm(\perturbedNetwork, T)$. 
Given perturbation is done only at the final layer $\forall i \in [\Layers - 1]$ $\OrgNetwork_{i} = \perturbedNetwork_{i}$ which implies $\MaxNorm(\OrgNetwork, T) = \MaxNorm(\perturbedNetwork, T)$. 
As $Y \in \FeasibleReg(\OrgNetwork_{\Layers}, T)$ then $\Lpc^Ty \geq \ProblemMin(\FeasibleReg(\OrgNetwork_{\Layers}, T))$.
\begin{equation*}
\begin{split}
        \Lpc^{T} Y' &= \Lpc^T Y + \Lpc^{T}(Y' - Y) \\
             &\geq \Lpc^T Y - \|\Lpc\|_2\|(Y' - Y)\|_2 \\
             &\geq \Lpc^T Y - \|\Lpc\|_2 \cdot \EpsNorm \cdot \MaxNorm(\OrgNetwork, T) \;\; \text{(Using Lemma~\ref{lem:feasiblesolution})} \\
             &\geq \ProblemMin(\FeasibleReg(\OrgNetwork_{\Layers}, T))  - |\ProblemMin(\FeasibleReg(\OrgNetwork_{\Layers}, T))| \geq 0\;\; \text{given}\;\;\ProblemMin(\FeasibleReg(\OrgNetwork_{\Layers}, T)) \geq 0 
\end{split}
\end{equation*}
The above derivation shows that $\Lpc^TY' \geq 0$ which contradicts the assumptions that $\Lpc^TY' < 0$ and $\ProblemMin(\FeasibleReg(\perturbedNetwork, T)) < 0$.
\end{proof}
\

\begin{restatable}{theorem}{terminationguarantee}
\label{thm:termination_guarantee}
The incremental verification time on any perturbed network $\perturbedNetwork \in \mathcal{M}(N, \EpsNorm)$ with $\EpsNorm \leq \frac{|\ProblemMin(\FeasibleReg(N_{\Layers}, T))|}{\|\Lpc\|_2 \cdot \MaxNorm(\OrgNetwork, T)}$ is $\tbo \cdot |\leaves{T}|$ provided $\Tinit = T$ and $\SolvFunc(N, T) = True$.
\end{restatable}
\begin{proof}
This result directly follows from Theorem~\ref{thm:perturb1}. As $\SolvFunc(\OrgNetwork, T) = True$, the proposed algorithm will terminate within $\leaves{T}$ number of bounding steps.
\end{proof}

\begin{equation}
    \ProblemMin(N, \treesetforN) = \min_{T \in \treesetforN} \frac{|\ProblemMin(\FeasibleReg(\OrgNetwork_{\Layers}, T))|}{\MaxNorm(N, T)}
\end{equation}

\noindent For any network $\perturbedNetwork \in \mathcal{M}(N, \EpsNorm)$ with $\EpsNorm \leq \frac{|\ProblemMin(N_{\Layers}, \treesetforN)|}{\|\Lpc\|_2}$ the baseline verifier can only verify the property with $T^{N}_{f}$. (using results from Theorem~\ref{thm:perturb1}). 
Therefore, the baseline verifier makes at least $|\nodes{T^{N}_{f}}|$ number of analyzer calls before terminating. 
While for any perturbed network $\perturbedNetwork$ the incremental verifier always terminates within $|\leaves{T^{N}_{f}}|$ analyzer calls. (Theorem~\ref{thm:termination_guarantee}) 
Assuming all bounding steps take the same time then the speed up achieved by the incremental verifier over the baseline is $\frac{|\nodes{T^{N}_{f}}|}{|\leaves{T^{N}_{f}}|}$.
Note we assume the same branching heuristic is used by both the baseline and incremental verifier such that $\treesetforN_{\Tool{}} = \treesetforN_{baseline} = \treesetforN$. 

\begin{restatable}{theorem}{last}
\label{thm:speed_up}
For any perturbed network $\perturbedNetwork \in \mathcal{M}(N, \EpsNorm)$ with $\EpsNorm \leq \frac{|\ProblemMin(N, \treesetforN)|}{\|\Lpc\|_2}$ the incremental verifier with $\Tinit = T^{N}_{f}$ always achives speed up of $\frac{|\nodes{T^{N}_{f}}|}{|\leaves{T^{N}_{f}}|}$ over the baseline verifier provided $\SolvFunc(N, T^{N}_{f}) = True$ and the branching heuristic $\hbranch$ is unchanged.
\end{restatable}
\begin{proof}
As $\frac{|\ProblemMin(N, \treesetforN)|}{\|\Lpc\|_2} \leq \frac{|\ProblemMin(\FeasibleReg(\OrgNetwork, T^{N}_{f}))|}{\|\Lpc\|_2 \cdot \MaxNorm(N, T^{N}_{f}, \Layers - 1)}$ the proposed algorithm always terminate within $|\leaves{T^{N}_{f}}|$ number of bounding steps for any $\perturbedNetwork \in \mathcal{M}(N, \EpsNorm)$. (Using Theorem~\ref{thm:termination_guarantee}). Apart from $T^{N}_{f}$ all specification trees $T \in \treesetforN$ were unsuccessful in verifying the property for $\OrgNetwork$. 
As shown in the following derivation all specification trees apart from $T^{N}_{f}$ will fail to prove the property for any network $\perturbedNetwork \in \mathcal{M}(N, \EpsNorm)$. 
\begin{align*}
\forall T \in \treesetforN (T < T^{N}_{f}) &\implies \neg \SolvFunc(N, T) \\
                                             &\implies \neg \SolvFunc(\perturbedNetwork, T)\;\;\;\text{[Using Theorem~\ref{thm:perturb1}]}
\end{align*}
\end{proof}

In this part, we briefly explain how the analyzers handle non-linear activation functions like $\Relu$ unit while verifying neural networks.
This is helpful in understanding the following proofs.
Let $x = \Relu(\hat{x})$ represents a relu unit with input $\hat{x}$ and output $x$.
As described in Section~\ref{sec:bab} we cannot linearize ambiguous $\Relu$ units where $lb(\hat{x}) < 0 < ub(\hat{x})$.
Therefore, the analyzer over-approximates the output of ambiguous $\Relu$ unit using convex relaxation.
\begin{figure}[!htbp]
\centering
\begin{subfigure}[b]{0.3\textwidth}
 \includegraphics[width=\textwidth]{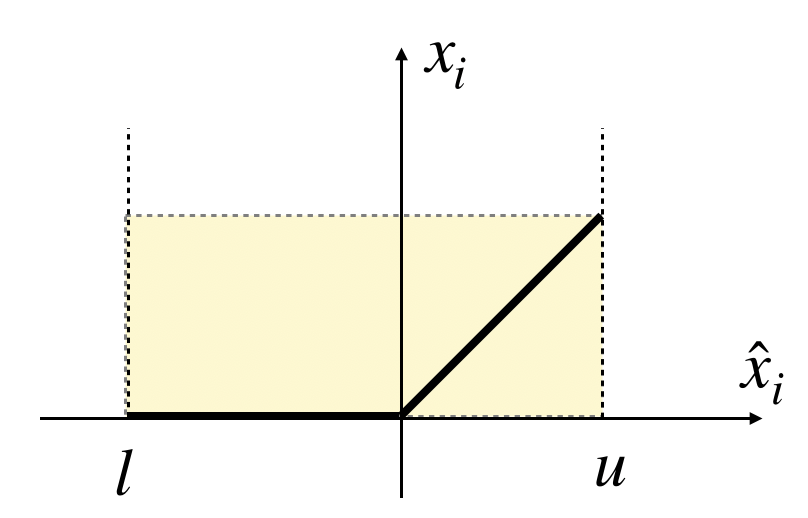}
 \caption{Box relaxation}
 \label{fig:convex_relax1}
\end{subfigure}
\hspace{3mm}
\begin{subfigure}[b]{0.3\textwidth}
 \includegraphics[width=\textwidth]{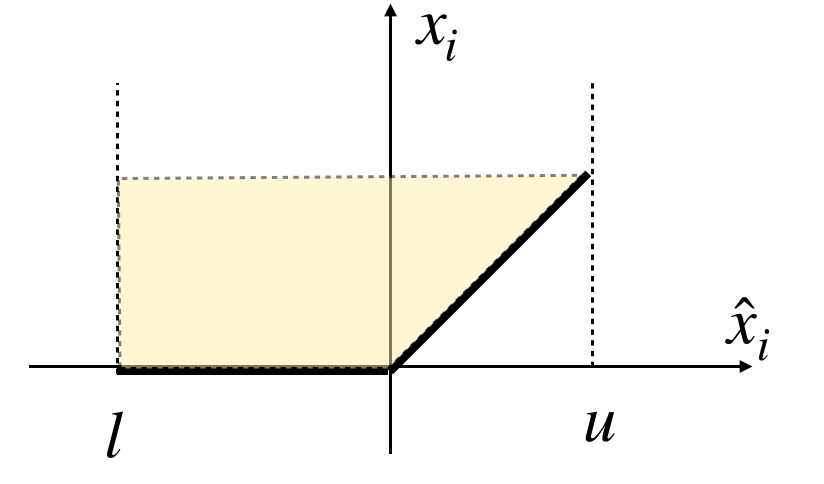}
 \caption{Qudrilateral relaxation}
 \label{fig:convex_relax2}
\end{subfigure}
\hspace{3mm}
\begin{subfigure}[b]{0.3\textwidth}
 \includegraphics[width=\textwidth]{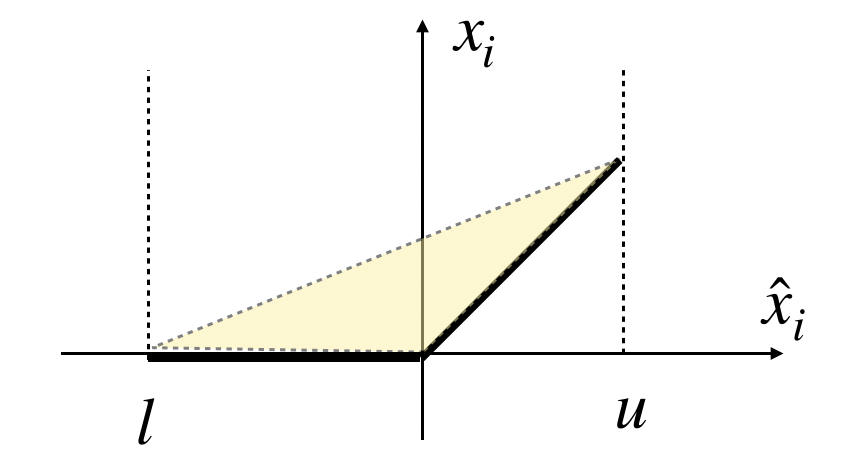}
 \caption{Triangle relaxation}
 \label{fig:convex_relax3}
\end{subfigure}
\hfill
\caption{Different convex relaxations of ambiguous $\Relu$ units. The dark line represents the actual output of the $\Relu$ unit and the shaded regions represent the over-approximated convex relaxations of the output.}
\label{fig:relu_convex_relax}
\end{figure} 
{sec:proofs2} convex relaxations for ambiguous $\Relu$s shown in Fig.~\ref{fig:relu_convex_relax}, the box relaxation is easiest to compute but is imprecise.
On the other hand, triangle relaxation is the most precise but complicated, while quadrilateral relaxation achieves a middle ground between them.
For all theoretical derivation presented below, we assume that the analyzer uses quadrilateral relaxation because it is more precise than box relaxation and simpler than triangle relaxation.
Let $\Reluabstract{\hat{x}}$ denotes over-approximated convex region of $\Relu(\hat{x})$.
For any ambiguous $\Relu$ unit, $x = \Relu(\hat{x})$ the $\Reluabstract{\hat{x}}$ under quadrilateral relaxation is defined by the following constraints.
\begin{align}
    \label{eq:convex_relu_constraints}
    x \geq 0 && x \geq \hat{x} && ub(\hat{x}) \geq x 
\end{align}


\begin{restatable}{lemma}{feasiblesolution}
\label{lem:feasiblesolution}
Let $\OrgNetwork$ and $\perturbedNetwork$ be two $\Layers$-layer networks with the same architecture and weight perturbation made only at the last layer $l$. If $\|\Eps\|_{F} \leq \EpsNorm$ then $\forall Y \in \FeasibleReg(\OrgNetwork_{\Layers}, T)$, $\exists Y' \in \FeasibleReg(\perturbedNetwork_{\Layers}, T)$ such that $\|Y - Y'\|_2 \leq \EpsNorm \cdot \MaxNorm(N, T)$. 
\end{restatable}
\begin{proof}
Let $Y[j]$ denotes the $j$-th coordinate of $Y$ and $A[j]$ and $\Eps[j]$ represent the $j$-th row of $A$ and $\Eps$ respectively. 
$Y \in \FeasibleReg(N_{\Layers}, T)$ then $\exists X \in \FeasibleReg(N_{(\Layers - 1)}, T)$ such that $\hat{Y} = A_{\Layers}X + B_{\Layers}$ and $(\forall j)$, $Y[j] \in \Reluabstract{\hat{Y}[j]}$.
 As first $\Layers -1$ layers of both $\OrgNetwork$ and $\perturbedNetwork$ are same then $X \in \FeasibleReg(\perturbedNetwork_{\Layers - 1}, T )$. 
Let $\hat{Y'} = (A_{\Layers} + \Eps)X + B_{\Layers}$. We first show that $(\forall j)$, $|\hat{Y}[j] - \hat{Y'}[j]| \leq \|\Eps[j]\|_2 \cdot \MaxNorm(N, T )$.
\begin{align*}
    |\hat{Y}[j] - \hat{Y'}[j]| &= |\Eps[j]X| \\
                            &\leq \|\Eps[j]\|_2 \cdot \|X\|_2 \\
                            &\leq \|\Eps[j]\|_2 \cdot \MaxNorm(N, T )
\end{align*}
\\ In the following part of the proof we show how to constuct $Y'[j]$ such that $|Y'[j] - Y[j]| \leq \|\Eps[j]\|_2 \cdot \MaxNorm(N, T )$ while ensuring that $Y'[j] \in \Reluabstract{\hat{Y'}[j]}$. In this case $\Reluabstract{Y'[j]}$ is the convex region defined by constraints presented in Eq.~\ref{eq:convex_relu_constraints}.

\begin{itemize}[leftmargin=*]
\item \textbf{Case 1:} $\hat{Y'}[j] \geq 0$ 
\begin{itemize}[leftmargin=*]
\item \textbf{Case 1.a } $ub(\hat{Y'}[j]) \geq Y[j] \geq \hat{Y'}[j])$
\\
In this case $Y'[j] = Y[j]$ satisfies the constraints defined in Eq.~\ref{eq:convex_relu_constraints}. Therefore, $Y'[j] = Y[j] \in \Reluabstract{\hat{Y'[j]}}$.
\item \textbf{Case 1.b } $\hat{Y'}[j] > Y[j]$
\\
$Y \in \FeasibleReg(N_{\Layers}, T)$ then $Y[j] \in \Reluabstract{\hat{Y}[j]}$ and $Y[j] \geq \hat{Y}[j]$. We show below for $Y'[j] = \hat{Y'}[j] \in \Reluabstract{\hat{Y}[j]}$ and $|Y[j] - Y'[j]| < \|\Eps[j]\|_2 \cdot \MaxNorm(N, T )$
\begin{align*}
    |Y[j] - Y'[j]| &= |Y[j] - \hat{Y'}[j]| \;\;\;\text{(for Case 1.b we define $Y'[j] = \hat{Y}[j]$)} \\
                &\leq |\hat{Y}[j] - \hat{Y'}[j]| \;\;\; (given\;\;\hat{Y'}[j] > Y[j] \geq \hat{Y}[j]) \\
                &\leq \|\Eps[j]\|_2 \cdot \MaxNorm(N, T )
\end{align*}
\item \textbf{Case 1.c } $Y[j] > ub(\hat{Y'}[j])$
\\
For this case we define $Y'[j] = ub(\hat{Y'}[j])$. We show below that $|Y[j] - Y'[j]| \leq \|\Eps[j]\|_2 \cdot \MaxNorm(N, T )$.
\begin{align*}
    |Y[j] - Y'[j]| &= |Y[j] - ub(\hat{Y'}[j])| \;\;\;\text{(for Case 1.c we define $Y'[j] = ub(\hat{Y'}[j])$)} \\
    &\leq |ub(\hat{Y}[j]) - ub(\hat{Y'}[j])|  \;\;\;(given\;\;ub(\hat{Y}[j]) \geq Y[j] > ub(\hat{Y'}[j])) \\
    &\leq \|\Eps[j]\|_2 \cdot \MaxNorm(N, T )
\end{align*}
\end{itemize}
\item \textbf{Case 2: } $\hat{Y'}[j] < 0$
\begin{itemize}[leftmargin=*]
\item \textbf{Case 2.a} $ub(\hat{Y'}[j]) \geq Y[j] \geq 0$
\\
Similar to case 1.a we define $Y'[j] = Y[j]$.
\item \textbf{Case 2b.} $Y[j] > ub(\hat{Y'}[j])$
\\
For this case we define $Y'[j] = max(0, ub(\hat{Y'}[j]))$. The proof $|Y[j] - Y'[j]| \leq \|\Eps[j]\|_2 \cdot \MaxNorm(N, T )$ is same as case 1.c.
\end{itemize}
\end{itemize}
In all the previous cases we assumed $j$-th $\Relu$ unit of the final layer is not splitted. Otherwise for $Y[j]$ either 0 or $Y[j]= \hat{Y}[j]$. Similarly $Y'[j]$ either 0 or $Y'[j]= \hat{Y'}[j]$.
In all of these cases $|Y[j] - Y'[j]| \leq \|\Eps[j]\|_2 \cdot \MaxNorm(N, T )$ as we already proved $|\hat{Y}[j] - \hat{Y'}[j]| \|\Eps[j]\|_2 \cdot \MaxNorm(N, T )$
All these cases we have shown that $(\forall j)\;|Y[j] - Y'[j]| \leq \|\Eps[j]\|_2 \cdot  \MaxNorm(N, T
)$ with $Y' \in \FeasibleReg(\perturbedNetwork, T)$
\begin{align*}
    \|Y - Y'\|^{2}_{2} &= \sum^{\Dimension_{\Layers+1}}_{j=1} |Y[j] - Y'[j]|^2 \\
    \|Y - Y'\|^{2}_{2} &\leq \MaxNorm(N, T)^2 \cdot \sum^{\Dimension_{\Layers+1}}_{j=1} \Eps[j]^2 \\
    \|Y - Y'\|_{2} &\leq \EpsNorm \cdot \MaxNorm(N, T )
\end{align*}
\end{proof}


\end{document}